\documentclass[aoas,preprint]{imsart}

\RequirePackage{amsthm,amsmath,amsfonts,amssymb}
\RequirePackage[authoryear]{natbib}
\RequirePackage[colorlinks,citecolor=blue,urlcolor=blue]{hyperref}
\RequirePackage{graphicx}
\graphicspath{{./}{figures/}{Real_Application/figures/}{../figures/}}
\RequirePackage{booktabs}
\RequirePackage{rotating}   
\RequirePackage{float}
\RequirePackage{array}

\startlocaldefs
\theoremstyle{plain}
\newtheorem{theorem}{Theorem}[section]
\newtheorem{lemma}[theorem]{Lemma}
\newtheorem{corollary}[theorem]{Corollary}

\theoremstyle{definition}

\newtheorem{remark}[theorem]{Remark}

\endlocaldefs

\begin{document}

\begin{frontmatter}

\title{Non-Crossing Deep Quantile Regression for Distributional Survival Prediction}

\runtitle{Non-Crossing Quantile Regression for Time-to-Event Analysis}

\begin{aug}
\author[A]{\fnms{Shuai}~\snm{Huang}
\ead[label=e1]{shuaishu@email.unc.edu}},
\author[B]{\fnms{Zhe}~\snm{Qu}\thanks{[\textbf{Corresponding author information is provided in the Acknowledgments section.}]}
\ead[label=e2]{zhe.qu@servier.com}},
\author[B]{\fnms{Zhaowei}~\snm{Hua}\ead[label=e3]{zhaowei.hua@servier.com}},
\author[C]{\fnms{Guohao}~\snm{Shen}\ead[label=e4]{guohao.shen@polyu.edu.hk}},
\author[D]{\fnms{Rui}~\snm{Tang}\ead[label=e5]{rui.tang@astellas.com}}
\and
\author[A]{\fnms{Hongtu}~\snm{Zhu}\thanks{[\textbf{Corresponding author information is provided in the Acknowledgments section.}]}
\ead[label=e6]{htzhu@email.unc.edu}}
\address[A]{Department of Biostatistics, University of North Carolina at Chapel Hill
\printead[presep={,\ }]{e1,e6}}

\address[B]{Servier Pharmaceuticals, Boston, MA
\printead[presep={,\ }]{e2,e3}}

\address[C]{The Hong Kong Polytechnic University, Hung Hom, Hong Kong, China
\printead[presep={,\ }]{e4}}

\address[D]{Astellas Pharma US, Northbrook, IL
\printead[presep={,\ }]{e5}}
\end{aug}

\begin{abstract}
In survival analysis the way covariates act on the risk of an event often differs
between early and late failure times, yet hazard- and mean-based summaries collapse
this variation into a single number. Quantile-based modeling instead describes the
full conditional distribution on the original time scale, but existing censored-data
methods are either inflexible or produce logically inconsistent crossing quantile
curves. We propose a Censored Non-crossing Quantile (CNQ) framework for
right-censored data that jointly estimates several conditional survival quantiles and
guarantees valid ordering by construction, with flexibility supplied by
Kolmogorov--Arnold and Transformer backbones, and we establish a finite-sample
excess-risk bound holding jointly across all fitted quantile levels. Across 27
simulation settings and six cohorts the framework attains lower pinball loss than
quantile-, hazard- and tree-based competitors whenever the conditional distribution is
asymmetric, with interval coverage closer to nominal on all six. In two clinical case
studies (METABRIC, breast cancer; FLCHAIN, population mortality) it recovers covariate
effects that vary across the survival distribution and would be hidden by a single
hazard ratio, and yields coherent individualized quantile milestones.
GitHub repo: \url{https://github.com/BIG-S2/deepcnq}.
\end{abstract}


\end{frontmatter}



\section{Introduction}
\label{s:intro}

Survival analysis forecasts whether and when an event will occur---for example, disease onset informing clinical decisions~\citep{zhao2021event}. Classical survival models~\citep{cox1972regression,wei1992accelerated,martinussen2022causality,kalbfleisch2023fifty,salerno2023high,kalbfleisch2002statistical}, including the Cox proportional hazards (CoxPH) and accelerated failure time (AFT) models, have long served as core tools because of their interpretability and solid theoretical foundations. However, assumptions such as proportional hazards and linear covariate effects limit their ability to represent the patient heterogeneity and interacting prognostic factors common in real data.
To relax these restrictions, Random Survival Forests (RSF)~\citep{ishwaran2008random} extend tree-ensemble methods to censored outcomes and can capture nonlinearities and high-dimensional effects more flexibly than traditional parametric or semiparametric models. Nevertheless, RSF may still generalize poorly in settings with complex structure and high noise~\citep{van2007support}.

More recently, deep learning has further expanded the survival modeling toolbox by replacing hand-specified functional forms with neural networks that learn flexible representations from data. For comprehensive reviews, see~\citep{wiegrebe2024deep, chen2024introduction}. Representative methods include DeepSurv~\citep{katzman2018deepsurv}, which generalizes CoxPH by learning nonlinear covariate effects, and Cox-Time~\citep{kvamme2019time}, which allows time-varying effects and thus accommodates non-proportional hazards. Cox-CC~\citep{kvamme2019time} improves scalability for large datasets, PC-Hazard~\citep{kvamme2021continuous} refines time-dependent hazard estimation, and Survival Kernets~\citep{chen2024survival} provides a scalable and interpretable deep kernel framework with theoretical accuracy guarantees. Interpolation and double descent have also been documented in likelihood-based neural survival models~\citep{liu2025overparametrization}.

Despite these advances, many deep survival models still target the hazard function or the mean survival time, summaries that need not capture heterogeneous risk or the full conditional distribution of event times. Quantile regression~\citep{koenker1978regression} offers a natural alternative: it models the conditional survival quantiles directly, letting covariate effects vary across different parts of the distribution while retaining an interpretation on the original time scale. Pioneering work extended quantile regression to censored data~\citep{powell1986censored, ying1995survival, peng2008survival}, enabling principled inference across a broad range of survival quantiles.

 Existing quantile-based methods, however, face a dilemma. Classical censored quantile regression relies on a linear predictor and therefore inherits the same inability to represent nonlinear, interacting covariate effects that limits the Cox and AFT models. Neural extensions such as DeepQuantreg~\citep{jia2022deep} and CQRNN~\citep{pearce2022censored} restore flexibility, but they typically fit each quantile level with a separate network and impose no ordering across levels, so the estimated quantile curves can cross (for instance, a predicted lower quantile exceeding a higher one), yielding a logically inconsistent conditional distribution. What is missing is a single model that is expressive enough to capture nonlinear covariate effects yet guarantees coherent, non-crossing quantiles.

To address this gap, we propose a Censored Non-crossing Quantile (CNQ)
framework for right-censored data, combining a flexible feature extractor with a
monotone output module. For the former we build on Kolmogorov--Arnold Networks
(KAN)~\citep{liu2024kan}, which represent high-dimensional functions through
compositions of learned univariate transformations, and
Transformers~\citep{vaswani2017attention}, which use self-attention to capture
complex dependencies among covariates; we also introduce a KAN--Transformer
hybrid. The output module jointly estimates multiple survival quantiles and
enforces non-crossing by construction, so the fitted conditional distribution is
internally consistent for every patient.

Our contributions are fourfold. First, we formulate censored survival analysis as
the joint estimation of multiple conditional quantiles within a single
optimization problem, rather than one network per level. Second, we instantiate
this formulation with the three backbones above, capturing
nonlinearities and feature interactions that traditional survival models and prior
MLP-based quantile networks may miss. Third, we establish a non-asymptotic excess-risk bound
that holds jointly across all fitted quantile levels, clarifying how performance
depends on the at-risk probability and the network covering number, and combine it
with architecture-specific approximation constructions to obtain a rate of order
$\{N/\log N\}^{-\beta/(2\beta+p)}$ for the KAN and Transformer sieves over
$\beta$-H\"older conditional quantiles. Finally, beyond benchmark comparisons across
$27$ simulation settings and six cohorts, we analyse two clinical datasets in depth,
recovering covariate effects that vary across the survival distribution,
together with coherent individualized
survival-time quantiles with calibrated uncertainty. That analysis also yields two
findings of independent interest: an apparent event-projection bias under
inverse-probability-of-censoring training, and a caution on the architecture
dependence of feature attributions in spline-based backbones.

\section{Data and Motivation}
\label{s:data}

\subsection{Real Data Descriptions}
\label{s:real}


We evaluate our methods on six real-world right-censored survival datasets that are widely used in the survival analysis literature and in modern deep-survival benchmarks: the Study to Understand Prognoses and Preferences for Outcomes and Risks of Treatments (SUPPORT)~\citep{knaus1995support}, the Molecular Taxonomy of Breast Cancer International Consortium (METABRIC)~\citep{curtis2012genomic}, the German Breast Cancer Study Group dataset (GBSG)~\citep{schumacher1994randomized}, a randomly selected subset of GBSG with $n=500$ subjects (GBSG-500), the Netherlands Cancer Institute 70-gene signature dataset (NKI70)~\citep{van2002gene}, and the Free Light Chain population-mortality cohort (FLCHAIN)~\citep{Dispenzieri2012}. Together, these cohorts cover diverse clinical settings and covariate modalities, with sample sizes ranging from $n=144$ to $n=8{,}873$ and censoring proportions spanning moderate to heavy. This diversity provides a stringent and realistic testbed for distributional survival prediction under typical follow-up constraints.

SUPPORT and FLCHAIN are the two non-oncology cohorts.  SUPPORT studies survival among seriously ill hospitalized adults~\citep{knaus1995support}, with demographic variables, comorbidities (e.g., diabetes, dementia, cancer) and physiological or laboratory measurements (e.g., vital signs, white blood cell count, serum sodium, creatinine); short-term mortality and longer-term survival there are driven by different processes, so the lower and upper survival quantiles carry distinct clinical content.  FLCHAIN~\citep{Dispenzieri2012} is a population-based cohort of $n=7{,}874$ adults from Olmsted County, Minnesota, with all-cause mortality as the outcome and the heaviest censoring in our study ($\approx 72\%$).  Its seven covariates are demographic and laboratory-based, with no treatment variables: age, sex, serum $\kappa$ and $\lambda$ free light chain concentrations, creatinine, monoclonal gammopathy (MGUS) status, and a creatinine missingness indicator.  It tests whether distributional predictions stay calibrated when most event times are unobserved.

The remaining four are breast cancer studies, chosen to form a gradient in sample size, censoring severity and covariate type.  METABRIC~\citep{curtis2012genomic} combines gene expression markers (e.g., MKI67, ESR1, PGR, ERBB2) with clinical and treatment covariates (e.g., chemotherapy, radiotherapy, hormone therapy, ER status, age), giving the complex, potentially non-additive predictor--outcome relationships that make multi-quantile estimation informative.  GBSG~\citep{schumacher1994randomized} is a classical trial-based benchmark in node-positive breast cancer; GBSG-500 randomly subsets it to 500 subjects while preserving covariate and censoring structure, so the pair isolates the effect of sample size on tail quantiles.  NKI70~\citep{van2002gene} is the most constrained setting---144 subjects, $67\%$ censoring and high-dimensional gene-signature covariates---typical of genomic prognostic studies.

Across all datasets, covariate selection follows the established benchmarking protocol in \citet{katzman2018deepsurv} to ensure fair comparisons with common deep-survival baselines and consistency with prior work. Table~\ref{tab:dataset_summary} summarizes the sample size, number of covariates, and censoring proportion for each dataset. We also report Kaplan--Meier curves with confidence intervals in the Supplementary Figure S1 to provide cohort-level context on survival patterns and censoring profiles.

\subsection{Motivation: Why Distributional (Quantile) Survival Modeling?}
\label{s:data_motivation}

Time-to-event modeling in these real datasets is complicated by three pervasive features: (i)~\emph{right censoring}; (ii)~\emph{heterogeneity}, in that prognostic factors may act differently on different parts of the survival-time distribution; and (iii)~\emph{nonlinearity and high-dimensional interactions} arising from biological mechanisms, treatment pathways and patient subtypes. These characteristics motivate a distributional perspective beyond hazard-only or mean-based summaries.

Knowing whether a patient's overall risk is higher or lower is often not enough; clinicians also need to know \emph{when} events are likely to occur, and how both tails of the survival distribution vary across patients. For example, in critical care (SUPPORT), accurately characterizing a patient's short-term mortality risk alongside the prospect of longer-term survival informs qualitatively different care decisions. In oncology (METABRIC, GBSG, NKI70), the distinction between early relapse and durable remission is central to treatment planning and patient counseling. Moreover, the clinical drivers of these outcomes may differ across the survival distribution: a biomarker or treatment that strongly reduces early relapse risk may have little bearing on long-term prognosis.

Quantile regression is the natural response to this heterogeneity \citep{koenker1978regression,portnoy2003censored, peng2021quantile}, and it suits censoring particularly well: under bounded follow-up the mean survival time may be non-identifiable, whereas survival quantiles remain identifiable over clinically relevant ranges \citep{peng2021quantile}. Estimating several quantiles flexibly, however, reintroduces the \emph{crossing} problem of Section~\ref{s:intro}, whose logical inconsistency makes the output unusable for treatment planning.

Motivated by these observations, we aim to address the following questions:
\begin{enumerate}
  \item[\textbf{(Q1)}] \textbf{Individualized survival-time quantiles clinicians can act on.}
  For an individual patient, can a small set of survival-time quantiles be estimated coherently and remain well calibrated (for example, a plausible earliest-decline time and a durable-survival horizon), and does the spread between them flag which patients carry the greatest prognostic uncertainty and therefore warrant closer monitoring?
  \item[\textbf{(Q2)}] \textbf{Which markers govern early relapse versus long-term survival, and does this refine counseling?}
  In breast-cancer cohorts, do specific molecular markers and treatments (e.g., estrogen-receptor status, proliferation markers, hormone therapy) act primarily on the risk of early events or on the long-term survival ceiling, and would recognizing such tail-specific effects (which a single hazard ratio averages away) change how patients are risk-stratified and counseled?
  \item[\textbf{(Q3)}] \textbf{Do these quantile estimates stay trustworthy in the small, heavily censored cohorts typical of oncology?}
  Genomic prognostic studies are often small and heavily censored (e.g., NKI70, with $n=144$ and $67\%$ censoring). Do the individualized quantile estimates remain calibrated and precise enough to inform decisions in such settings, or must their individual-level use be down-weighted relative to larger cohorts?
\end{enumerate}
These questions drive the methodological development presented in Section~\ref{sec:method} and the empirical evaluation in Sections~\ref{sec:experiments}--\ref{sec:realdata}.


\section{Methodology}
\label{sec:method}

\subsection{Existing Work}

Quantile regression~\citep{koenker1978regression} characterizes the conditional $\tau$-th quantile of the event time $T$ given covariates $X$, defined as
\begin{equation}\label{eq:linear_quantreg}
Q_{T\mid X}(\tau)=\inf\bigl\{t:\Pr(T\le t\mid X)\ge \tau\bigr\}.
\end{equation}
A common specification is the log-linear quantile regression model,
$
\log\{Q_{T\mid X}(\tau)\}=\beta_\tau^\top X,$
where $\beta_\tau$ is a vector of unknown parameters.

In the presence of right censoring, for subject $i=1,\ldots,N$, let $T_i$ and $C_i$ denote the event time and censoring time, and we observe $Y_i=\min(T_i,C_i)$ and the event indicator $\delta_i=\mathbb{I}(T_i\le C_i)$. We assume $\{(Y_i,\delta_i,X_i)\}_{i=1}^N$ are i.i.d., where $X_i=(X_{i1},\ldots,X_{ip})^\top\in\mathbb{R}^p$ denotes the covariate vector. A censored quantile regression estimator $\widehat{\beta}_\tau\in\mathbb{R}^p$ can be obtained by minimizing an inverse-probability-weighted check loss:

\begin{equation}\label{eq:cqr_ipw}
\widehat{\beta}_\tau
=\arg\min_{\beta}\sum_{i=1}^N \omega_i\,\rho_\tau\!\bigl(\log(Y_i)-\beta^\top X_i\bigr),
\end{equation}
where $\rho_\tau(a)=a\{\tau-\mathbf{1}(a<0)\}$ is the check loss and $\omega_i=\delta_i/\{N\,\widehat{G}(Y_i^-)\}$ are weights based on $\widehat{G}(t)$, the Kaplan--Meier estimator of $G(t)=\Pr(C>t)$. Consistency of $\widehat{\beta}_\tau$ (e.g., for median regression) was established in~\citet{huang2007least}.

DeepQuantreg~\citep{jia2022deep} generalizes this framework by replacing the linear predictor $\beta_\tau^\top X$ with a flexible nonlinear function learned by a deep neural network:
\[
\log\{Q_{T\mid X}(\tau)\}=f_{\theta_\tau}(X),
\]
where $f_{\theta_\tau}(\cdot)$ denotes a network with parameters $\theta_\tau$. The parameters are estimated by minimizing the corresponding weighted check loss,
\[
L_N(\theta_\tau)=\sum_{i=1}^N \omega_i\,\rho_\tau\!\bigl(\log(Y_i)-f_{\theta_\tau}(X_i)\bigr).
\]
In \citet{jia2022deep} the check function is Huber-smoothed \citep{huber1973robust} to ensure differentiability; since the smoothing alters $\rho_\tau$ only for residuals within a small bandwidth $\xi$ of the origin, we work with the exact check loss $\rho_\tau$ throughout.
This approach aims to capture complex nonlinear relationships between covariates and event times under right censoring. However, DeepQuantreg uses a relatively simple MLP with two hidden layers, which may be insufficient for highly structured or high-dimensional covariates. In addition, it fits a separate network for each quantile level $\tau$, increasing computational cost and potentially leading to quantile crossing across different $\tau$.

\paragraph{What is new relative to CQRNN and DeepQuantreg}
Our proposed CNQ framework departs from these methods in three ways that are tied to the data features described in Section~\ref{s:data}. First, whereas DeepQuantreg fits a separate MLP per quantile and CQRNN optimizes a shared quantile grid without an ordering constraint, CNQ estimates all quantiles jointly and enforces non-crossing \emph{by construction}, which is what makes the coherent individualized milestones sought in (Q1) well defined. The non-crossing construction itself builds on the deep non-crossing quantile networks of \citet{wu2023dnet} and \citet{shen2025deep}; our contribution is to adapt it to right-censored survival data through an inverse-probability-of-censoring-weighted objective and to combine it with the feature-learning backbones introduced below. Second, the heterogeneous, non-additive clinical-plus-molecular covariates in cohorts such as METABRIC motivate replacing the shallow MLP backbones of both baselines with feature-learning architectures that can represent interactions a two-layer MLP misses: attention across covariates (Trans-CNQ) and learnable spline transforms (KAN). Third, Section~\ref{sec:theory} establishes a finite-sample excess-risk bound for the censored non-crossing estimator, holding jointly across all fitted quantile levels and covering the complete Trans-CNQ architecture; neither baseline provides a corresponding guarantee.  The choice of backbone is supporting rather than central: the primary contribution is the demonstration, on real cohorts, that a coherent multi-quantile formulation delivers calibrated, clinically interpretable predictions and exposes quantile-specific covariate effects that hazard-based summaries obscure.

\subsection{Censored Non-Crossing Quantile (CNQ) Framework}
\label{sec:cnq}

We propose a \textbf{Censored Non-Crossing Quantile (CNQ)} framework for right-censored survival data that jointly estimates multiple conditional quantiles through a single optimization problem while guaranteeing that the predicted quantiles are properly ordered (non-crossing). The overall architecture is illustrated in Figure~\ref{Model_structure}.

\begin{figure}[htp]
	{\includegraphics[width=.85\textwidth]{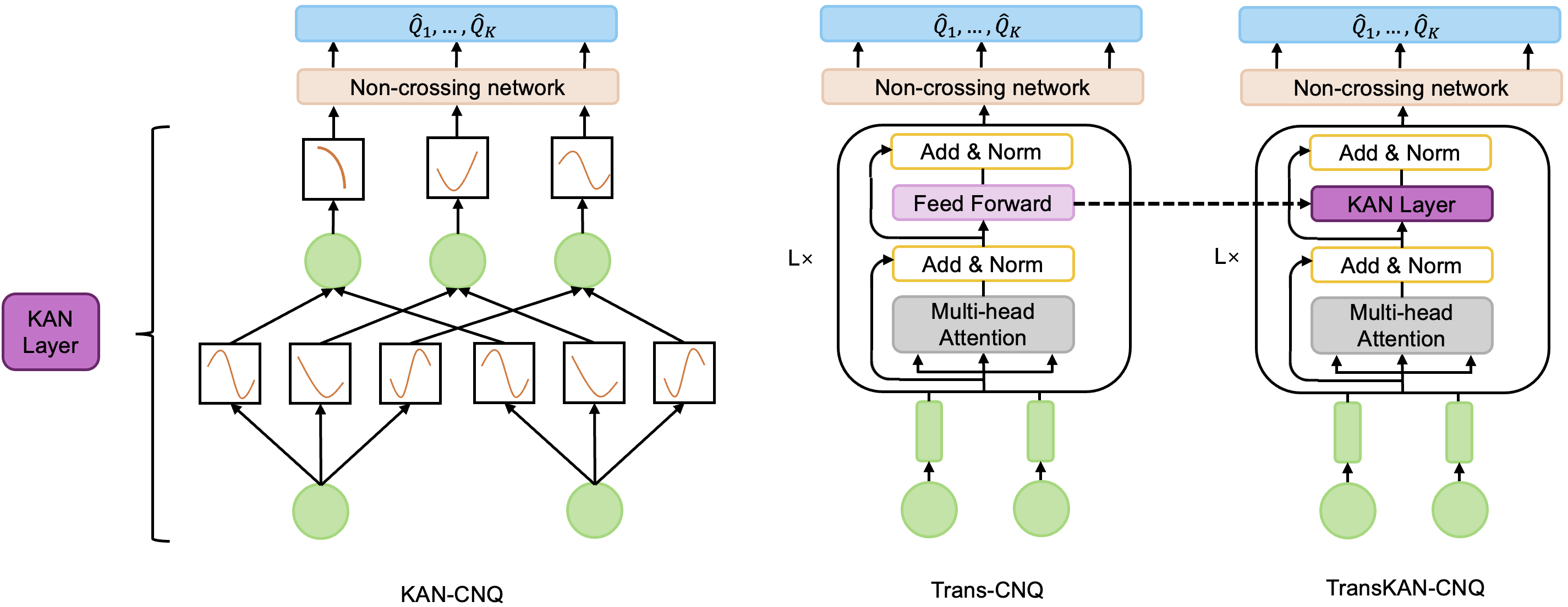}}
	\caption{The architectures of KAN-CNQ, Trans-CNQ, and TransKAN-CNQ.}
	\label{Model_structure}
\end{figure}

Let $\tau_1<\cdots<\tau_K$ denote the target quantile levels. For $N$ i.i.d.\ subjects with observations $\{(Y_i,\delta_i,X_i)\}_{i=1}^N$, we define a vector-valued quantile network
\[
f(X;\theta)=\bigl(f_1(X; \theta),\ldots,f_K(X;\theta)\bigr)^\top,
\]
where $f_k(X;\theta)$ estimates the $\tau_k$-th conditional quantile of $\log(T)$ given $X$.
The CNQ objective is the average of inverse-probability-weighted check losses across both subjects and quantile levels:
\begin{equation}\label{eq:cnq_loss}
L_N(f)=\frac{1}{K}\sum_{k=1}^K\sum_{i=1}^N \omega_i\,\rho_{\tau_k}\!\bigl(\log(Y_i)-f_k(X_i;\theta)\bigr),
\end{equation}
where $\rho_{\tau}(a)=a\{\tau-\mathbf{1}(a<0)\}$ is the check loss and $\omega_i$s' are inverse-probability-of-censoring weights (e.g., inverse Kaplan--Meier weights).

\subsubsection*{Non-Crossing Output Module}
A key difficulty in multi-quantile modeling is \textbf{quantile crossing}: when quantiles are trained independently, estimated lower quantiles can exceed higher ones (e.g., $\widehat{Q}_{T\mid X}(0.3)>\widehat{Q}_{T\mid X}(0.7)$), yielding an invalid conditional distribution. To prevent this, CNQ adopts a structurally constrained parameterization that builds on the non-crossing quantile networks of \citet{wu2023dnet} and \citet{shen2025deep} and enforces
\[
f_1(X;\theta)\le f_2(X;\theta)\le \cdots \le f_K(X;\theta)\quad \text{for all } X,
\]
thereby guaranteeing monotonicity across quantile levels by construction.

Let $Z = g(X;\theta_g) \in \mathbb{R}^{d_{\mathrm{out}}}$ denote the output of the feature-extraction backbone, where $d_{\mathrm{out}}$ is the backbone's output dimension and $g$ is one of the three architectures described in the \nameref{sssec:model_arch} paragraph below. This representation is fed into a shared non-crossing output module consisting of two components. All outputs of this module are first defined on the log-time scale. The \textbf{base network} $h(X;\theta)\in\mathbb{R}$ outputs the logarithm of the $\tau_1$-quantile estimate, and the \textbf{steps network} $s(X;\theta)\in\mathbb{R}^{K-1}$ produces unconstrained increments encoding the gaps between consecutive log-time quantiles. Specifically, two separate linear heads produce the base and step outputs:
\begin{equation}\label{eq:heads}
  h(X;\theta)=W_b\,Z+b_b\in\mathbb{R},\qquad
  s(X;\theta)=W_s\,Z+b_s\in\mathbb{R}^{K-1},
\end{equation}
where $W_b \in \mathbb{R}^{1 \times d_{\mathrm{out}}}$, $W_s \in \mathbb{R}^{(K-1) \times d_{\mathrm{out}}}$, and $b_b, b_s$ are bias terms. Since the raw increments $s(X;\theta)$ may take negative values, we apply the softplus function $\sigma^+(u)=\log(1+e^u)$ to map them to nonnegative values. Defining $d(X)=\sigma^+\!\bigl(s(X;\theta)\bigr)\in\mathbb{R}_+^{K-1}$, we construct the quantile outputs via cumulative summation:
\begin{equation}\label{eq:noncrossing_output}
  \widehat{V}_1(X)=h(X;\theta),\qquad
  \widehat{V}_k(X)=h(X;\theta)+\sum_{j=1}^{k-1}d_j(X),
  \quad k=2,\ldots,K.
\end{equation}
The corresponding quantiles on the original time scale are
\(\widehat Q_k(X)=\exp\{\widehat V_k(X)\}\).
Because each increment $d_j(X)\ge 0$, the resulting log-time and original-time quantile estimates are automatically ordered without post-hoc rearrangement.

\subsubsection*{Model Architectures}
\phantomsection\label{sssec:model_arch}
We instantiate the CNQ framework with three independent feature-extraction backbones: KAN, Transformer, and a novel KAN--Transformer hybrid. Each model is trained separately end-to-end; they differ only in the feature-extraction component, while all adopting the same non-crossing output architecture described above, as illustrated in Figure~\ref{Model_structure}. This modular design enables direct comparison of the core feature-extraction components while ensuring valid quantile ordering across all models.

\textbf{KAN-CNQ.} Kolmogorov--Arnold Networks, introduced by \citet{liu2024kan}, are motivated by the Kolmogorov--Arnold representation theorem~\citep{kolmogorov1957representation}, which states that any multivariate continuous function can be decomposed into compositions of univariate functions. Unlike MLPs, which apply fixed activation functions at nodes, KANs place learnable univariate functions on the edges connecting nodes. Specifically, let $n_l$ denote the number of nodes in the $l$-th layer and let $X_{l,i}$ represent the $i$-th node in that layer. The output of the $j$-th node in the $(l+1)$-th layer is computed as
$
X_{l+1,j} = \sum_{i=1}^{n_l} \varphi_{l,j,i}(X_{l,i})$ for $ \quad j=1, \dots, n_{l+1},$
where each edge activation $\varphi_{l,j,i}$ combines a fixed basis function with a learnable spline component
as $
\varphi(x) = w_b \, b(x) + w_s \sum_{m=1}^{M} \alpha_m B_m(x).
$
Here $b(x)$ is a basis function (e.g., SiLU), $\{B_m\}_{m=1}^{M}$ are B-spline basis functions, and $\alpha_m$ are learnable coefficients. By stacking multiple KAN layers, the model builds hierarchical representations of the input covariates. The final KAN layer produces the representation $Z \in \mathbb{R}^{d_{\mathrm{out}}}$, which is fed into the two linear heads in~\eqref{eq:heads} to produce the quantile estimates via the non-crossing output module~\eqref{eq:noncrossing_output}.

\textbf{Trans-CNQ.} The Transformer architecture~\citep{vaswani2017attention} captures complex dependencies among input features through self-attention. While originally designed for sequential data, the Transformer encoder can be adapted to tabular settings by computing attention across features rather than temporal positions, following the approach of TabTransformer~\citep{huang2020tabtransformer}.

We embed every scalar feature with the same affine map,
$\widetilde{X}_{ij}=\psi(X_{ij})\in\mathbb R^q$; thus the implementation uses a shared scalar-token embedding rather than $p$ feature-specific layers.

This transforms the input $X_i \in \mathbb{R}^p$ into an embedding matrix $\widetilde{X}_i \in \mathbb{R}^{p \times q}$. A fixed sinusoidal positional encoding \(P\in\mathbb R^{p\times q}\) is added to this matrix, followed by an initial affine LayerNorm, giving \(Z_{i0}=\mathrm{LN}_0(\widetilde X_i+P)\). Self-attention is then applied independently to each sample as
\begin{equation*}
\mathrm{Attention}(Q_i, K_i, V_i) = \mathrm{softmax}\!\left( \frac{Q_i K_i^\top}{\sqrt{d_k}} \right)V_i,
\end{equation*}
where $Q_i = Z_{i0} W^Q$, $K_i = Z_{i0} W^K$, and $V_i = Z_{i0} W^V$ are affine projections, with biases suppressed in the notation, and $d_k$ is the key dimensionality. We stack $L$ Transformer encoder layers, each consisting of multi-head self-attention and feedforward sub-layers. Each layer uses post-LayerNorm residual blocks, with an Add--LayerNorm operation after attention and another after the two-layer ReLU feedforward sub-layer. The final token matrix is flattened and passed through an affine map and ReLU to obtain the representation $Z \in \mathbb{R}^{d_{\mathrm{out}}}$, which is fed into the two linear heads in~\eqref{eq:heads} to produce the quantile estimates via the non-crossing output module~\eqref{eq:noncrossing_output}. In the fitted model $d_{\mathrm{out}}=q$; the approximation sieve below allows this flattened-readout width to be tuned independently while retaining the same encoder family.

\textbf{TransKAN-CNQ.}
We further propose a hybrid architecture that integrates the strengths of both Transformer and KAN. Specifically, we replace the standard MLP-based feedforward sub-layers within the Transformer encoder blocks with KAN layers. This substitution retains the Transformer's capacity for modeling pairwise feature interactions through self-attention, while leveraging KAN's learnable spline-based activations for more expressive nonlinear transformations within each encoder block.

The hybrid encoder consists of $L$ modified Transformer layers, each comprising a multi-head self-attention sub-layer followed by a KAN-based feedforward sub-layer. The resulting representation $Z \in \mathbb{R}^{d_{\mathrm{out}}}$ is fed into the two linear heads in~\eqref{eq:heads} to produce the quantile estimates via the non-crossing output module~\eqref{eq:noncrossing_output}, ensuring monotonicity in the predicted quantiles. This design allows TransKAN-CNQ to capture both inter-feature dependencies (via attention) and complex univariate nonlinearities (via KAN) within a unified architecture.

\section{Theoretical Guarantees}
\label{sec:theory}

We formulate the result directly in the empirical Euclidean metric used by
\citet{zhang2024generalization}.  This matches the available KAN covering
bound and treats all \(K\) log-time quantile outputs jointly.

Let \(V=\log T\), \(W=\log Y\), \(Y=\min(T,C)\), and assume \(T,C>0\)
almost surely.  Assume also that \(\mathbb E|\log T|<\infty\).
Let \(\delta=\mathbb I\{T\leq C\}\).  We make the following
assumptions.
\begin{itemize}
    \item[(A1)] \((T,X)\perp C\).  Write
    \(G(t)=\mathbb P(C>t)\) and \(G(t^-)=\mathbb P(C\geq t)\).
    \item[(A2)] There is a deterministic set
    \(\mathcal T\subset(0,\infty)\) such that
    \(\mathbb P(T\in\mathcal T)=1\) and
    \[
        G_\star:=\inf_{t\in\mathcal T}G(t^-)>0.
    \]
    The Kaplan--Meier estimator \(\widehat G\) obeys, for the values of
    \(\eta\) under consideration,
    \[
      \mathbb P\!\left\{
      \Delta_N:=\sup_{t\in\mathcal T}
      |\widehat G(t^-)-G(t^-)|>\kappa_N(\eta)\right\}
      \leq c_Ge^{-\eta},\qquad
      \kappa_N(\eta)\leq G_\star/2.
    \]
\end{itemize}

Both assumptions are dictated by the estimator actually used.  (A1) makes the
censoring time independent of the covariates as well as of the event time, which
is what validates the marginal Kaplan--Meier weight \(\widehat G(Y_i^-)\) of
Section~\ref{sec:method}, used for every evaluation metric in
Section~\ref{sec:metrics}.  (A2) strengthens the endpoint condition
\(\tau_T<\tau_C\) to positivity of the censoring survival over the support of
\(T\), excluding the unbounded-support case \(\tau_T=\tau_C=\infty\).  Its
second display isolates the only property of \(\widehat G\) needed below; the
DKW--Kaplan--Meier inequality of \citet{bitouze1999dkw} gives
\(\kappa_N(\eta)=O\{(D_o+\sqrt{\eta})/\sqrt N\}\) under the fixed-horizon
conditions of \citet{goldberg2019hoeffding}, which do not verify (A2) over a
full continuous support with vanishing at-risk probability; no such
verification is claimed here.

For a non-crossing vector-valued function
\(f=(f_1,\ldots,f_K)\in\mathcal F_N\), suppose
\(\max_k\|f_k\|_\infty\leq M_f\), and define the log-time risk
\begin{equation}\label{excess_risk}
 L(f)=\mathbb E\!\left[\frac1K\sum_{k=1}^K
 \rho_{\tau_k}\{V-f_k(X)\}\right],\qquad
 R_{\mathcal F_N}(f)=L(f)-L(f_N^\star),
\end{equation}
where \(f_N^\star\in\arg\min_{f\in\mathcal F_N}L(f)\).  Its empirical
IPCW counterpart is
\[
 \widehat L_N(f)=\frac1{NK}\sum_{i=1}^N\sum_{k=1}^K
 \frac{\delta_i}{\widehat G(Y_i^-)}
 \rho_{\tau_k}\{W_i-f_k(X_i)\},
 \qquad
 \widehat f_N\in\arg\min_{f\in\mathcal F_N}\widehat L_N(f).
\]
By the IPCW identity proved in the Supplementary Material,
replacing \(\widehat G\) by \(G\) makes this
an unbiased empirical version of \(L\).

For a realized design \(\mathbf X=(X_1,\ldots,X_N)\), let
\[
 \mathcal F_N(\mathbf X)
 =\{(f_k(X_i))_{i\leq N,k\leq K}:f\in\mathcal F_N\}
 \subset\mathbb R^{N\times K}.
\]
Assume that a deterministic \(A_N>0\) satisfies, almost surely in
\(\mathbf X\), for every \(u>0\),
\begin{equation}\label{eq:empirical_entropy}
 \log\mathcal N\{\mathcal F_N(\mathbf X),\|\cdot\|_F,u\}
 \leq \frac{A_N}{u^2}.
\end{equation}

\begin{theorem}[Estimation error under empirical \(L_2\) entropy]
\label{Risk_error}
Under (A1)--(A2) and~\eqref{eq:empirical_entropy}, there are universal
constants \(C,c>0\) such that, for every
\(\eta>0\) satisfying (A2), with probability at least
\(1-(c+c_G)e^{-\eta}\),
\begin{align}
 R_{\mathcal F_N}(\widehat f_N)
 \leq C\bigg[&
 \frac{\sqrt{A_N}}{G_\star\sqrt K\,N}
 \left\{1+\log_+\!\left(
 \frac{M_f\sqrt K\,N}{\sqrt{A_N}}\right)\right\}\notag\\
 &+
 \frac{M_f}{G_\star}
 \left\{\sqrt{\frac{\eta+1}{N}}+\frac{\eta+1}{N}\right\}
 +
 \frac{M_f\kappa_N(\eta)}{G_\star^2}
 \bigg],                                            \label{eq:risk_bound}
\end{align}
where \(\log_+(x)=\max\{\log x,0\}\).
\end{theorem}

The proof uses Cauchy--Schwarz to transfer the empirical Frobenius metric
to the true-weight IPCW loss class, followed by symmetrization, a truncated
Dudley entropy integral, bounded empirical-process concentration, and the
deterministic plug-in comparison on the event in (A2).  It is given in the
Supplementary Material.

Writing \(q=(q_1,\ldots,q_K)\) for the true conditional quantiles of \(V\)
and \(\mathcal A_N=\inf_{f\in\mathcal F_N}\{L(f)-L(q)\}\), the total excess
risk decomposes exactly as \(L(\widehat f_N)-L(q)=R_{\mathcal F_N}(\widehat
f_N)+\mathcal A_N\), so Theorem~\ref{Risk_error} controls the estimation term
and is enough on its own when \(q\in\mathcal F_N\).  The two results below
instead bound \(\mathcal A_N\) for explicit growing KAN and Trans-CNQ sieves
and balance it against the estimation term, without assuming that a fixed
fitted architecture contains \(q\).  Each combines Theorem~\ref{Risk_error}
with its own approximation construction; the Lipschitz constant of the
raw-to-non-crossing map \(\Phi\) in \eqref{eq:noncrossing_output}, the
complexity-indexed instantiations for fixed architectures and norm budgets
(both giving an estimation rate \(O_{\mathbb P}(\log N/\sqrt N)\)), the
supporting lemmas and the complete proofs are given in Supplementary
Section~4.

\begin{theorem}[Excess-risk convergence rate for the KAN-CNQ estimator]
\label{thm:kan-holder-rate}
Suppose \(X\in[0,1]^p\), each true conditional log-time quantile is in a
bounded \(C^\beta\) ball, and adjacent quantiles are separated uniformly
by a positive constant.  Consider the fixed-grid, two-layer spline-KAN
sieve constructed in the Supplement: its first-layer spline degree is at
least two with \(\beta\leq m+1\), its second-layer dictionary contains
\(t\mapsto t^p\), all effective edge coefficients are bounded, and the
class has a fixed output envelope.  For spline resolution \(J\), this full
class has \(P_J=O(J^p)\) free coefficients and
\begin{equation}\label{eq:kan-total-main}
 L(\widehat f_{N,J})-L(q)
 =O_{\mathbb P}\left\{
 J^{-\beta}+\sqrt{\frac{J^p\log(2J)}{N}}+\frac1{\sqrt N}\right\}.
\end{equation}
Consequently, \(J_N\asymp\{N/\log N\}^{1/(2\beta+p)}\) gives
\[
 L(\widehat f_{N,J_N})-L(q)
 =O_{\mathbb P}\left[
 \left\{\frac N{\log N}\right\}^{-\beta/(2\beta+p)}\right].
\]
\end{theorem}

\begin{theorem}[Excess-risk convergence rate for the Trans-CNQ estimator]
\label{thm:trans-holder-rate}
Under the smoothness and positive-gap conditions of
Theorem~\ref{thm:kan-holder-rate}, suppose additionally that
\(0<\beta<(p+3)/2\).  Hold the shared-embedding post-LayerNorm encoder
dimensions fixed and let the flattened ReLU readout have at most \(KM\)
units and output variation budget \(\Lambda\).  For
\[
 \nu=\frac{p+3-2\beta}{2p},\qquad
 \Lambda_N=M_N^\nu,\qquad
 M_N\asymp\left\{\frac N{\log N}\right\}^{p/(2\beta+p)},
\]
the corresponding fixed-envelope Trans-CNQ sieve satisfies
\[
 L(\widehat f_{N,M_N,\Lambda_N})-L(q)
 =O_{\mathbb P}\left[
 \left\{\frac N{\log N}\right\}^{-\beta/(2\beta+p)}\right].
\]
\end{theorem}

\begin{remark}[Scope of the bounds]\label{rem:scope}
The sieve rates of Theorems~\ref{thm:kan-holder-rate}--\ref{thm:trans-holder-rate}
hold for fixed \(p\) and \(K\) under hard norm bounds, a
fixed output envelope, ideal ERM, uniformly positive quantile gaps, and
(A1)--(A2).  In particular, (A2) fails for bounded-uniform censoring combined
with an unbounded event-time support.  The KAN \emph{approximation} result,
and hence the KAN total rate, concerns a degree-\(p\) theoretical sieve rather
than the fitted cubic-spline EfficientKAN architecture, and none of these
results covers TransKAN-CNQ.  These are upper bounds and do not rank the
fitted architectures.
\end{remark}


\section{Experimental Setup}
\label{sec:experiments}

\subsection{Evaluation Metrics}
\label{sec:metrics}

Using the notation established in Section~\ref{sec:method}, quantile prediction errors are evaluated on the log-time scale, with $V_i=\log(Y_i)$ and $\widehat{V}_i(\tau_k)=\log\{\widehat{Q}_i(\tau_k)\}$. Interval coverage, by contrast, is assessed on the original time scale through the predicted quantiles $\widehat{Q}_i(\tau)$. Metrics that target population quantities defined in terms of the event time $T_i$ require correction for right censoring. We use IPCW weights $w_i = \delta_i / \widehat{G}(Y_i^-)$, where $\widehat{G}$ is the Kaplan--Meier estimate of the censoring survival function. All weighted metrics below use the normalized form $\sum_{i} w_i \ell_i \big/ \sum_{i} w_i$, which is consistent under the marginal independent-censoring and positivity conditions in (A1)–(A2).

{\it (i) IPCW-weighted mean pinball loss.} The pinball loss at level $\tau$ is $\rho_{\tau}(u)=u\{\tau-\mathbb{I}(u<0)\}$. The IPCW-weighted pinball loss averaged over quantiles is
\begin{equation}
\label{eq:wpinball}
\mathcal{L}_{\mathrm{wPB}}
=
\frac{1}{K}\sum_{k=1}^K
\frac{\sum_{i=1}^N w_i\;
\rho_{\tau_k}\!\big(V_i-\widehat{V}_i(\tau_k)\big)}
{\sum_{i=1}^N w_i}\,.
\end{equation}
This is the primary measure of overall quantile prediction accuracy. The quantile-specific pinball losses $\mathcal{L}_{\mathrm{wPB}}(\tau_k)$ (i.e., the individual terms before averaging over $k$) are reported in the supplementary appendix to reveal potential heterogeneity across quantile levels.

{\it (ii) IPCW-weighted interval coverage probability (ICP).} To assess interval calibration under censoring, we propose the IPCW-weighted
empirical coverage (wICP) for a nominal $(1-2\alpha)\times 100\%$ prediction
interval $[\widehat{Q}_i(\alpha),\, \widehat{Q}_i(1-\alpha)]$:
\begin{equation}
\label{eq:wicp}
\mathrm{wICP}_{1-2\alpha}
=
\frac{\sum_{i=1}^N w_i\;\mathbb{I}\!\big\{Y_i \in [\widehat{Q}_i(\alpha),\, \widehat{Q}_i(1-\alpha)]\big\}}
{\sum_{i=1}^N w_i}\,.
\end{equation}
We report coverage at two nominal levels: 80\% ($\alpha = 0.1$, using $\tau = 0.1$ and $\tau = 0.9$) and 50\% ($\alpha = 0.25$, using $\tau = 0.25$ and $\tau = 0.75$). A well-calibrated model should achieve empirical coverage close to the nominal rate.

\subsection{Implementation Setup}
\label{sec:setup}

\textbf{Training and hyperparameter tuning.}
All models are trained using the AdamW optimizer~\citep{loshchilov2017decoupled} under a learning rate schedule combining linear warmup with cosine annealing (CosineAnnealingLR), and training is terminated via early stopping on the validation loss. For the Transformer-based models (Trans-CNQ and TransKAN-CNQ) we search over learning rates $\{10^{-3},\,\allowbreak 5{\times}10^{-4},\,\allowbreak 10^{-4},\,\allowbreak 5{\times}10^{-5},\,\allowbreak 10^{-5}\}$, dropout rates $\{0,\,\allowbreak 0.2,\,\allowbreak 0.5\}$, weight decays $\{0,\,\allowbreak 10^{-4},\,\allowbreak 10^{-3}\}$, hidden layer widths $\{64,\,\allowbreak 100,\,\allowbreak 128,\,\allowbreak 200,\,\allowbreak 256\}$ and network depths $\{2,\,\allowbreak 3\}$, with an early-stopping patience of 10 epochs. KAN-CNQ was tuned in a separate sweep over learning rates $\{10^{-3},\,\allowbreak 5{\times}10^{-4},\,\allowbreak 3{\times}10^{-4},\,\allowbreak 10^{-4}\}$, dropout rates $\{0,\,\allowbreak 0.05,\,\allowbreak 0.1,\,\allowbreak 0.2\}$, weight decays $\{0,\,\allowbreak 10^{-5},\,\allowbreak 10^{-4},\,\allowbreak 10^{-3}\}$, hidden layer widths $\{32,\,\allowbreak 64,\,\allowbreak 128,\,\allowbreak 200,\,\allowbreak 256\}$, network depths $\{1,\,\allowbreak 2,\,\allowbreak 3\}$ and B-spline grid sizes $\{3,\,\allowbreak 5,\,\allowbreak 8\}$, with an early-stopping patience of 20 epochs; the separate sweep reflects the spline-resolution hyperparameter, which has no Transformer counterpart. In all cases the configuration with the best validation performance is selected for final evaluation.

\textbf{Data splitting.}
For the simulation study, we partition the data into training, validation, and test sets in a 1:1:1 ratio. For real-world datasets, we adopt a 65\%--15\%--20\% split following standard practice. All experiments are repeated over 25 random seeds to account for variability in data partitioning.

\textbf{Baseline methods.}
We compare the three proposed models (TransKAN-CNQ, Trans-CNQ, and KAN-CNQ) against ten existing methods in three categories:
(a)~quantile regression-based deep learning models: DeepQuantreg~\citep{jia2022deep} and CQRNN~\citep{pearce2022censored};
(b)~non-quantile deep learning models: DeepSurv~\citep{katzman2018deepsurv}, CoxTime, CoxCC, and PCHazard~\citep{kvamme2019time};
(c)~traditional survival models: Random Survival Forests (RSF)~\citep{ishwaran2008random,ishwaran2007random,ishwaran2019fast}, censored quantile regression via the \texttt{ctqr} package in R~\citep{ctqr,frumento2017estimating}, the Cox proportional hazards model (CoxPH)~\citep{cox1972regression}, and the accelerated failure time model (AFT)~\citep{kalbfleisch2002statistical,kleinbaum1996survival}.
CTQR represents the classical linear censored quantile regression estimators of \citet{portnoy2003censored} and \citet{peng2008survival} in this comparison: all three model the conditional quantile as linear in the covariates and share the linearity restriction that motivates our framework.  CTQR ranks in the bottom three of the thirteen methods in five of the six cohorts and in $18$ of the $27$ simulation settings, so the two additional linear estimators are unlikely to alter any conclusion below.  All deep quantile baselines follow the same training protocol as the proposed models (identical IPCW weights, optimizer, schedule and tuning budget), with DeepQuantreg fitting one MLP per quantile level as originally proposed, and all ten are evaluated in both the simulation study and the real-data benchmark.

\section{Simulation Studies}
\label{sec:sim}
To complement the real-data analysis, we ran a simulation study with covariate dimension fixed at $p=10$. The design spans $27$ settings formed by three event-time distributions (Gaussian, Gamma, Weibull) with covariate-dependent parameters, three censoring levels ($\approx 25\%$, $50\%$, $75\%$), and three sample sizes ($n\in\{150,750,1500\}$).
Full design details, per-setting results, and tables are provided in the Supplementary Material; we summarize the conclusions here.

The results are strongly distribution-dependent. Under the Weibull and Gamma designs, where covariates reshape the full survival distribution rather than only its center, the CNQ models attain the lowest IPCW pinball loss (in $7/9$ and $6/9$ settings, respectively), with the largest margins at moderate-to-large samples. Under the symmetric Gaussian design, by contrast, the quantile- and hazard-based methods perform very similarly and Random Survival Forest (RSF) attains the lowest pinball loss in eight of the nine settings, though only marginally. RSF is thus the strongest competitor overall, leading on the Gaussian design and running a close second under Weibull/Gamma, while the proposed models lead precisely where flexible distributional modeling matters most. Among the quantile-based baselines, CQRNN is less accurate than the best proposed model in every one of the 27 settings and DeepQuantreg in 26 of 27, the sole exception being Gamma with $75\%$ censoring at $n=150$; the hazard-based and remaining classical baselines are weaker distributional predictors. Within our family, TransKAN-CNQ is the most consistent, with Trans-CNQ a close and sometimes preferable alternative under small samples or heavy censoring; at the smallest sample size ($n=150$) the proposed models' advantage narrows. Importantly for interpretability, the CNQ parameterization makes \emph{quantile crossing impossible by construction}, and the measured crossing rate of all three architectures is exactly zero in every setting. Fitting the same five levels separately, as DeepQuantreg does, is not so benign: averaged over the 27 settings, $34.2\%$ of test subjects have at least one inverted adjacent pair and $2.1\%$ have $\widehat q_{0.1}>\widehat q_{0.9}$, rising to $65.2\%$ and $29.6\%$ in the hardest setting (Supplementary Material).

\section{Real Data Analysis}
\label{sec:realdata}


\begin{table}[h!]
\centering
\setlength{\tabcolsep}{4pt} 
\renewcommand{\arraystretch}{1.2} 
\begin{tabular}{|p{2.2cm}|p{1.3cm}|p{2.2cm}|p{2.3cm}|}
\hline
\textbf{Data set} & \textbf{Size} & \textbf{Covariates} & \textbf{Censored\%} \\
\hline
SUPPORT  & 8,873  & 14 & 32 \\
\hline
METABRIC & 1,904  & 9  & 42 \\
\hline
GBSG     & 2,232  & 7  & 43 \\
\hline
GBSG-500  & 500    & 7  & 42 \\
\hline
NKI70    & 144    & 8  & 67 \\
\hline
FLCHAIN  & 7,874  & 7  & 72 \\
\hline
\end{tabular}
\caption{Summary of datasets used in the analysis.}
\label{tab:dataset_summary} 
\end{table}

\subsection{Overall Predictive Performance}
\label{sec:overall}


We compare the three proposed architectures (TransKAN-CNQ, Trans-CNQ and KAN-CNQ) with ten competing approaches on the six right-censored cohorts of Table~\ref{tab:dataset_summary}, which span two orders of magnitude in sample size ($n=144$ to $8{,}873$) and censoring from $32\%$ to $72\%$.

\textbf{Pinball loss.}
Table~\ref{tab:overall-pinball} reports the mean IPCW pinball loss across 25 random splits. The proposed methods attain the lowest mean pinball loss on all six datasets, with the strongest results typically delivered by the two Transformer-based variants. DeepQuantreg is the strongest baseline on four of the six cohorts, which is expected given that it is the only other method fitting the quantile function directly; RSF is strongest on METABRIC, while on the 144-subject NKI70 cohort the baselines are effectively tied (CoxPH lowest at $0.265$, with all of them within $0.005$). On the large SUPPORT cohort, TransKAN-CNQ and Trans-CNQ achieve $\overline{L}_{\mathrm{pin}}=0.486$ and $0.487$, against the best baseline (DeepQuantreg, $0.492$) and the best non-quantile baseline (RSF, $0.541$); on METABRIC, TransKAN-CNQ is best ($0.215$) ahead of RSF ($0.228$) and DeepQuantreg ($0.243$). The largest gap is on the heavily censored FLCHAIN cohort, where TransKAN-CNQ attains $0.279$ versus DeepQuantreg $0.287$ and RSF $0.332$, a relative improvement of roughly $16\%$ over the latter. On GBSG and GBSG-500 the three proposed models lie within $0.006$ of one another and their internal ranking shifts with the sample, as it does on NKI70 (Trans-CNQ best at $0.242$), where the larger standard deviations reflect the instability expected from 144 subjects; DeepQuantreg is the weakest quantile-based method there ($0.305$). The consistent ordering of the proposed models ahead of DeepQuantreg on every cohort isolates the contribution of joint multi-quantile estimation, since the two approaches share the IPCW objective and differ chiefly in fitting all quantiles together under a non-crossing parameterization rather than one network per level.
Figure~\ref{fig:realdata-pinball} complements Table~\ref{tab:overall-pinball} by showing the split-to-split distribution of IPCW pinball loss across all six cohorts; the proposed models generally occupy the lowest-loss region, with DeepQuantreg the closest quantile-based competitor.

\textbf{Coverage.}
Table~\ref{tab:overall-icp} reports the IPCW interval coverage probability (ICP) for the nominal 80\% prediction interval. The proposed methods generally remain closer to the nominal target than most competitors while avoiding the systematic overcoverage seen in several baselines. The one baseline that matches them on this criterion is DeepQuantreg, which is closest to nominal on GBSG ($0.793$), SUPPORT ($0.793$) and FLCHAIN ($0.802$). This is unsurprising, since it shares the IPCW quantile objective and therefore inherits the same width calibration, and on coverage alone the two are effectively tied: where DeepQuantreg is nearer the nominal level its advantage in absolute deviation is $0.001$ on GBSG, $0.003$ on FLCHAIN and $0.010$ on SUPPORT, all well inside the split-to-split standard deviations of $0.02$--$0.06$. Where the proposed models are nearer, the margin is larger: $0.010$ on GBSG-500, $0.024$ on METABRIC and $0.170$ on NKI70, where DeepQuantreg degrades to $0.625$. Coverage therefore does not separate the two approaches; what separates them is that the proposed models attain this calibration while also achieving strictly lower pinball loss on all six cohorts (Table~\ref{tab:overall-pinball}), and while guaranteeing a coherent conditional distribution. Because DeepQuantreg fits an independent network per quantile level, nothing in its construction prevents the estimated curves from crossing, and on these cohorts they do: at least one adjacent pair is inverted for $18.2\%$ of test subjects on average, ranging from $2.4\%$ on SUPPORT to $45.7\%$ on NKI70, and $146$ of the $150$ fits contain at least one such subject (Supplementary Material). The outer interval itself inverts only rarely here ($\widehat q_{0.1}>\widehat q_{0.9}$ for $0.18\%$ of subjects), so its coverage in Table~\ref{tab:overall-icp} remains interpretable; but the underlying quantile function is not monotone for a substantial minority of patients, which is precisely the coherence a clinician needs when reading a set of quantile milestones. The CNQ parameterization rules this out identically --- its measured crossing rate is zero on every cohort --- at no measurable cost in coverage.

Among the baselines the failure modes are systematic. AFT and RSF reach nominal-looking coverage on SUPPORT ($0.898$ and $0.888$) only through very wide intervals, whereas all three proposed models sit near $0.78$ there; CQRNN overcovers on most cohorts ($0.906$ on METABRIC, $0.925$ on FLCHAIN, $0.881$ on SUPPORT), so its nominally wide intervals are not sharp, the exception being NKI70 ($0.649$), where the small sample destabilizes every method. On GBSG the proposed models are tightly concentrated around the nominal level (Trans-CNQ, $0.808$). Coverage is most variable on METABRIC, where they span $0.720$--$0.784$; TransKAN-CNQ is closest to nominal there, just ahead of CoxCC ($0.781$). These estimates are the least stable of the six cohorts, because METABRIC has by far the smallest effective sample size under IPCW weighting ($(\sum_i w_i)^2/\sum_i w_i^2 \approx 35\%$ of the observed events, versus $75$--$98\%$ elsewhere), so a handful of late events carries a large share of the weight; capping the weights at their 95th percentile moves the three proposed models to $0.795$--$0.820$ with essentially unchanged pinball loss (Supplementary Material), but we report untruncated weights throughout. On FLCHAIN the proposed models remain close to nominal (TransKAN-CNQ $0.795$, Trans-CNQ $0.806$, KAN-CNQ $0.828$), as does DeepQuantreg ($0.802$), whereas the hazard-based and traditional baselines undercover markedly (e.g., RSF $0.693$, CoxTime $0.593$), further evidence that an IPCW quantile objective translates nominal width into reliable empirical coverage even under heavy censoring.

Overall, the CNQ models deliver the most favorable calibration--sharpness trade-off among the methods considered: the lowest pinball loss on every cohort, with empirical coverage close to nominal outside METABRIC, rather than apparent calibration bought with overly conservative intervals.

\subsection{Model Checking and Calibration}
\label{sec:model-checking}
Aggregate loss and interval coverage do not by themselves establish that a fitted quantile model is internally well calibrated. We therefore add a direct calibration check on METABRIC. Using the held-out test predictions from all 25 splits, we estimate, for each nominal level $\tau$, the IPCW-corrected probability $\widehat{\Pr}\{T\le \widehat{Q}_\tau(X)\}$, which equals $\tau$ under perfect calibration; censoring is handled with the same inverse-probability weights (training-set Kaplan--Meier $\widehat{G}$) used for the loss metrics.

Figure~\ref{fig:calibration}(a) plots these estimates against the nominal levels. All three architectures track the $45^\circ$ line closely: $\widehat{\Pr}\{T\le\widehat{Q}_{0.5}\}$ is $0.49$, $0.47$ and $0.48$ for Trans-CNQ, TransKAN-CNQ and KAN-CNQ, and at $\tau=0.1$ the three give $0.11$, $0.10$ and $0.10$ against the nominal $0.10$.  The residual miscalibration is concentrated in the upper tail, where all three undershoot ($0.83$, $0.88$ and $0.86$ at $\tau=0.9$). Figure~\ref{fig:calibration}(b) summarizes central-interval calibration: empirical IPCW coverage of the nominal $50\%$ and $80\%$ intervals is $0.43$--$0.47$ and $0.72$--$0.78$, respectively, a mild but consistent undercoverage that is most pronounced for Trans-CNQ and matches the interval-coverage results in Table~\ref{tab:overall-icp}. A log-time residual diagnostic at the predicted median (Supplementary Material) shows median residuals near zero for the Transformer-based models. The same checks replicate on FLCHAIN (Supplementary Material), where calibration is tighter (empirical coverage of the nominal $80\%$ interval is $0.80$--$0.83$) consistent with its substantially larger test sets. These checks confirm that the calibration advantages reported above reflect genuinely well-behaved conditional quantiles rather than averaging artifacts, and they localize the residual miscalibration to the extreme quantiles of the spline-based architecture.

\textbf{Sensitivity to the quantile grid.}
To verify that our conclusions do not hinge on the choice of five quantile levels, we retrained all three models with denser grids of $K\in\{11,19\}$ quantiles (identical hyperparameters, 25 splits) and re-evaluated at the five common levels. Table~\ref{tab:qgrid-sensitivity} shows the pinball loss and the mean absolute calibration error (MACE) changing by at most $\sim\!0.01$ as $K$ grows, well within one split-to-split standard deviation ($\approx\!0.02$), with the ordering (TransKAN-CNQ $\lesssim$ Trans-CNQ $<$ KAN-CNQ) preserved at every resolution. Varying the IPCW weight-truncation threshold (none through the $90$th--$99$th percentiles) and the train/validation/test ratio ($65/15/20$, $50/25/25$, $80/10/10$) likewise leaves the pinball loss within $\sim\!0.01$ and the ordering intact, with $80\%$ coverage staying close to nominal (Supplementary Material); only the smallest test set ($80/10/10$) is noticeably noisier. The five-quantile, untruncated-weight results reported above are therefore not artifacts of those choices.

\begin{table}[t]
  \centering
  \caption{Quantile-grid sensitivity on METABRIC. Models are trained with $K\in\{5,11,19\}$ quantiles and evaluated at the five common levels $\{0.1,0.25,0.5,0.75,0.9\}$ (mean over 25 splits). $\overline{L}_{\mathrm{pin}}$ is the IPCW pinball loss (log scale); MACE is the mean absolute calibration error. The baseline is $K=5$.}
  \label{tab:qgrid-sensitivity}
  \small
  \begin{tabular}{l ccc c ccc}
    \toprule
    & \multicolumn{3}{c}{$\overline{L}_{\mathrm{pin}}$} & & \multicolumn{3}{c}{MACE} \\
    \cmidrule(lr){2-4}\cmidrule(lr){6-8}
    Model & $K{=}5$ & $K{=}11$ & $K{=}19$ & & $K{=}5$ & $K{=}11$ & $K{=}19$ \\
    \midrule
    TransKAN-CNQ & 0.215 & 0.226 & 0.223 & & 0.079 & 0.109 & 0.090 \\
    Trans-CNQ & 0.224 & 0.230 & 0.223 & & 0.092 & 0.102 & 0.091 \\
    KAN-CNQ & 0.227 & 0.231 & 0.231 & & 0.089 & 0.098 & 0.099 \\\bottomrule
  \end{tabular}
\end{table}

\begin{figure}[t]
  \centering
  \begin{tabular}{cc}
    \includegraphics[width=0.46\textwidth]{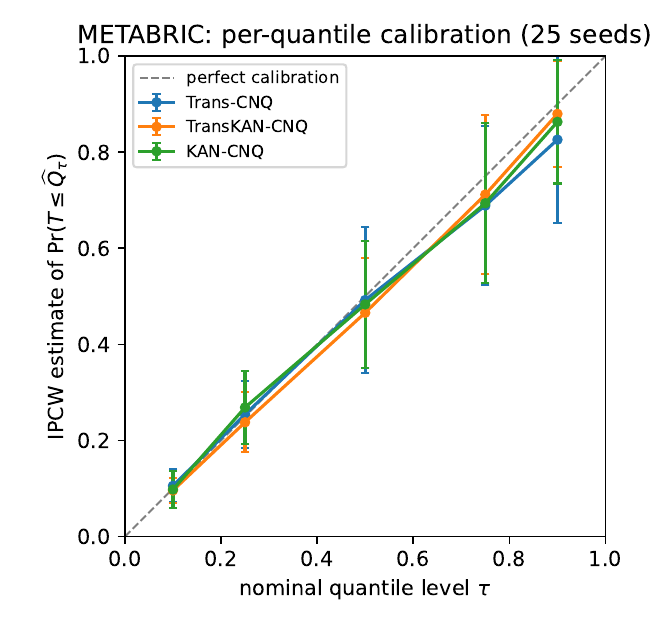} &
    \includegraphics[width=0.46\textwidth]{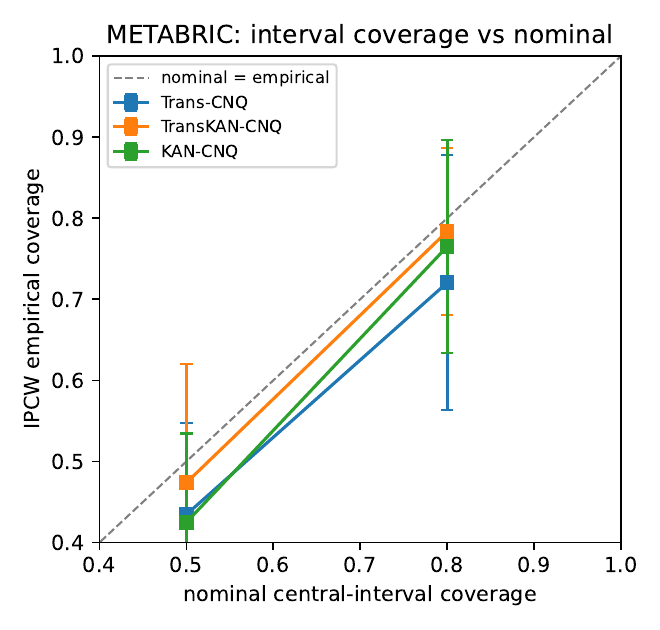} \\[-2pt]
    \small (a) Per-quantile calibration &
    \small (b) Central-interval coverage
  \end{tabular}
  \caption{METABRIC model checking (25 seeds, IPCW-weighted).
  (a)~Per-quantile calibration: for each nominal level $\tau$, the IPCW estimate of $\Pr\{T\le\widehat{Q}_\tau(X)\}$ is plotted against $\tau$; the dashed line is perfect calibration and error bars are $\pm1$ standard deviation across splits.
  (b)~Central-interval coverage versus nominal level ($50\%$ and $80\%$); points below the diagonal indicate mild undercoverage.}
  \label{fig:calibration}
\end{figure}

\begin{figure}
    \centering
    \includegraphics[width=\linewidth]{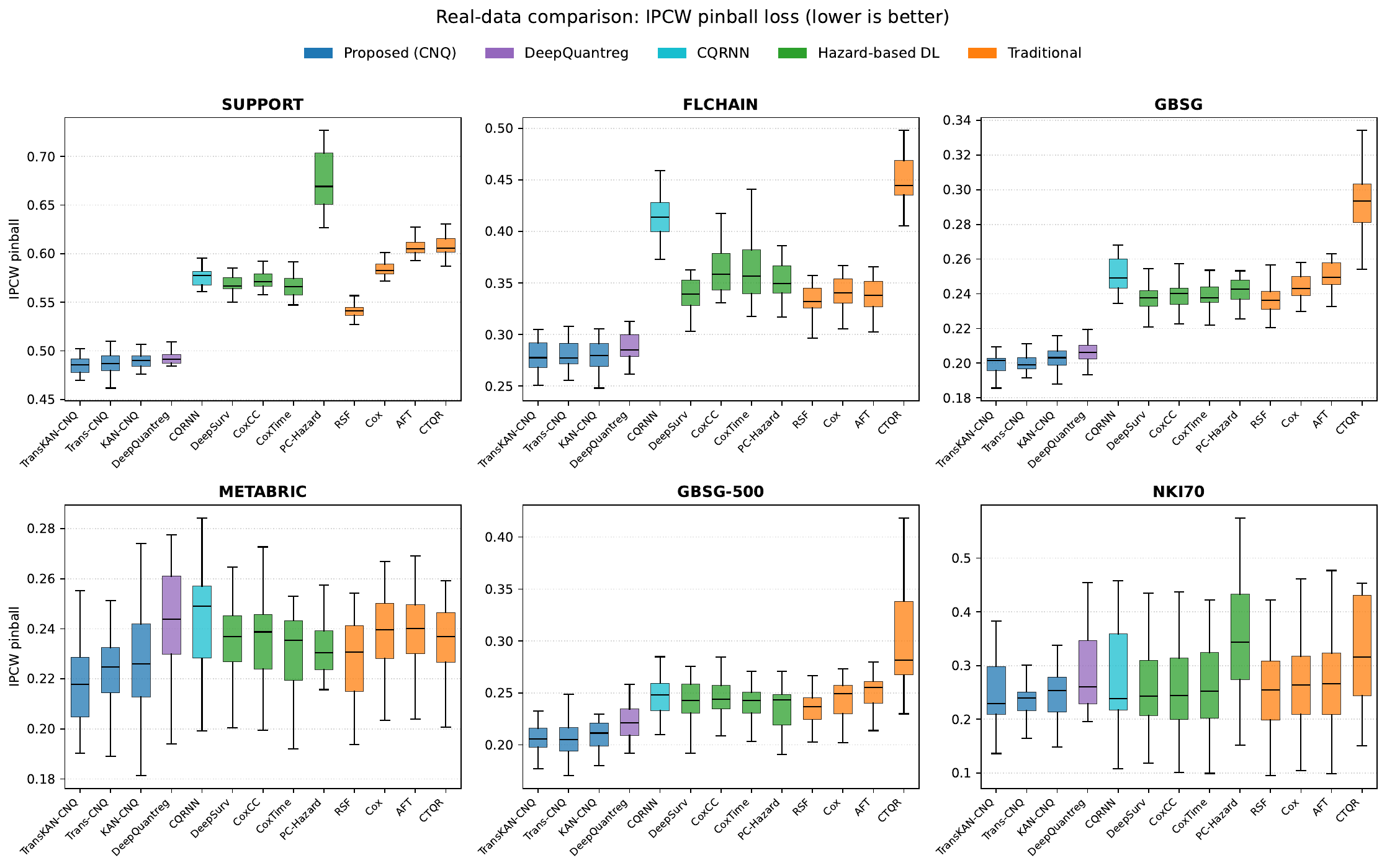}
    \caption{Boxplots of IPCW pinball loss (self-normalized, log-time scale) across the six real-world datasets. Lower values indicate better quantile predictive performance. The proposed CNQ models (blue) consistently occupy the most favorable region of the distribution, with DeepQuantreg (purple), the other quantile-based deep learning method, the closest competitor on most cohorts.}
    \label{fig:realdata-pinball}
\end{figure}

\begin{table}[t]
  \centering
  \caption{Mean IPCW pinball loss ($\overline{L}_{\mathrm{pin}}$, averaged over $\tau \in \{0.1, 0.25, 0.5, 0.75, 0.9\}$) on six real-world datasets. Each entry shows the mean $\pm$ standard deviation across 25 random splits. Lower is better. The best result per dataset is \textbf{bolded}.}
  \label{tab:overall-pinball}
  \small
  \setlength{\tabcolsep}{3.5pt}
  \begin{tabular}{l cccccc}
    \toprule
    & GBSG & GBSG-500 & METABRIC & SUPPORT & NKI70 & FLCHAIN \\
    \midrule
    \multicolumn{7}{l}{\emph{Proposed methods}} \\
    TransKAN-CNQ   & \textbf{.199} {\scriptsize $\pm$ .006} & .206 {\scriptsize $\pm$ .015} & \textbf{.215} {\scriptsize $\pm$ .022} & \textbf{.486} {\scriptsize $\pm$ .011} & .251 {\scriptsize $\pm$ .066} & \textbf{.279} {\scriptsize $\pm$ .015} \\
    Trans-CNQ      & .200 {\scriptsize $\pm$ .006} & \textbf{.204} {\scriptsize $\pm$ .017} & .224 {\scriptsize $\pm$ .024} & .487 {\scriptsize $\pm$ .010} & \textbf{.242} {\scriptsize $\pm$ .044} & .282 {\scriptsize $\pm$ .015} \\
    KAN-CNQ        & .203 {\scriptsize $\pm$ .007} & .210 {\scriptsize $\pm$ .015} & .227 {\scriptsize $\pm$ .025} & .490 {\scriptsize $\pm$ .010} & .251 {\scriptsize $\pm$ .058} & .280 {\scriptsize $\pm$ .016} \\
    \addlinespace
    \multicolumn{7}{l}{\emph{Deep learning baselines}} \\
    DeepQuantreg   & .206 {\scriptsize $\pm$ .007} & .220 {\scriptsize $\pm$ .018} & .243 {\scriptsize $\pm$ .024} & .492 {\scriptsize $\pm$ .009} & .305 {\scriptsize $\pm$ .100} & .287 {\scriptsize $\pm$ .015} \\
    CQRNN          & .251 {\scriptsize $\pm$ .010} & .248 {\scriptsize $\pm$ .020} & .244 {\scriptsize $\pm$ .023} & .576 {\scriptsize $\pm$ .010} & .291 {\scriptsize $\pm$ .133} & .412 {\scriptsize $\pm$ .021} \\
    DeepSurv       & .239 {\scriptsize $\pm$ .010} & .242 {\scriptsize $\pm$ .020} & .235 {\scriptsize $\pm$ .016} & .568 {\scriptsize $\pm$ .009} & .267 {\scriptsize $\pm$ .097} & .339 {\scriptsize $\pm$ .018} \\
    CoxCC          & .241 {\scriptsize $\pm$ .010} & .245 {\scriptsize $\pm$ .020} & .236 {\scriptsize $\pm$ .017} & .573 {\scriptsize $\pm$ .009} & .266 {\scriptsize $\pm$ .104} & .364 {\scriptsize $\pm$ .026} \\
    CoxTime        & .240 {\scriptsize $\pm$ .009} & .241 {\scriptsize $\pm$ .020} & .230 {\scriptsize $\pm$ .019} & .567 {\scriptsize $\pm$ .013} & .266 {\scriptsize $\pm$ .101} & .374 {\scriptsize $\pm$ .058} \\
    PC-Hazard      & .242 {\scriptsize $\pm$ .010} & .235 {\scriptsize $\pm$ .021} & .231 {\scriptsize $\pm$ .019} & .673 {\scriptsize $\pm$ .031} & .360 {\scriptsize $\pm$ .131} & .351 {\scriptsize $\pm$ .018} \\
    \addlinespace
    \multicolumn{7}{l}{\emph{Traditional methods}} \\
    Cox PH         & .246 {\scriptsize $\pm$ .011} & .244 {\scriptsize $\pm$ .019} & .239 {\scriptsize $\pm$ .017} & .585 {\scriptsize $\pm$ .009} & .265 {\scriptsize $\pm$ .089} & .341 {\scriptsize $\pm$ .018} \\
    AFT            & .252 {\scriptsize $\pm$ .011} & .249 {\scriptsize $\pm$ .019} & .239 {\scriptsize $\pm$ .019} & .608 {\scriptsize $\pm$ .011} & .267 {\scriptsize $\pm$ .091} & .339 {\scriptsize $\pm$ .019} \\
    RSF            & .237 {\scriptsize $\pm$ .009} & .234 {\scriptsize $\pm$ .016} & .228 {\scriptsize $\pm$ .017} & .541 {\scriptsize $\pm$ .007} & .269 {\scriptsize $\pm$ .114} & .332 {\scriptsize $\pm$ .017} \\
    CTQR           & .295 {\scriptsize $\pm$ .019} & .327 {\scriptsize $\pm$ .150} & .235 {\scriptsize $\pm$ .017} & .609 {\scriptsize $\pm$ .013} & .418 {\scriptsize $\pm$ .300} & .449 {\scriptsize $\pm$ .023} \\
    \bottomrule
  \end{tabular}
\end{table}

\begin{table}[t]
  \centering
  \caption{IPCW interval coverage probability (ICP) at the 80\% nominal level on six real-world datasets. Each entry shows the mean $\pm$ standard deviation across 25 random splits. Values closest to the nominal 0.80 are \textbf{bolded}.}
  \label{tab:overall-icp}
  \small
  \setlength{\tabcolsep}{3.5pt}
  \begin{tabular}{l cccccc}
    \toprule
    & GBSG & GBSG-500 & METABRIC & SUPPORT & NKI70 & FLCHAIN \\
    \midrule
    \multicolumn{7}{l}{\emph{Proposed methods}} \\
    TransKAN-CNQ   & .789 {\scriptsize $\pm$ .038} & \textbf{.802} {\scriptsize $\pm$ .085} & \textbf{.784} {\scriptsize $\pm$ .105} & .782 {\scriptsize $\pm$ .030} & .690 {\scriptsize $\pm$ .297} & .795 {\scriptsize $\pm$ .061} \\
    Trans-CNQ      & .808 {\scriptsize $\pm$ .037} & .787 {\scriptsize $\pm$ .073} & .720 {\scriptsize $\pm$ .160} & .778 {\scriptsize $\pm$ .043} & .759 {\scriptsize $\pm$ .229} & .806 {\scriptsize $\pm$ .055} \\
    KAN-CNQ        & .784 {\scriptsize $\pm$ .039} & .818 {\scriptsize $\pm$ .080} & .765 {\scriptsize $\pm$ .134} & .783 {\scriptsize $\pm$ .024} & \textbf{.805} {\scriptsize $\pm$ .227} & .828 {\scriptsize $\pm$ .042} \\
    \addlinespace
    \multicolumn{7}{l}{\emph{Deep learning baselines}} \\
    DeepQuantreg   & \textbf{.793} {\scriptsize $\pm$ .046} & .788 {\scriptsize $\pm$ .057} & .759 {\scriptsize $\pm$ .075} & \textbf{.793} {\scriptsize $\pm$ .024} & .625 {\scriptsize $\pm$ .233} & \textbf{.802} {\scriptsize $\pm$ .039} \\
    CQRNN          & .871 {\scriptsize $\pm$ .025} & .851 {\scriptsize $\pm$ .054} & .906 {\scriptsize $\pm$ .036} & .881 {\scriptsize $\pm$ .017} & .649 {\scriptsize $\pm$ .245} & .925 {\scriptsize $\pm$ .024} \\
    DeepSurv       & .849 {\scriptsize $\pm$ .028} & .826 {\scriptsize $\pm$ .039} & .765 {\scriptsize $\pm$ .120} & .872 {\scriptsize $\pm$ .012} & .773 {\scriptsize $\pm$ .191} & .706 {\scriptsize $\pm$ .024} \\
    CoxCC          & .844 {\scriptsize $\pm$ .025} & .825 {\scriptsize $\pm$ .043} & .781 {\scriptsize $\pm$ .121} & .869 {\scriptsize $\pm$ .011} & .766 {\scriptsize $\pm$ .212} & .632 {\scriptsize $\pm$ .063} \\
    CoxTime        & .844 {\scriptsize $\pm$ .024} & .829 {\scriptsize $\pm$ .043} & .770 {\scriptsize $\pm$ .123} & .862 {\scriptsize $\pm$ .012} & .769 {\scriptsize $\pm$ .182} & .593 {\scriptsize $\pm$ .170} \\
    PC-Hazard      & .818 {\scriptsize $\pm$ .029} & .806 {\scriptsize $\pm$ .046} & .674 {\scriptsize $\pm$ .142} & .839 {\scriptsize $\pm$ .021} & .276 {\scriptsize $\pm$ .214} & .712 {\scriptsize $\pm$ .023} \\
    \addlinespace
    \multicolumn{7}{l}{\emph{Traditional methods}} \\
    Cox PH         & .848 {\scriptsize $\pm$ .027} & .840 {\scriptsize $\pm$ .045} & .775 {\scriptsize $\pm$ .096} & .879 {\scriptsize $\pm$ .010} & .643 {\scriptsize $\pm$ .237} & .699 {\scriptsize $\pm$ .023} \\
    AFT            & .879 {\scriptsize $\pm$ .021} & .868 {\scriptsize $\pm$ .033} & .834 {\scriptsize $\pm$ .048} & .898 {\scriptsize $\pm$ .011} & .620 {\scriptsize $\pm$ .229} & .710 {\scriptsize $\pm$ .026} \\
    RSF            & .852 {\scriptsize $\pm$ .021} & .850 {\scriptsize $\pm$ .044} & .882 {\scriptsize $\pm$ .035} & .888 {\scriptsize $\pm$ .011} & .750 {\scriptsize $\pm$ .230} & .693 {\scriptsize $\pm$ .025} \\
    CTQR           & .839 {\scriptsize $\pm$ .026} & .821 {\scriptsize $\pm$ .076} & .744 {\scriptsize $\pm$ .146} & .860 {\scriptsize $\pm$ .009} & .526 {\scriptsize $\pm$ .235} & .690 {\scriptsize $\pm$ .024} \\
    \bottomrule
  \end{tabular}
\end{table}

\subsection{Overview of the Two Case Studies}
\label{sec:real-overview}

To evaluate the proposed censored non-crossing quantile regression framework beyond simulated settings, we conduct a systematic analysis on two real-world survival datasets that differ in clinical context, sample size, and censoring severity.

\paragraph{Datasets}
Table~\ref{tab:datasets} summarizes the two cohorts, whose covariates are listed in Section~\ref{s:data}. They were chosen for the contrast they provide: METABRIC pairs molecular markers with treatment indicators under moderate censoring, whereas FLCHAIN is a treatment-free population cohort four times larger and far more heavily censored.

\begin{table}[t]
  \centering
  \caption{Summary of the two real-data cohorts used for in-depth analysis.}
  \label{tab:datasets}
  \small
  \begin{tabular}{lccccl}
    \toprule
    Dataset & $n$ & Censoring & Outcome & Covariates & Domain \\
    \midrule
    METABRIC & 1{,}904 & 42\% & Overall survival (months) & 9 & Breast cancer \\
    FLCHAIN  & 7{,}874 & 72\% & Time to death (days)     & 7 & General population \\
    \bottomrule
  \end{tabular}
\end{table}

\paragraph{Evaluation protocol}
For both datasets, we employ 25 independent random partitions (seeds 41--65), each split into training, validation, and test sets.
Three architectures (Trans-CNQ, TransKAN-CNQ, and KAN-CNQ) each predict five non-crossing quantiles at $\tau \in \{0.1, 0.25, 0.5, 0.75, 0.9\}$.
All evaluation metrics use inverse-probability-of-censoring weighting (IPCW; \citealt{Gerds2006}), with Kaplan--Meier censoring weights estimated from the training set.
IPCW weights are used untruncated throughout; sensitivity to weight truncation is reported in the Supplementary Material.

\paragraph{Analysis plan}
Questions~\textbf{(Q1)} and~\textbf{(Q2)} concern how individualized quantile estimates and covariate effects behave within a cohort, and are addressed by the two case studies below; \textbf{(Q3)}, which concerns reliability in small and heavily censored cohorts, was addressed by the cross-cohort comparison in Section~\ref{sec:overall}. We organize the case-study results around four components:
\begin{enumerate}
  \item \textbf{Quantile-specific feature importance} \textbf{(Q2)} (Section~\ref{sec:metabric-fi} and \ref{sec:flchain-fi}):
        Permutation-based analysis measuring the increase in IPCW pinball loss when each covariate is shuffled, separately for every $\tau$ level.
        This reveals how covariate effects vary across the survival distribution.

  \item \textbf{Clinically defined group contrasts} \textbf{(Q2)} (Sections~\ref{sec:metabric-subgroup} and \ref{sec:flchain-subgroup}):
        Comparison of mean predicted quantiles across clinically defined patient groups.
        A non-constant contrast $\Delta(\tau)$ across quantile levels indicates distributional heterogeneity that would not be captured by a proportional-hazards summary.

  \item \textbf{Individual patient profiles} \textbf{(Q1)}:
        Visualization of complete quantile profiles for representative patients, demonstrating how the framework distinguishes patients with similar medians but different tail risks or uncertainty levels.
        A companion uncertainty-stratification analysis, reported in the Supplementary Material, tests whether the predicted interval width functions as an internal, label-free indicator of predictive difficulty.

  \item \textbf{Event projection} (beyond Q1--Q3; Section~\ref{sec:metabric-ep}): Given a
        calendar cutoff, how well does the model project the events accumulating beyond that
        horizon?
\end{enumerate}

METABRIC, which offers the richest combination of treatment, biomarker and prognostic heterogeneity, is the primary application. Additional patient-profile and uncertainty analyses are reported in the Supplementary Material.

\subsection{METABRIC}
\label{sec:metabric}

\subsubsection{Feature Importance}
\label{sec:metabric-fi}

For each covariate, we randomly permute its test-set values (50 repetitions per seed) and measure the increase $\Delta$ in IPCW pinball loss at each $\tau$.
A heatmap averaged over all three architectures and 25 seeds is provided in the Supplementary Material.
\textbf{Age} and \textbf{chemotherapy} dominate across all $\tau$, with effects concentrated at $\tau = 0.5$--$0.9$; their influence at $\tau = 0.1$ is modest, indicating that these factors primarily shape the bulk and upper tail of the survival distribution.
\textbf{Hormone therapy} peaks at $\tau = 0.75$ ($\Delta = 1.64$) and is negligible at $\tau = 0.1$ ($\Delta = 0.10$): it predicts long-term but not short-term survival. Because treatment is assigned by indication and the model does not condition on ER status, the direction of this effect is not identified (Section~\ref{sec:metabric-subgroup}).
\textbf{PGR} shows an inverted-U pattern peaking at $\tau = 0.5$ ($\Delta = 0.79$), differentiating patients primarily in the central survival range.
\textbf{ERBB2} yields negative $\Delta$ at upper quantiles ($-1.46$ at $\tau = 0.75$; $-0.79$ at $\tau = 0.9$), indicating that permutation \emph{improves} prediction, suggesting complex, non-monotonic associations in the upper tail that merit further investigation. These patterns are consistent across architectures, as shown in the supplementary average heatmap and the architecture-specific heatmaps.

\subsubsection{Clinically Defined Group Contrasts}
\label{sec:metabric-subgroup}

We next examine how the predicted survival distribution shifts across clinically defined patient groups. For each binary grouping variable, we compute
$
\Delta(\tau)=\bar{q}_\tau^{\,\text{group}_1}-\bar{q}_\tau^{\,\text{group}_0},$
the difference in mean predicted $\tau$-quantiles between the two groups, averaged over 25 seeds. We report 95\% bootstrap confidence bands based on 1{,}000 resamples of the seed-level summaries. Because METABRIC is an observational cohort, these contrasts are descriptive rather than causal; their value lies in revealing where along the survival distribution the prognostic separation between groups is most pronounced. They are exploratory: the bands are pointwise, the contrasts were not prespecified, and no multiplicity adjustment is made.

\textbf{ER+ vs.\ ER$-$ (Figure~\ref{fig:metabric-subgroup}a).}
All architectures consistently predict longer survival for ER-positive patients, with $\Delta(\tau)$ increasing from approximately 22--26 months at $\tau=0.1$ to 47--80 months at $\tau=0.9$. This pronounced quantile dependence suggests that ER status is not merely associated with a global location shift in the survival distribution. Instead, the prognostic separation widens toward the upper tail, indicating that ER-negative disease is associated not only with earlier events but also with a substantially reduced long-term survival horizon \citep{Blows2010, Broglio2009}. Such a pattern is difficult to summarize adequately with a single hazard ratio.

\textbf{Hormone therapy (Figure~\ref{fig:metabric-subgroup}b).}
$\Delta(\tau)$ is close to zero at the lower tail (approximately $-1$ to 2 months at $\tau=0.1$) but becomes increasingly negative at higher quantiles (approximately $-14$ to $-34$ months at $\tau=0.9$). This pattern should not be interpreted as an adverse treatment effect. More plausibly, it reflects confounding by indication: patients selected for hormone therapy tend to have underlying disease characteristics associated with different long-term risk trajectories \citep{EBCTCG2011}. The value of the quantile-based view is that it makes this tail-specific separation explicit, rather than collapsing it into a single average contrast.

\begin{figure}[t]
  \centering
  \begin{tabular}{cc}
    \includegraphics[width=0.46\textwidth]{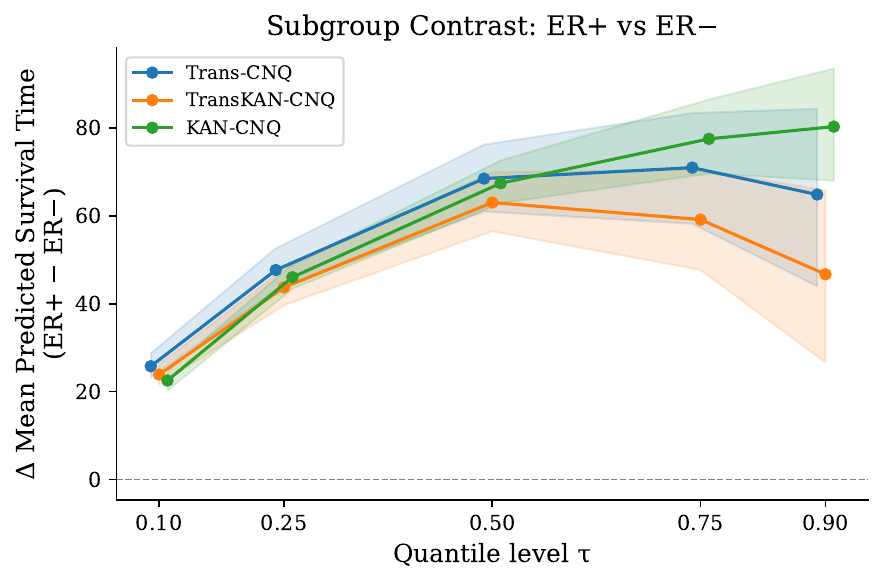} &
    \includegraphics[width=0.46\textwidth]{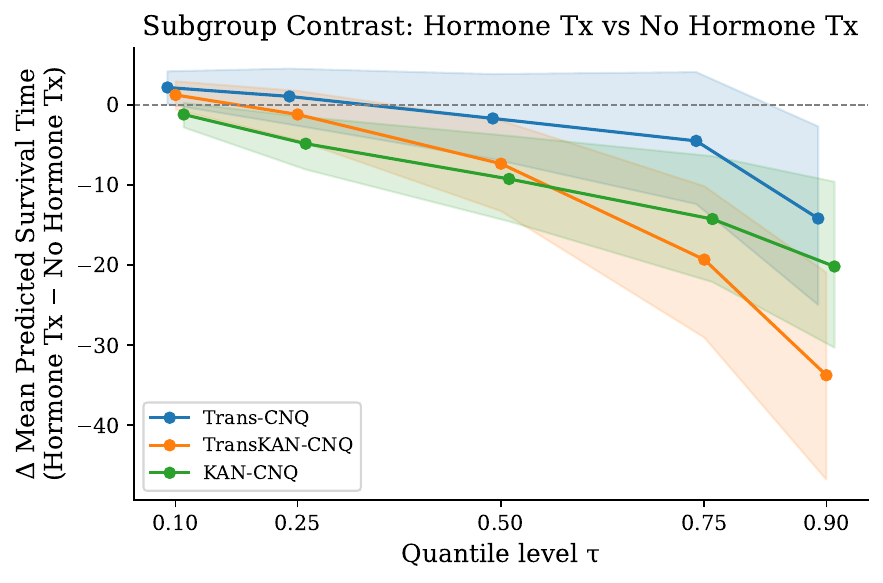} \\[-2pt]
    \small (a) ER+ vs.\ ER$-$ &
    \small (b) Hormone therapy vs.\ none
  \end{tabular}
  \caption{Clinically defined group contrasts $\Delta(\tau)$ for METABRIC with 95\% bootstrap bands.
           Non-constant profiles indicate distributional heterogeneity beyond a single proportional-hazards summary.}
  \label{fig:metabric-subgroup}
\end{figure}

\subsubsection{Representative Patient Profiles}
\label{sec:metabric-profiles}

Supplementary patient-profile materials present predicted quantile milestones for five representative METABRIC patients (seed 42, Trans-CNQ). These examples illustrate three clinically relevant uses of the model: distinguishing short-term risk among patients with similar medians, identifying patients with substantially different prognostic uncertainty, and assessing whether observed outcomes are compatible with the predicted survival distribution.

\textbf{Similar median, different tail risk (A vs.\ B).}
Both patients have comparable medians ($\hat{q}_{0.5} = 91.5$ and $100.6$ months), yet Patient~A's lower tail starts at $\hat{q}_{0.1} = 28.3$ vs.\ 34.0 months.
Patient~A (age 86) experienced an event at 41.8 months, within the predicted lower tail, while Patient~B (age 69) survived to 119.5 months near the median.
A single-point summary would treat these two patients as broadly similar, obscuring the clinically meaningful difference in short-term risk.

\textbf{Similar median, different uncertainty (C vs.\ D).}
Patient~C ($\hat{q}_{0.5}=102.7$, ER+, age 78) has an interval width of 150.9 months, whereas Patient~D ($\hat{q}_{0.5}=119.3$, ER$-$, age 58) has a much wider interval of 267.9 months. This difference indicates materially greater prognostic uncertainty for Patient~D, despite a somewhat longer median prediction, and could plausibly inform the intensity of follow-up or the degree of confidence attached to individual prognostic counseling.

\textbf{Individual calibration (E).}
The observed event ($T=100.3$ months) falls within 1\% of $\hat{q}_{0.5}=99.8$ and well inside the predicted 80\% interval $[\hat{q}_{0.1},\hat{q}_{0.9}]=[34.0,192.4]$, illustrating the plausibility of the individualized forecast on the original time scale. We likewise examined uncertainty stratification by prediction-interval width and found the expected monotone relationship between interval width, pinball loss, and empirical coverage; these supporting results are reported in the Supplementary Material to keep the main text focused on the primary clinical findings.

\subsubsection{Event Projection}
\label{sec:metabric-ep}

A clinically relevant downstream use of individualized survival distributions
is \emph{event projection}: given a calendar data cutoff $s$, how many
cumulative events are expected by a horizon $t>s$?  We compared a latent
projection, a censoring-aware Trans-CNQ projection, and a parametric Weibull
benchmark on the METABRIC test sets (seeds 41--65).  All three overestimate the
observed cumulative count at every cutoff.  The censoring-aware Trans-CNQ
projection is the most stable, with errors of $+18$--$+21\%$ across cutoffs,
whereas the Weibull benchmark is strongly horizon-dependent, overshooting by
$+63.8\%$ when only 24 months of follow-up are available but reaching $+6.9\%$
at the 96-month cutoff.  The Supplementary Material gives the projection
setup, the full error table and the projection curves.  The overestimation is
not a tuning artifact but a structural consequence of the objective: censored
observations enter it only through their weight, so nothing enforces $T_i>C_i$ and
predicted event times are pulled early. A hybrid loss that adds a censoring-aware lower-bound penalty to the
check loss is the natural remedy, which we leave to future work.

\subsection{FLCHAIN}
\label{sec:flchain}

\subsubsection{Secondary Validation Findings}
\label{sec:flchain-fi}

FLCHAIN serves as a secondary validation dataset: it tests whether the
qualitative findings from METABRIC generalize to a very different clinical
setting, and whether the architectures remain interpretively stable under
substantially heavier censoring.  Permutation feature importance (30 repeats
per feature per seed, against 50 for the smaller METABRIC test set; heatmaps in
the Supplementary Material) gives a coherent
and clinically plausible ordering, with age dominating across all $\tau$
(peak $\Delta=29.3$ at $\tau=0.5$), serum $\lambda$ FLC second, and MGUS and
sex negligible in the Transformer-based models.

\textbf{Architectural divergence.}
KAN-CNQ instead assigns the largest importance to creatinine, whereas both
Transformer-based models rank age first (creatinine: $\Delta\approx45.5$ for
KAN-CNQ vs.\ $\approx2.5$ for Trans-CNQ at $\tau=0.5$), despite comparable
predictive accuracy.  The disagreement reflects KAN's learnable B-spline
activations converging to different but equally accurate internal
representations, a non-identifiability that leaves prediction intact while
redistributing attribution across correlated covariates.  For interpretation, this matters: agreement between the two
Transformer variants does not by itself establish that an attribution is
architecture-independent.  Being able to visualize the learned spline components is
therefore not the same as reproducing the same scientific conclusions across
architectures.

\subsubsection{Clinically Defined Group Contrasts}
\label{sec:flchain-subgroup}

Figure~\ref{fig:flchain-subgroup} shows $\Delta(\tau)$ for three clinically defined group contrasts: sex, age (above vs.\ below the per-split median), and MGUS status. As in Section~\ref{sec:metabric-subgroup}, these comparisons are exploratory and descriptive with pointwise bands, intended to summarize prognostic separation across the survival distribution rather than to support causal claims.

\textbf{Sex (Figure~\ref{fig:flchain-subgroup}a).}
TransKAN-CNQ shows modestly positive $\Delta(\tau)$ for females (approximately 18--48 days longer), while Trans-CNQ hovers near zero ($-9$ to $28$ days) with confidence bands overlapping zero at every $\tau$. KAN-CNQ is positive at the lower quantiles but changes sign above $\tau=0.5$, reaching $-77$ days at $\tau=0.9$. The three architectures therefore agree neither on magnitude nor on direction, and no contrast separates from zero at any $\tau$: sex carries little prognostic signal in this cohort once the remaining covariates are accounted for.

\textbf{Age (Figure~\ref{fig:flchain-subgroup}b).}
Older patients have markedly shorter predicted survival across the entire distribution, with $\Delta(\tau)$ ranging from approximately $-40$ to $-257$ days and confidence bands excluding zero for all three architectures at every $\tau$. This is the one contrast on which all three models agree in sign, magnitude, and shape. The separation is smallest at $\tau=0.1$ and widens from $\tau=0.25$ onward, indicating that age compresses the central and upper survival range rather than merely shifting the lower tail.

\textbf{MGUS (Figure~\ref{fig:flchain-subgroup}c).}
Trans-CNQ reveals a crossover pattern: MGUS-positive patients have \emph{longer} predicted survival at lower quantiles ($\Delta \approx +93$ to $+174$ days at $\tau=0.1$--$0.25$) but \emph{shorter} predicted survival at upper quantiles ($\Delta \approx -116$ to $-198$ days at $\tau=0.75$--$0.9$). The other two architectures do not reproduce this shape: TransKAN-CNQ stays close to zero throughout ($-45$ to $-6$ days, with bands covering zero at every $\tau$), while KAN-CNQ increases monotonically to $+1{,}214$ days at $\tau=0.9$. MGUS-positive subjects constitute only about $1.5\%$ of each test split (roughly 24 patients), so all three contrasts rest on a very small group and their bands are correspondingly wide. We regard the MGUS contrast as inconclusive, and report it to illustrate that quantile-level contrasts inherit the sampling limitations of the subgroup defining them.

\begin{figure}[t]
  \centering
  \begin{tabular}{ccc}
    \includegraphics[width=0.31\textwidth]{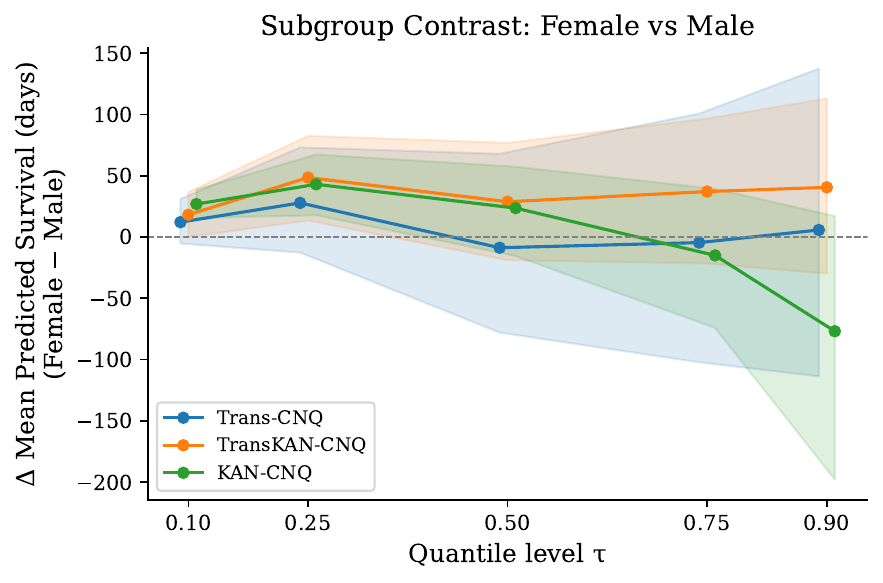} &
    \includegraphics[width=0.31\textwidth]{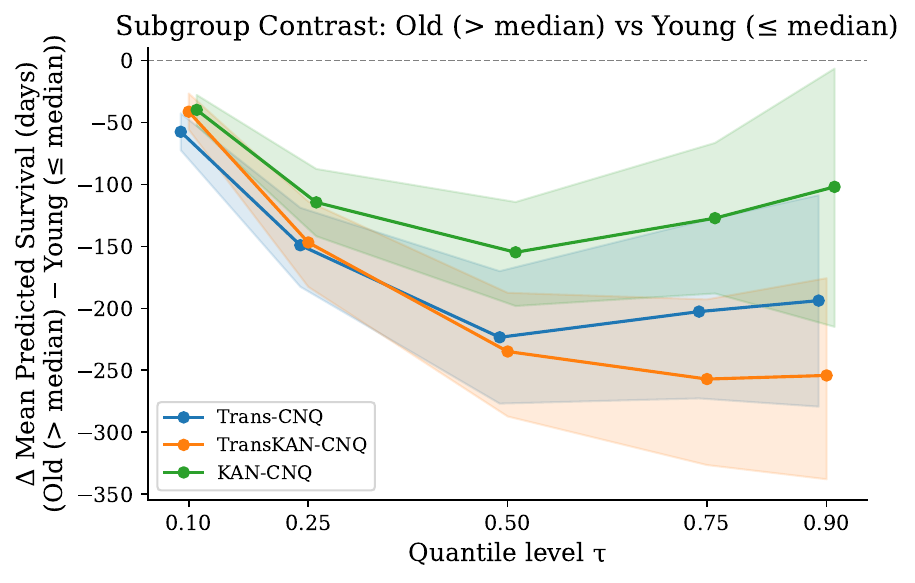} &
    \includegraphics[width=0.31\textwidth]{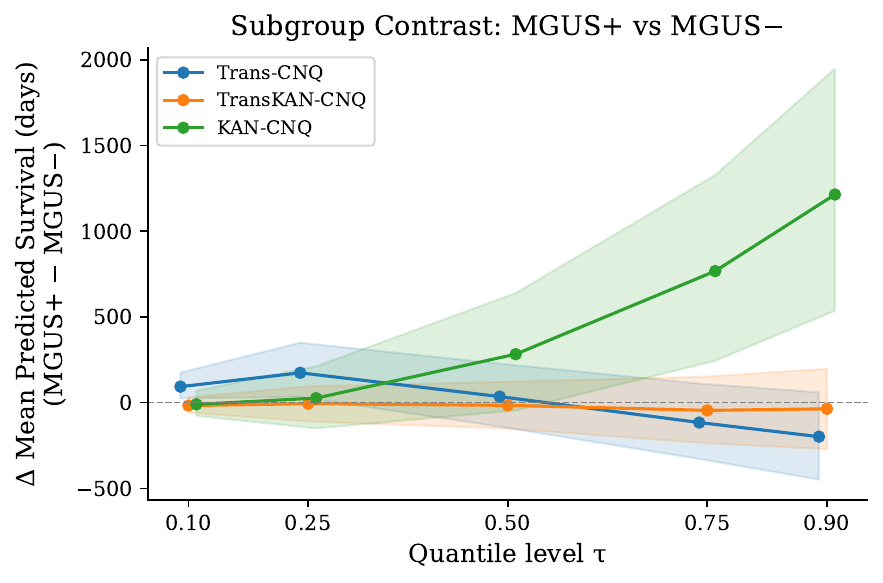} \\[-2pt]
    \small (a) Female vs.\ Male &
    \small (b) Old vs.\ Young &
    \small (c) MGUS+ vs.\ MGUS$-$
  \end{tabular}
  \caption{Clinically defined group contrasts $\Delta(\tau)$ for FLCHAIN with 95\% bootstrap bands.
           The three architectures agree closely on age (b) but not on sex (a) or MGUS (c); the MGUS bands are wide for
           all models because MGUS-positive subjects make up only about $1.5\%$ of each test split.}
  \label{fig:flchain-subgroup}
\end{figure}

Representative patient profiles and uncertainty stratification for FLCHAIN reproduce the METABRIC findings of Section~\ref{sec:metabric-profiles}; the figures and tables are in the Supplementary Material.

\section{Discussion}
\label{sec:discussion}

The proposed CNQ framework attained the lowest IPCW pinball loss on every cohort, with
interval calibration matching DeepQuantreg and bettering the hazard- and
tree-based competitors.

METABRIC and FLCHAIN differ in clinical context, size and censoring, so agreement
between them is replication rather than repetition.  Returning
to the questions in Section~\ref{s:data_motivation}: \textbf{(Q1)} the monotone parameterization yields non-crossing
individualized five-quantile milestones, and interval width behaves as an internal indicator of predictive
difficulty, increasing monotonically with both pinball loss and empirical
coverage. \textbf{(Q2)} quantile-specific importance and group contrasts reveal
covariate effects that vary materially across the survival distribution---the
METABRIC ER contrast widens steadily from the lower to the upper tail, and the
FLCHAIN age contrast is separated from zero by all three architectures at every
$\tau$---heterogeneity that a single hazard ratio would obscure. \textbf{(Q3)} the Transformer-based variants remained the
most accurate methods even on the small or heavily censored cohorts (GBSG-500,
NKI70), although the larger cross-split variability warrants caution for
individual-level predictions in data-limited regimes.

On the theoretical side, the rates are model-family upper bounds subject to the
scope conditions of Remark~\ref{rem:scope}, and every bound deteriorates as
$G_\star$ decreases.

Empirically, feature attributions were architecture-dependent even where
predictive accuracy was not, so we recommend comparing attributions across
backbones rather than relying on one.  Finally, the real-data subgroup and
feature-importance analyses are descriptive and hypothesis-generating rather
than causal, and should be confirmed in purpose-designed studies.

Future work falls into two directions.  On the applied side, the framework
extends naturally to competing risks and to time-dependent or longitudinally
measured covariates, and censoring-aware training objectives that use
lower-bound information directly should improve far-tail calibration and event
projection.  On the theoretical side, bounds for the fitted TransKAN-CNQ architecture, and for
cubic-spline EfficientKAN without the degree-$p$ realization used here, remain open.






\begin{acks}[Acknowledgments]
\textbf{Corresponding authors:} Zhe Qu (zhe.qu@servier.com), and Hongtu Zhu (htzhu@email.unc.edu).

This work was carried out in part during S.~Huang's internship at Servier
Pharmaceuticals, whose support is gratefully acknowledged.
\end{acks}

\begin{funding}

\end{funding}


\appendix
\clearpage
\section{Supplementary Dataset Diagnostics}
\label{sec:dataset-diagnostics}

\subsection{Data provenance and availability}
\label{sec:data-availability}

All six cohorts are public. We use previously preprocessed versions rather than
the raw sources, so the sample sizes and covariate counts reported in the main
text refer to those versions and need not match the original publications.
METABRIC, GBSG and SUPPORT are taken from the \texttt{pycox} distribution
\citep{kvamme2019time}, which redistributes the preprocessing of
\citet{katzman2018deepsurv}; GBSG-500 is a random 500-subject subsample of GBSG
drawn once and held fixed across seeds. FLCHAIN is obtained from the
\texttt{flchain} dataset of the R \texttt{survival} package
\citep{Dispenzieri2012}. NKI70 is obtained from the \texttt{nki70} dataset of
the R \texttt{penalized} package \citep{van2002gene}; we retain age, estrogen-receptor
status and six genes of the 70-gene signature (NUSAP1, FGF18, ZNF533, COL4A2,
CDCA7, MCM6), all standardized. Subjects with missing covariates are excluded
and, following the training pipeline, the few records with zero observed
duration are dropped. The scripts that build the $25$ seeded
train/validation/test partitions are released with the implementation, so the
exact splits underlying every reported number can be regenerated.

Before turning to the modeling results, Figure~\ref{fig:km-curves} documents the
marginal survival behavior of the six real-data cohorts used throughout the
paper. Each panel shows the Kaplan--Meier estimate of the marginal survival
function $\widehat S(t)=\mathbb P(T>t)$ together with pointwise $95\%$
confidence bands, so that the overall event rate, the speed of early decline,
and the length of the right tail can be read off directly for each dataset.
The six cohorts are deliberately heterogeneous: they range from the small,
heavily censored NKI70 study ($n=144$, $67\%$ censoring) to the large SUPPORT
cohort ($n=8{,}873$, $\approx32\%$ censoring) and the large, heavily censored
FLCHAIN cohort ($n=7{,}874$, $\approx72\%$ censoring), and they span oncology
trials, genomic prognostic studies, and a population-based cohort. This spread
of sample sizes, censoring proportions, and survival shapes is what allows the
real-data benchmark to stress-test distributional prediction across a broad
range of operating conditions; the curves in Figure~\ref{fig:km-curves} provide
the unconditional reference against which the conditional, covariate-dependent
quantile predictions of the proposed models are evaluated in the remainder of
this supplement.

\begin{center}
\includegraphics[width=\textwidth]{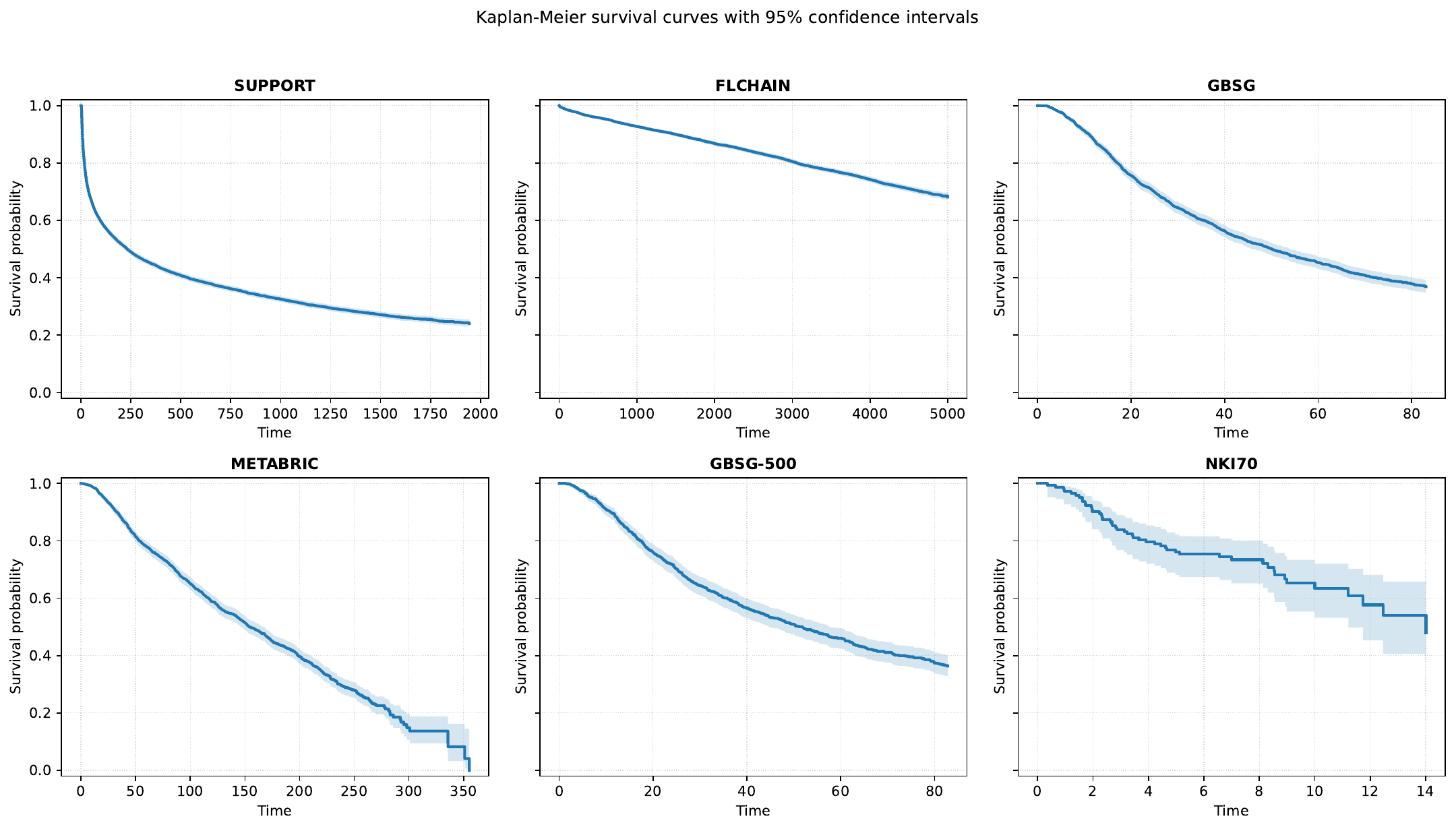}

\vspace{0.5em}
\refstepcounter{figure}
\label{fig:km-curves}
\parbox{0.95\textwidth}{
\small
\textbf{Figure~\thefigure.}
Kaplan--Meier survival curves with pointwise $95\%$ confidence intervals
(Greenwood's formula, shaded bands) for the six real-data cohorts analyzed
in the paper---GBSG, GBSG-500, METABRIC, SUPPORT, NKI70, and FLCHAIN. In
each panel the horizontal axis is the follow-up time (in the native
time unit of the cohort) and the vertical axis is the estimated marginal
survival probability $\widehat S(t)=\mathbb P(T>t)$, with the number-at-risk
declining as events accrue and subjects are censored. The panels illustrate
the wide range of operating conditions used to evaluate the methods: cohort
sizes span $n=144$ (NKI70) to $n=8{,}873$ (SUPPORT), censoring proportions
range from about $32\%$ (SUPPORT) to $72\%$ (FLCHAIN), and the curves differ
markedly in their overall event rate, the steepness of early decline, and
the length of the right tail. These marginal descriptives provide the
baseline against which the conditional, covariate-dependent quantile
predictions of the proposed models are assessed.
}
\end{center}

\clearpage
\section{Technical Result}

\begin{lemma}[IPCW identity for the marginal Kaplan--Meier weight]
\label{lem:ipcw}
Assume $(T,X)\perp C$.  Let $Y=\min(T,C)$,
$\delta=\mathbb I\{T\leq C\}$, $G(t)=\mathbb P(C>t)$, and
$G(t^-)=\mathbb P(C\geq t)$.  Assume \(G(T^-)>0\) almost surely.
For every measurable \(a(X)\) and \(\tau\in(0,1)\)
such that \(\rho_\tau\{\log T-a(X)\}\) is integrable,
\[
\mathbb E\!\left[
\frac{\delta}{G(Y^-)}
\rho_\tau\{\log Y-a(X)\}\right]
=
\mathbb E\!\left[\rho_\tau\{\log T-a(X)\}\right].
\]
\end{lemma}

\begin{proof}
On $\{\delta=1\}$, $Y=T$.  Conditioning on $(T,X)$ and using
$(T,X)\perp C$ gives
\begin{align*}
\mathbb E\!\left[
\left.
\frac{\delta}{G(Y^-)}
\rho_\tau\{\log Y-a(X)\}
\right|T,X\right]
&=
\frac{\rho_\tau\{\log T-a(X)\}}{G(T^-)}
\mathbb P(C\geq T\mid T,X)\\
&=\rho_\tau\{\log T-a(X)\}.
\end{align*}
Taking expectations proves the claim.  The left limit is essential here:
with the convention $G(t)=\mathbb P(C>t)$, one has
$G(t^-)=\mathbb P(C\geq t)$.
\end{proof}

\section{Proof of Theorem~\ref{Risk_error}}
\label{appendix:proof}
Write $P_N$ for the empirical measure and abbreviate
\[
\ell_f(v,x)=\frac1K\sum_{k=1}^K
\rho_{\tau_k}\{v-f_k(x)\},\qquad
q_f(v,x)=\ell_f(v,x)-\ell_{f_N^\star}(v,x).
\]
The check loss is one-Lipschitz in its prediction argument.  Hence
\begin{equation}\label{eq:q-envelope}
|q_f(v,x)|\leq\frac1K\sum_{k=1}^K
|f_k(x)-f_{N,k}^\star(x)|\leq2M_f.
\end{equation}
Define the true-weight and plug-in centered losses
\[
    h_f(O)=\frac{\delta}{G(Y^-)}q_f(W,X),\qquad
    \widehat h_f(O)=\frac{\delta}{\widehat G(Y^-)}q_f(W,X).
\]
Both ratios are defined as zero on \(\{\delta=0\}\).
Lemma~\ref{lem:ipcw} gives
\begin{equation}\label{eq:Ph-risk}
Ph_f=L(f)-L(f_N^\star)=R_{\mathcal F_N}(f).
\end{equation}

\medskip\noindent\textit{Step 1: ERM and the estimated censoring weights.}\quad
Because $\widehat f_N$ minimizes $\widehat L_N$,
$P_N\widehat h_{\widehat f_N}\leq0$.  On the event
$\Delta_N\leq G_\star/2$, equations~\eqref{eq:q-envelope}--\eqref{eq:Ph-risk}
therefore imply
\begin{align}
R_{\mathcal F_N}(\widehat f_N)
&=(P-P_N)h_{\widehat f_N}
  +P_N(h_{\widehat f_N}-\widehat h_{\widehat f_N})
  +P_N\widehat h_{\widehat f_N}\notag\\
&\leq
\sup_{f\in\mathcal F_N}(P-P_N)h_f
+\frac{2M_f\Delta_N}{G_\star(G_\star-\Delta_N)}
\notag\\
\leq
\sup_{f\in\mathcal F_N}(P-P_N)h_f
+\frac{4M_f\Delta_N}{G_\star^2}.          \label{eq:erm-decomp}
\end{align}

\medskip\noindent\textit{Step 2: empirical \(L_2\) contraction.}\quad
For $f,g\in\mathcal F_N$, Cauchy--Schwarz over the $K$ outputs gives
\begin{align}
\left\{\sum_{i=1}^N
|h_f(O_i)-h_g(O_i)|^2\right\}^{1/2}
&\leq
\frac1{G_\star}
\left[\sum_{i=1}^N
\left\{\frac1K\sum_{k=1}^K
|f_k(X_i)-g_k(X_i)|\right\}^2\right]^{1/2}\notag\\
&\leq
\frac1{G_\star\sqrt K}
\left\{ \sum_{i=1}^N\sum_{k=1}^K
|f_k(X_i)-g_k(X_i)|^2\right\}^{1/2}.       \label{eq:l2-contraction}
\end{align}
Consequently, if
$\mathcal H_N(O_{1:N})=\{(h_f(O_i))_{i=1}^N:f\in\mathcal F_N\}$,
then
\begin{equation}\label{eq:loss-entropy}
\log\mathcal N\{\mathcal H_N(O_{1:N}),\|\cdot\|_2,u\}
\leq\frac{A_N}{K G_\star^2u^2}.
\end{equation}
This is the point at which the empirical Frobenius covering result of
\citet{zhang2024generalization} enters; no population-to-supremum-norm
conversion is required.

\medskip\noindent\textit{Step 3: entropy integral.}\quad
Conditional on the observations, Dudley's truncated entropy integral,
in the form used in Lemmas 2--3 of \citet{zhang2024generalization}, and
\eqref{eq:loss-entropy} yield
\begin{align}
\widehat{\mathfrak R}_N(\mathcal H_N)
&\leq
\inf_{0<a\leq 2M_f\sqrt N/G_\star}\left[
\frac{4a}{\sqrt N}
+\frac{12}{N}\int_a^{2M_f\sqrt N/G_\star}
\sqrt{\frac{A_N}{K G_\star^2u^2}}\,du
\right]\notag\\
&\leq
C\frac{\sqrt{A_N}}{G_\star\sqrt K\,N}
\left\{1+\log_+\!\left(
\frac{M_f\sqrt K\,N}{\sqrt{A_N}}\right)\right\}.    \label{eq:dudley}
\end{align}
For
\(a_0=\sqrt{A_N}/(G_\star\sqrt{KN})\), use \(a=a_0\) when
\(a_0\leq2M_f\sqrt N/G_\star\).  In the complementary case the trivial
bound \(\widehat{\mathfrak R}_N(\mathcal H_N)\leq2M_f/G_\star\) is no larger
than a universal multiple of the right-hand side of~\eqref{eq:dudley}.
Here and below $C$ denotes a universal constant whose value may change
from line to line.  Symmetrization and the standard bounded empirical
process concentration inequality, using
$|h_f|\leq2M_f/G_\star$, give, with probability at least
$1-ce^{-\eta}$,
\begin{align}
\sup_{f\in\mathcal F_N}(P-P_N)h_f
\leq C\bigg[
&\frac{\sqrt{A_N}}{G_\star\sqrt K\,N}
\left\{1+\log_+\!\left(
\frac{M_f\sqrt K\,N}{\sqrt{A_N}}\right)\right\}\notag\\
&+\frac{M_f}{G_\star}
\left\{\sqrt{\frac{\eta+1}{N}}+\frac{\eta+1}{N}\right\}
\bigg].                                             \label{eq:ep-bound}
\end{align}

\medskip\noindent\textit{Step 4: conclusion.}\quad
Intersect the event in~\eqref{eq:ep-bound} with the Kaplan--Meier event
in (A2), substitute $\Delta_N\leq\kappa_N(\eta)$ into
\eqref{eq:erm-decomp}, and absorb numerical constants into $C$.  This is
exactly the bound in Theorem~\ref{Risk_error}, with failure probability at most
$(c+c_G)e^{-\eta}$.

For completeness, the DKW--Kaplan--Meier inequality
of \citet{bitouze1999dkw}, also quoted in
\citet{goldberg2019hoeffding}, states on a fixed identifiable time range
$[0,\tau]$,
\[
\mathbb P\!\left[
\sqrt N\,S_\tau
\sup_{0<t\leq\tau}|\widehat G(t^-)-G(t^-)|
\geq\sqrt{\eta/2}+D_o/2\right]
\leq\frac52e^{-\eta},
\]
where $S_\tau=\mathbb P(T\geq\tau)$.  Thus the abstract quantity
$\kappa_N(\eta)$ in (A2) has the asserted $N^{-1/2}$ order whenever the
targeted event-time range lies inside such an identifiable interval and
$S_\tau$ is bounded away from zero.
It does not by itself verify (A2) over an entire unbounded support,
nor at a bounded terminal point with \(S_\tau=0\).
Accordingly, no verification of (A2) over the full continuous
event-time support is claimed here.

\section{Complexity-Indexed Architecture Bounds and Approximation Rates}
\label{app:architecture-theory}

This section retains architecture complexity before choosing a
sample-size-dependent sieve and then bounds approximation and estimation
error inside the same class.  Throughout, \(q=(q_1,\ldots,q_K)\) denotes the
true conditional log-time quantile vector and \(\Phi\) is the
softplus--cumulative-sum map in the main text.  The covariate dimension
\(p\) and number of target quantiles \(K\) are held fixed as the sieve index
and sample size vary.  Its Jacobian satisfies
\begin{equation}\label{eq:app-Phi-lip}
 \|\Phi(r)-\Phi(\widetilde r)\|_2
 \leq L_\Phi\|r-\widetilde r\|_2,\qquad
 L_\Phi\leq\{K(K+1)/2\}^{1/2}.
\end{equation}
Because check loss is one-Lipschitz in its prediction argument,
\begin{equation}\label{eq:app-risk-lip}
 0\leq L(f)-L(q)
 \leq \frac1K\mathbb E\|f(X)-q(X)\|_1
 \leq K^{-1/2}
       \{\mathbb E\|f(X)-q(X)\|_2^2\}^{1/2}.
\end{equation}

\subsection{KAN: arbitrary complexity}
\label{app:kan-complexity}

Let
\[
 s_{\rm KAN}
 =\{L,(d_\ell),(p_\ell),
     (B_\ell,c_\ell,C_\ell,\rho_\ell)_{\ell\leq L}\}
\]
collect the depth, widths, basis counts, coefficient budgets, basis
Lipschitz bounds, layer offsets, and layer Lipschitz constants of a
basis-expansion KAN; include the affine \(K\)-output head as its final
layer.  Write
\(\widetilde d(s)=\max_\ell d_\ell\),
\(\widetilde p(s)=\max_\ell p_\ell\), and
\(C(s)=\max_\ell C_\ell\).
For this subsection, set
\(\bar\rho_\ell=1\vee\rho_\ell\) and evaluate
\eqref{eq:app-alpha}--\eqref{eq:app-kan-oracle} with
\(\bar\rho_\ell\) in place of \(\rho_\ell\); this only increases
\(\alpha_{i,N}(s)\).  Writing \(R_0=D_N\) and
\(R_\ell=C(s)+\bar\rho_\ell R_{\ell-1}\), we have
\(R_\ell\geq R_{\ell-1}\).  The basis-function Lipschitz step is
therefore bounded by
\(c_\ell R_{\ell-1}\leq c_\ell R_\ell\), which validates the possibly
enlarged coefficient used below.
For a realized design with
\(\|\mathbf X\|_F\leq D_N\), define
\begin{align}
 \widetilde\alpha_N(s)
 &=\sum_{i=1}^L\alpha_{i,N}(s),\notag\\
 \alpha_{i,N}(s)
 &=B_i^{2/3}c_i^{2/3}
 \left(\prod_{j=i+1}^L\rho_j\right)^{2/3}
 \left\{
 C(s)\sum_{j=0}^{i-1}\prod_{k=i-j+1}^{i}\rho_k
 +D_N\prod_{k=1}^{i}\rho_k
 \right\}^{2/3},                                             \label{eq:app-alpha}\\
 \mathfrak E_N^{\rm KAN}(s)
 &=L_\Phi^2\widetilde\alpha_N(s)^3
   \log\{2\widetilde d(s)\widetilde p(s)\}.                  \label{eq:app-kan-E}
\end{align}
Empty products equal one.

\begin{corollary}[KAN oracle bound at arbitrary complexity]
\label{cor:app-kan-oracle}
Let \(s=s_N\) be any deterministic architecture-and-budget sequence
satisfying Assumptions 1--2 of
\citet{zhang2024generalization}.  Suppose
\(\|X_i\|_2\leq B_X\) almost surely, take \(D_N=B_X\sqrt N\), and intersect
the corresponding non-crossing class with
\(\max_k\|f_k\|_\infty\leq M_f\).  Under (A1)--(A2) and
\(\kappa_N(\eta)=O(N^{-1/2})\) at fixed confidence,
\begin{align}
 R_{\mathcal F_s}(\widehat f_{N,s})
 =O_{\mathbb P}\Bigg[&
 \frac{\sqrt{\mathfrak E_N^{\rm KAN}(s)}}{N}
 \left\{1+\log_+\!\left(
 \frac{M_f\sqrt K\,N}
      {\sqrt{\mathfrak E_N^{\rm KAN}(s)}}\right)\right\}
 +\frac1{\sqrt N}\Bigg].                                    \label{eq:app-kan-oracle}
\end{align}
If \(s\) is fixed, then
\(\mathfrak E_N^{\rm KAN}(s)=O(N)\) and the rate reduces to
\(O_{\mathbb P}(\log N/\sqrt N)\).
\end{corollary}

\begin{proof}
Theorem 1 of \citet{zhang2024generalization} gives the empirical
Frobenius covering coefficient
\(\widetilde\alpha_N(s)^3
\log\{2\widetilde d(s)\widetilde p(s)\}\) for the raw KAN output matrix.
Equation~\eqref{eq:app-Phi-lip} multiplies it by \(L_\Phi^2\), so
Theorem~\ref{Risk_error} gives~\eqref{eq:app-kan-oracle}.  For fixed
\(s\), every brace in~\eqref{eq:app-alpha} is \(O(\sqrt N)\).  Hence
\(\alpha_{i,N}=O(N^{1/3})\),
\(\widetilde\alpha_N=O(N^{1/3})\), and
\(\mathfrak E_N^{\rm KAN}=O(N)\).
\end{proof}

\begin{remark}
Widths, basis counts, and budgets may depend on \(N\) in
Corollary~\ref{cor:app-kan-oracle}, but their dependence must remain in
\eqref{eq:app-alpha}--\eqref{eq:app-kan-oracle}.  In particular, the
fixed-class order \(\mathfrak E_N^{\rm KAN}=O(N)\) cannot be inserted after
choosing a growing spline resolution.
\end{remark}

\subsection{KAN: approximation and a matched growing sieve}
\label{app:kan-holder}

Assume \(X\in[0,1]^p\), and, for some \(\beta>0\), \(B_q<\infty\), and
\(\gamma>0\),
\begin{equation}\label{eq:app-holder}
 \max_{k\leq K}\|q_k\|_{C^\beta([0,1]^p)}\leq B_q,\qquad
 \inf_x\{q_k(x)-q_{k-1}(x)\}\geq\gamma,\quad k\geq2.
\end{equation}
Define the raw target
\begin{equation}\label{eq:app-raw-target}
 r^\circ_1=q_1,\qquad
 r^\circ_k=\log\{\exp(q_k-q_{k-1})-1\},\quad k\geq2.
\end{equation}
The positive gap and the standard Hölder composition inequality imply
\(\max_k\|r^\circ_k\|_{C^\beta}\leq C_r\).

Let \(\{B_{\nu,J}\}_{\nu\in\mathcal I_J}\) be a degree-\(m\) B-spline
basis on a quasi-uniform \(J\)-cell partition, where \(m\geq2\) and
\(\beta\leq m+1\), and put \(M_J=|\mathcal I_J|=J+O(1)\).  A stable
tensor-product quasi-interpolant gives coefficients, uniformly bounded in
\(J\), such that
\begin{align}
 r_{k,J}(x)
 &=\sum_{\boldsymbol\nu\in\mathcal I_J^p}
 c_{k,\boldsymbol\nu,J}
 \prod_{j=1}^pB_{\nu_j,J}(x_j),\notag\\
 \max_{k\leq K}\|r_{k,J}-r^\circ_k\|_\infty
 &\leq C J^{-\beta};
                                                               \label{eq:app-tensor}
\end{align}
see \citet{schumaker2007spline}.

\begin{lemma}[Two-layer spline-KAN realization]
\label{lem:app-kan-realization}
Suppose the second KAN layer has a fixed edge dictionary whose span
contains \(t\mapsto t^p\) on the bounded node range; a fixed-grid
B-spline space of degree at least \(p\) suffices.  Then \(r_J\) in
\eqref{eq:app-tensor} has an exact two-layer KAN realization with hidden
width \(K2^pM_J^{p-1}=O(J^{p-1})\).  If every edge coefficient of this
architecture is free, the full dense class has \(P_J=O(J^p)\) scalar
coefficients.
\end{lemma}

\begin{proof}
For
\(\boldsymbol\nu_{-p}=(\nu_1,\ldots,\nu_{p-1})\), write
\[
 u_{k,\boldsymbol\nu_{-p},j}(t)=B_{\nu_j,J}(t),\quad j<p,
 \qquad
 u_{k,\boldsymbol\nu_{-p},p}(t)
 =\sum_{\nu_p}c_{k,(\boldsymbol\nu_{-p},\nu_p),J}B_{\nu_p,J}(t).
\]
Then \(r_{k,J}\) is the sum over \(\boldsymbol\nu_{-p}\) of
\(\prod_{j=1}^pu_{k,\boldsymbol\nu_{-p},j}(x_j)\).  The polarization
identity
\[
 \prod_{j=1}^pu_j
 =
 \frac1{2^pp!}\sum_{\epsilon\in\{-1,1\}^p}
 \left(\prod_{j=1}^p\epsilon_j\right)
 \left(\sum_{j=1}^p\epsilon_ju_j\right)^p
\]
therefore gives the realization: a first-layer node computes each signed
sum and a second-layer edge applies a signed multiple of \(t^p\).  B-spline
nonnegativity, partition of unity, and the coefficient bound in
\eqref{eq:app-tensor} keep all node ranges and required polynomial
coefficients bounded independently of \(J\).

The first layer has
\(p\{K2^pM_J^{p-1}\}O(J)=O(J^p)\) free coefficients.  The second has
\(O(J^{p-1})\) coefficients with a fixed polynomial dictionary, or
\(O(J^p)\) if a common \(J\)-grid dictionary is used.  Thus
\(P_J=O(J^p)\).
\end{proof}

For completeness, the next parameter cover applies to the full class, not
only to the particular approximant.  Write an \(L\)-layer fixed-dictionary
KAN as
\[
 z_{\ell,j}=\sum_{i=1}^{n_{\ell-1}}\sum_{a=1}^{q_\ell}
 \theta_{\ell,j,i,a}g_{\ell,j,i,a}(z_{\ell-1,i}),
 \qquad |\theta_{\ell,j,i,a}|\leq A_\ell,
\]
where \(\|g_{\ell,j,i,a}\|_\infty\leq G_\ell\) and
\(\operatorname{Lip}(g_{\ell,j,i,a})\leq H_\ell\) on every reached range.
Include the affine raw-output head as layer \(L+1\), and define
\begin{align}
 P&=\sum_{\ell=1}^Ln_{\ell-1}n_\ell q_\ell+K(n_L+1),\notag\\
 \rho_\ell&=A_\ell q_\ell H_\ell\sqrt{n_{\ell-1}n_\ell},
 &\lambda_\ell&=G_\ell\sqrt{n_{\ell-1}q_\ell},\quad \ell\leq L,\notag\\
 \rho_{L+1}&=A_{\rm out}\sqrt{Kn_L},
 &\lambda_{L+1}&=(R_L^2+1)^{1/2},\notag\\
 \Lambda_{\rm par}
 &=L_\Phi\left\{
 \sum_{\ell=1}^{L+1}\lambda_\ell^2
 \prod_{r=\ell+1}^{L+1}\rho_r^2
 \right\}^{1/2},\notag\\
 A_{\max}&=\max(A_1,\ldots,A_L,A_{\rm out}).               \label{eq:app-kan-parlip}
\end{align}

\begin{lemma}[Full coefficient-class cover]
\label{lem:app-kan-cover}
For every realized design and \(u>0\),
\begin{equation}\label{eq:app-kan-cover}
 \log\mathcal N\{\mathcal F^{\rm dense}(\mathbf X),
                  \|\cdot\|_F,u\}
 \leq
 P\log\left(
 1+\frac{2A_{\max}\sqrt P\,\Lambda_{\rm par}\sqrt N}{u}
 \right).
\end{equation}
\end{lemma}

\begin{proof}
Cauchy--Schwarz gives, at layer \(\ell\),
\[
 \|\Psi_\ell(z)-\Psi_\ell(z')\|_2\leq\rho_\ell\|z-z'\|_2,
 \qquad
 \|\Psi_{\ell,\Theta}(z)-\Psi_{\ell,\widetilde\Theta}(z)\|_2
 \leq\lambda_\ell\|\Theta-\widetilde\Theta\|_F.
\]
Telescoping through the layers, applying
\eqref{eq:app-Phi-lip}, and then using Cauchy--Schwarz across parameter
blocks yields
\(\|f_\vartheta(x)-f_{\widetilde\vartheta}(x)\|_2
\leq\Lambda_{\rm par}\|\vartheta-\widetilde\vartheta\|_2\).
The parameter box lies in the radius-\(A_{\max}\sqrt P\) Euclidean ball.
A proper volume net with
\(\delta=u/(\Lambda_{\rm par}\sqrt N)\) proves
\eqref{eq:app-kan-cover}.
\end{proof}

\begin{theorem}[Excess-risk convergence rate for the KAN-CNQ estimator]
\label{thm:app-kan-total}
Let \(\mathcal F_J^{\rm KAN}\) be the full fixed-grid class on the
architecture in Lemma~\ref{lem:app-kan-realization}, intersected with a
fixed output envelope large enough to contain the approximant.  Suppose
all quantities in~\eqref{eq:app-kan-parlip} grow at most polynomially in
\(J\).  Under~\eqref{eq:app-holder}, (A1)--(A2), and
\(\kappa_N(\eta)=O(N^{-1/2})\) at fixed confidence,
\begin{equation}\label{eq:app-kan-total}
 L(\widehat f_{N,J})-L(q)
 =
 O_{\mathbb P}\left\{
 J^{-\beta}
 +\sqrt{\frac{J^p\log(2J)}N}
 +\frac1{\sqrt N}\right\}.
\end{equation}
Thus \(J_N\asymp\{N/\log N\}^{1/(2\beta+p)}\) gives
\[
 L(\widehat f_{N,J_N})-L(q)
 =
 O_{\mathbb P}\left[
 \left\{\frac N{\log N}\right\}^{-\beta/(2\beta+p)}\right].
\]
\end{theorem}

\begin{proof}
Equations~\eqref{eq:app-Phi-lip}, \eqref{eq:app-risk-lip}, and
\eqref{eq:app-tensor} give an approximation term \(O(J^{-\beta})\).

For estimation, we reuse Steps 1, 2, and 4 of the proof of
Theorem~\ref{Risk_error}, while recomputing Step 3 directly under the logarithmic
entropy bound~\eqref{eq:app-kan-cover}.
After writing \(u=\sqrt N\,t\), its Dudley
integral is bounded by
\[
 C\sqrt{\frac{P_J}{N}}\int_0^{CM_f}
 \sqrt{\log(1+\Xi_J/t)}\,dt
 \leq
 C\sqrt{\frac{P_J}{N}}\{1+\sqrt{\log(2+\Xi_J)}\},
\]
where
\(\Xi_J=2A_{\max,J}\sqrt{P_J}\Lambda_{{\rm par},J}\).
Since \(P_J=O(J^p)\) and \(\Xi_J\) is polynomial in \(J\), this is
\(O\{\sqrt{J^p\log(2J)/N}\}\).  Adding the
\(O_{\mathbb P}(N^{-1/2})\) Kaplan--Meier and concentration terms proves
\eqref{eq:app-kan-total}; balancing the first two terms gives the stated
choice of \(J_N\).
\end{proof}

\subsection{Complete deterministic Trans-CNQ entropy}
\label{app:trans-entropy}

There are \(p\) scalar-feature tokens, model width \(d\), feedforward width
\(m\), \(H\) heads with \(H\mid d\) and \(d_h=d/H\), \(L\) encoder layers, flattened
readout width \(r\), and \(K\) raw outputs.  For an affine map
\(a(z)=Wz+b=\bar W\bar z\), impose
\begin{equation}\label{eq:app-affine-bounds}
 \|\bar W\|_{\rm op}\leq B_a,\qquad
 \|\bar W^\top\|_{2,1}\leq S_a,\qquad
 \bar z=(z^\top,1)^\top.
\end{equation}
The constants may be layer-specific.  Standard affine LayerNorm is
\begin{align}
 P_d&=I_d-d^{-1}{\bf1}{\bf1}^{\top},\qquad
 \nu_\epsilon(z)=
 \frac{P_dz}{\{d^{-1}\|P_dz\|_2^2+\epsilon_{\rm LN}\}^{1/2}},\notag\\
 {\rm LN}_{\gamma,\beta}(z)
 &=\gamma\odot\nu_\epsilon(z)+\beta,                         \label{eq:app-ln}
\end{align}
where \(\epsilon_{\rm LN}>0\),
\(\|\gamma\|_\infty\leq G_\infty\),
\(\|\gamma\|_2\leq G_2\), and
\(\|\beta\|_2\leq B_\beta\).

Let \(E_{\rm sh}:\mathbb R\to\mathbb R^d\) be the shared affine embedding
and \(P\in\mathbb R^{p\times d}\) the fixed sinusoidal encoding.  The
initial matrix and each post-LayerNorm encoder layer are
\begin{align}
 Z_0&={\rm LN}_0\{E_{\rm sh}(x)+P\},\notag\\
 U_\ell&={\rm LN}_{\ell1}
       \{Z_{\ell-1}+{\rm Att}_\ell(Z_{\ell-1})\},\notag\\
 Z_\ell&={\rm LN}_{\ell2}
       \{U_\ell+{\rm FF}_\ell(U_\ell)\},\qquad
 {\rm FF}_\ell(u)=W_{\ell2}{\rm ReLU}(W_{\ell1}u+b_{\ell1})
                  +b_{\ell2}.                               \label{eq:app-trans-layer}
\end{align}
For head \(h\), use augmented inputs to include query and key biases and
set
\[
 A_{\ell h}
 =\bar W_{K,\ell h}^{\top}\bar W_{Q,\ell h}/\sqrt{d_h},\qquad
 \|A_{\ell h}\|_{\rm op}\leq B_{A,\ell h},\quad
 \|A_{\ell h}^{\top}\|_{2,1}\leq S_{A,\ell h}.
\]
In the covering arguments, \(A_{\ell h}\) is treated as the
effective query--key parameter, with its admissible set restricted to
products attainable by the factorized query and key maps; the attention
output depends on those maps only through \(A_{\ell h}\).
The usual row-softmax head uses scores
\(\bar z_s^\top A_{\ell h}\bar z_t\), affine values
\(\bar V_{\ell h}\bar z_s\), concatenates the \(H\) heads, and applies an
affine output projection.  Finally,
\begin{equation}\label{eq:app-trans-readout}
 g(x)={\rm ReLU}\{W_{\rm flat}{\rm vec}(Z_L)+b_{\rm flat}\}
 \in\mathbb R^r,\quad
 r(x)=W_{\rm out}g(x)+b_{\rm out},\quad f(x)=\Phi\{r(x)\}.
\end{equation}
This is the evaluation-time predictor; all dropout maps are then the
identity.

We record the component bounds used in the recursion.  If
\(\|\bar z\|_2\leq R\), a proper volume net in the affine parameter set
gives
\begin{equation}\label{eq:app-affine-cover}
 \log\mathcal N_\infty(\{z\mapsto\bar W\bar z\},\eta)
 \leq
 \frac{4d_{\rm out}(d_{\rm in}+1)S_a^2R^2}{\eta^2}.
\end{equation}
The bound is independent of the number of evaluation points.  For
LayerNorm, differentiation of \(u\mapsto
u/\{d^{-1}\|u\|_2^2+\epsilon_{\rm LN}\}^{1/2}\) gives
\begin{align}
 \|{\rm LN}(z)-{\rm LN}(z')\|_2
 &\leq L_{\rm LN}\|z-z'\|_2,\qquad
 L_{\rm LN}=G_\infty/\sqrt{\epsilon_{\rm LN}},\notag\\
 \|{\rm LN}(z)\|_2
 &\leq R_{\rm LN}:=G_\infty\sqrt d+B_\beta,\notag\\
 \log\mathcal N_\infty(\{{\rm LN}_{\gamma,\beta}\},\eta)
 &\leq C_{\rm LN}/\eta^2,\qquad
 C_{\rm LN}=8d(\sqrt d\,G_2+B_\beta)^2.                       \label{eq:app-ln-cover}
\end{align}
The last line follows by proper Euclidean volume covers formed
from maximal separated subsets of the admissible \(\gamma\)- and
\(\beta\)-sets, using
\(\|(\gamma-\widetilde\gamma)\odot\nu_\epsilon(z)\|_2
\leq\sqrt d\|\gamma-\widetilde\gamma\|_2\).

For an FFN evaluated on rows of norm at most \(R\), put
\[
 \bar R=(R^2+1)^{1/2},\quad
 R_1=B_1\bar R,\quad \bar R_1=(R_1^2+1)^{1/2},\quad
 c_1=4m(d+1),\quad c_2=4d(m+1).
\]
Two applications of~\eqref{eq:app-affine-cover}, ReLU
nonexpansiveness, and optimal allocation of the two approximation
accuracies give
\begin{align}
 C_{\rm FF}(R)
 &=
 \left[
 c_1^{1/3}(B_2S_1\bar R)^{2/3}
 +c_2^{1/3}(S_2\bar R_1)^{2/3}
 \right]^3,\notag\\
 \log\mathcal N_\infty(\{{\rm FF}\},\eta)
 &\leq C_{\rm FF}(R)/\eta^2,\qquad
 L_{\rm FF}=B_2B_1.                                          \label{eq:app-ff-cover}
\end{align}

For attention, let rows again have norm at most \(R\), put
\(\bar R=(R^2+1)^{1/2}\),
\(c_A=4(d+1)^2\), and \(c_V=4d_h(d+1)\), and define
\begin{align}
 R_h&=B_{V,h}\bar R,\qquad
 L_h=B_{V,h}\{1+4B_{A,h}\bar R^2\},\notag\\
 C_h(R)&=
 \left[
 c_A^{1/3}(2B_{V,h}S_{A,h}\bar R^3)^{2/3}
 +c_V^{1/3}(S_{V,h}\bar R)^{2/3}
 \right]^3,\notag\\
 R_{\rm cat}&=\left(\sum_{h=1}^HR_h^2\right)^{1/2},\qquad
 C_O(R)=4d(d+1)S_O^2(R_{\rm cat}^2+1),\notag\\
 C_{\rm Att}(R)&=
 \left[
 C_O(R)^{1/3}+B_O^{2/3}\sum_{h=1}^HC_h(R)^{1/3}
 \right]^3.                                                  \label{eq:app-att-coeff}
\end{align}
Indeed, Corollary A.7 of \citet{edelman2022inductive} gives
\(\|\operatorname{softmax}(a)-
\operatorname{softmax}(b)\|_1\leq2\|a-b\|_\infty\), so query--key and
value errors \(\eta_A,\eta_V\) change a head by at most
\(2B_{V,h}\bar R^2\eta_A+\eta_V\).  Covering the output projection after
the head covers and allocating all accuracies yields
\begin{equation}\label{eq:app-att-cover}
 \log\mathcal N_\infty(\{{\rm Att}\},\eta)
 \leq C_{\rm Att}(R)/\eta^2,\qquad
 L_{\rm Att}=B_O\left(\sum_{h=1}^HL_h^2\right)^{1/2}.
\end{equation}

We now combine the component bounds.  Let \(R_\star\) be the maximum
LayerNorm output radius in~\eqref{eq:app-ln-cover}.  For encoder layer
\(\ell\), abbreviate the corresponding constants by
\(L_{A,\ell},C_{A,\ell}\),
\(L_{F,\ell},C_{F,\ell}\), and
\(L_{N,\ell j},C_{N,\ell j}\), and set
\begin{align}
 \lambda_\ell
 &=L_{N,\ell2}(1+L_{F,\ell})
   L_{N,\ell1}(1+L_{A,\ell}),\notag\\
 a_\ell&=L_{N,\ell2}(1+L_{F,\ell})L_{N,\ell1},\quad
 b_\ell=L_{N,\ell2}(1+L_{F,\ell}),\quad
 c_\ell=L_{N,\ell2},\notag\\
 P_0&=\prod_{\ell=1}^L\lambda_\ell,\qquad
 P_\ell=\prod_{q=\ell+1}^L\lambda_q.                         \label{eq:app-trans-rec}
\end{align}
For \(\bar B_X=(B_X^2+1)^{1/2}\), define
\begin{align}
 C_E&=8d(S_E\bar B_X)^2,\notag\\
 R_{\rm flat,in}&=(pR_\star^2+1)^{1/2},\qquad
 C_{\rm flat}=4r(pd+1)(S_{\rm flat}R_{\rm flat,in})^2,\notag\\
 R_g&=B_{\rm flat}R_{\rm flat,in},\quad
 R_{\rm out,in}=(R_g^2+1)^{1/2},\quad
 C_{\rm out}=4K(r+1)(S_{\rm out}R_{\rm out,in})^2.            \label{eq:app-trans-endcoeff}
\end{align}
Associate final-output weights
\begin{align}
 w_E&=L_\Phi B_{\rm out}B_{\rm flat}\sqrt p\,P_0L_{N,0},
 &w_{N,0}&=L_\Phi B_{\rm out}B_{\rm flat}\sqrt p\,P_0,\notag\\
 w_{A,\ell}&=L_\Phi B_{\rm out}B_{\rm flat}\sqrt p\,P_\ell a_\ell,
 &w_{N,\ell1}&=L_\Phi B_{\rm out}B_{\rm flat}\sqrt p\,P_\ell b_\ell,\notag\\
 w_{F,\ell}&=L_\Phi B_{\rm out}B_{\rm flat}\sqrt p\,P_\ell c_\ell,
 &w_{N,\ell2}&=L_\Phi B_{\rm out}B_{\rm flat}\sqrt p\,P_\ell,\notag\\
 w_{\rm flat}&=L_\Phi B_{\rm out},&
 w_{\rm out}&=L_\Phi.                                       \label{eq:app-trans-weights}
\end{align}

\begin{theorem}[Complete Trans-CNQ sample-wise entropy]
\label{thm:app-trans-entropy}
Let \(s_{\rm Tr}\) denote the architecture and all budgets above.  For
every realized sample with \(\|x^{(i)}\|_2\leq B_X\),
\begin{equation}\label{eq:app-trans-sample-cover}
 \log\mathcal N_\infty\{\mathcal F_{s_{\rm Tr}},u;
 x^{(1)},\ldots,x^{(N)},\|\cdot\|_2\}
 \leq C_{\rm Tr}(s_{\rm Tr})/u^2,
\end{equation}
where
\begin{align}
 C_{\rm Tr}(s_{\rm Tr})
 =\Bigg[&
 C_E^{1/3}w_E^{2/3}+C_{N,0}^{1/3}w_{N,0}^{2/3}\notag\\
 &+\sum_{\ell=1}^L\{
 C_{A,\ell}^{1/3}w_{A,\ell}^{2/3}
 +C_{N,\ell1}^{1/3}w_{N,\ell1}^{2/3}
 +C_{F,\ell}^{1/3}w_{F,\ell}^{2/3}
 +C_{N,\ell2}^{1/3}w_{N,\ell2}^{2/3}\}\notag\\
 &+C_{\rm flat}^{1/3}w_{\rm flat}^{2/3}
 +C_{\rm out}^{1/3}w_{\rm out}^{2/3}
 \Bigg]^3.                                                   \label{eq:app-trans-C}
\end{align}
For fixed \(s_{\rm Tr}\), this coefficient is independent of \(N\).
\end{theorem}

\begin{proof}
Build each parameter cover conditionally on the preceding finite set of
centers.  If \(\delta_\ell\) is the maximum row error after layer
\(\ell\), residual connections and~\eqref{eq:app-trans-rec} give
\[
 \delta_0\leq L_{N,0}\eta_E+\eta_{N,0},\qquad
 \delta_\ell\leq
 \lambda_\ell\delta_{\ell-1}
 +a_\ell\eta_{A,\ell}+b_\ell\eta_{N,\ell1}
 +c_\ell\eta_{F,\ell}+\eta_{N,\ell2}.
\]
Iteration, followed by
\(\|Z_L-\widetilde Z_L\|_F\leq\sqrt p\,\delta_L\), the two readout
Lipschitz bounds, and~\eqref{eq:app-Phi-lip}, produces exactly the weights
in~\eqref{eq:app-trans-weights}.  The total log-cardinality is
\(\sum_jC_j/\eta_j^2\).  The allocation identity
\[
 \inf_{\sum_jw_j\eta_j=u}\sum_j\frac{C_j}{\eta_j^2}
 =u^{-2}\left(\sum_jC_j^{1/3}w_j^{2/3}\right)^3
\]
proves~\eqref{eq:app-trans-sample-cover}--\eqref{eq:app-trans-C}.
\end{proof}

\begin{remark}
The result is asymptotic.  In the displayed bound, the LayerNorm
Lipschitz factor \(G_\infty/\sqrt{\epsilon_{\rm LN}}\) enters
\(\lambda_\ell\) twice per encoder layer and compounds through \(P_0\).
Consequently, the worst-case upper-bound constant can scale exponentially
with depth and may be numerically vacuous at the sample sizes considered
here.
\end{remark}

\begin{corollary}[Complexity-indexed Trans-CNQ estimation]
\label{cor:app-trans-oracle}
For any deterministic \(s_{{\rm Tr},N}\), define
\(\mathfrak E_N^{\rm Tr}
=NC_{\rm Tr}(s_{{\rm Tr},N})\).  Under the fixed output envelope,
(A1)--(A2), and \(\kappa_N(\eta)=O(N^{-1/2})\) at fixed confidence,
\begin{align}
 R_{\mathcal F_{s_{{\rm Tr},N}}}
   (\widehat f_{N,s_{{\rm Tr},N}})
 =O_{\mathbb P}\Bigg[&
 \frac{\sqrt{\mathfrak E_N^{\rm Tr}}}{N}
 \left\{1+\log_+\!\left(
 \frac{M_f\sqrt K\,N}{\sqrt{\mathfrak E_N^{\rm Tr}}}\right)\right\}
 +\frac1{\sqrt N}\Bigg].                                    \label{eq:app-trans-oracle}
\end{align}
For fixed architecture and budgets this is
\(O_{\mathbb P}(\log N/\sqrt N)\).
\end{corollary}

\begin{proof}
For \(D\in\mathbb R^{N\times K}\),
\(\|D\|_F\leq\sqrt N\max_i\|D_{i\cdot}\|_2\).  Apply
Theorem~\ref{thm:app-trans-entropy} at radius \(u/\sqrt N\), then use
Theorem~\ref{Risk_error}.
\end{proof}

\subsection{Trans-CNQ: approximation and a matched growing sieve}
\label{app:trans-holder}

Retain~\eqref{eq:app-holder} and suppose
\(0<\beta<(p+3)/2\).  The following lemma handles the shared embedding and
the fixed positional encoding explicitly.

\begin{lemma}[Shallow-ReLU subclass of shared-embedding Trans-CNQ]
\label{lem:app-trans-shallow}
Suppose \(d\geq2\), \(H\mid d\), and the fixed encoder norm
budgets are large enough to admit the parameter choices in the proof.
The shared-embedding post-LayerNorm class
\eqref{eq:app-trans-layer}--\eqref{eq:app-trans-readout}, with readout width
at most \(KM\), contains
\begin{equation}\label{eq:app-shallow-form}
 r_k(x)=a_{k0}+\sum_{\ell=1}^M
 a_{k\ell}{\rm ReLU}\{v_{k\ell}^{\top}\chi(x)+b_{k\ell}\},
 \qquad k=1,\ldots,K,
\end{equation}
where
\(\chi(x)=\{\chi_1(x_1),\ldots,\chi_p(x_p)\}\) and every fixed
\(\chi_j:[0,1]\to I_j\) is smooth and bi-Lipschitz.  All encoder parameters
in this subclass are independent of \(M\).
\end{lemma}

\begin{proof}
Choose a unit vector \(e\in\mathbb R^d\) with
\({\bf1}^{\top}e=0\), and set \(E_{\rm sh}(t)=te\).  If \(P_{j\cdot}\) is
the positional vector at token \(j\), put
\(v_j=P_dP_{j\cdot}\) and \(y_j(t)=v_j+te\).  Set every attention and FFN
map to zero.  In every LayerNorm use zero shift and a positive scalar
scale vector \(g_a{\bf1}\).  Induction over the initial and post-residual
LayerNorms gives
\begin{equation}\label{eq:app-ln-path}
 Z_j(t)=\frac{c\,y_j(t)}
              {\{a\|y_j(t)\|_2^2+b\}^{1/2}},
 \qquad c>0,\quad a\geq0,\quad b>0,
\end{equation}
with constants shared across tokens.  Indeed, applying another scalar
LayerNorm to~\eqref{eq:app-ln-path} replaces its denominator by
\[
 \{(c^2/d+\epsilon_{\rm LN}a)\|y_j(t)\|_2^2
   +\epsilon_{\rm LN}b\}^{1/2},
\]
up to a positive numerator constant, and therefore preserves the form.

Let \(\chi_j(t)=e^\top Z_j(t)\).  Differentiation gives
\[
 \chi_j'(t)=
 \frac{c\left[
 b\|e\|_2^2+a\{\|e\|_2^2\|y_j(t)\|_2^2
                  -(e^\top y_j(t))^2\}\right]}
 {\{a\|y_j(t)\|_2^2+b\}^{3/2}}>0.
\]
Strict positivity follows from \(b>0\) and Cauchy--Schwarz.  Compactness
then gives positive lower and finite upper derivative bounds, so every
\(\chi_j\) is smooth and bi-Lipschitz.  Linear functionals of the flattened
tokens extract the \(\chi_j(x_j)\)'s.  The flattened ReLU map and affine
raw-output head can therefore realize~\eqref{eq:app-shallow-form}, using
at most \(KM\) units.
\end{proof}

\begin{theorem}[Hölder approximation by Trans-CNQ]
\label{thm:app-trans-approx}
Under the conditions of Lemma~\ref{lem:app-trans-shallow}, for
\(M,\Lambda\geq1\), there is a function
\(f_{M,\Lambda}=\Phi\circ r_{M,\Lambda}\) in the shared-embedding
Trans-CNQ family, with at most \(KM\) readout units and total absolute
output coefficient at most \(C\Lambda\) in each raw coordinate, such that
\begin{align}
 \sup_x\|f_{M,\Lambda}(x)-q(x)\|_2
 &\leq Ca_{M,\Lambda},\qquad
 L(f_{M,\Lambda})-L(q)\leq Ca_{M,\Lambda},\notag\\
 a_{M,\Lambda}
 &:=
 M^{-\beta/p}\vee
 \Lambda^{-2\beta/(p+3-2\beta)}.                             \label{eq:app-trans-approx}
\end{align}
\end{theorem}

\begin{proof}
By Lemma~\ref{lem:app-trans-shallow}, it is enough to approximate
\[
 z\longmapsto
 r^\circ_k\{\chi_1^{-1}(z_1),\ldots,\chi_p^{-1}(z_p)\}
\]
on \(I_1\times\cdots\times I_p\), where \(r^\circ\) is defined in
\eqref{eq:app-raw-target}.  These transformed functions lie in a bounded
\(C^\beta\) ball.  After an affine rescaling and a standard Hölder
extension to a Euclidean ball, Corollary 2.4 of
\citet{yang2025shallow}, specialized to ReLU, gives a width-\(M\) shallow
network with output variation norm at most \(\Lambda\) and uniform error
\eqref{eq:app-trans-approx}.  Combining the \(K\) coordinatewise networks
uses at most \(KM\) units.  Equations~\eqref{eq:app-Phi-lip} and
\eqref{eq:app-risk-lip} complete the proof.
\end{proof}

Let \(\mathcal F_{M,\Lambda}^{\rm Tr}(M_0)\) be the full norm-restricted
Trans-CNQ class with fixed encoder dimensions and encoder budgets,
readout width \(KM\), and
\begin{equation}\label{eq:app-trans-growing-budgets}
 B_{\rm flat}\leq C\sqrt M,\qquad
 S_{\rm flat}\leq CM,\qquad
 B_{\rm out}+S_{\rm out}\leq C\Lambda,
 \qquad \max_k\|f_k\|_\infty\leq M_0.
\end{equation}
The approximant in Theorem~\ref{thm:app-trans-approx} belongs to this class
when \(M_0\) is sufficiently large.

\begin{lemma}[Growing-readout parameter entropy]
\label{lem:app-trans-growing-cover}
For every realized design, with
\begin{equation}\label{eq:app-trans-Xi}
 \Xi_{M,\Lambda}=C(M+\Lambda+1)\Lambda\sqrt M,
\end{equation}
we have
\begin{equation}\label{eq:app-trans-growing-cover}
 \log\mathcal N\{
 \mathcal F_{M,\Lambda}^{\rm Tr}(M_0)(\mathbf X),
 \|\cdot\|_F,u\}
 \leq
 CM\log\left(1+\frac{\Xi_{M,\Lambda}\sqrt N}{u}\right).
\end{equation}
\end{lemma}

\begin{proof}
For this cover, represent each network by an effective parameter
vector \(\vartheta^{\rm eff}\): its query--key coordinates are the
\(A_{\ell h}\)'s above, restricted to attainable products, and its other
coordinates are the remaining affine and LayerNorm parameters.
The number of scalar effective parameters is at most \(CM\), because the encoder
dimension is fixed and only the two readout matrices grow.  The norm
constraints imply that the concatenated effective parameter vector has radius at
most \(C(M+\Lambda+1)\).  The affine perturbation inequality, the
LayerNorm bound~\eqref{eq:app-ln-cover}, the attention and FFN
perturbation calculations above, and the finite recursion in the proof of
Theorem~\ref{thm:app-trans-entropy} give
\[
 \|f_{\vartheta^{\rm eff}}(x)
      -f_{\widetilde\vartheta^{\rm eff}}(x)\|_2
 \leq C\Lambda\sqrt M\,
       \|\vartheta^{\rm eff}
          -\widetilde\vartheta^{\rm eff}\|_2
\]
uniformly on \([0,1]^p\).  A maximal
\(\delta\)-separated subset of the admissible effective parameter set is
a proper net; volume comparison with
\(\delta=u/(C\Lambda\sqrt{MN})\) then proves
\eqref{eq:app-trans-growing-cover}.  Intersecting with the output envelope
does not increase the covering number because the net is constructed
inside the admissible parameter set.
\end{proof}

\begin{theorem}[Excess-risk convergence rate for the Trans-CNQ estimator]
\label{thm:app-trans-total}
Under~\eqref{eq:app-holder}, the conditions of
Lemma~\ref{lem:app-trans-shallow}, (A1)--(A2), and
\(\kappa_N(\eta)=O(N^{-1/2})\) at fixed confidence, the marginal-IPCW ERM
over \(\mathcal F_{M,\Lambda}^{\rm Tr}(M_0)\) satisfies
\begin{align}
 L(\widehat f_{N,M,\Lambda})-L(q)
 =O_{\mathbb P}\Bigg\{&
 a_{M,\Lambda}
 +\sqrt{\frac MN}
  \left[1+\sqrt{\log(2+\Xi_{M,\Lambda})}\right]
 +\frac1{\sqrt N}\Bigg\}.                                   \label{eq:app-trans-total}
\end{align}
In particular, for
\[
 \nu=\frac{p+3-2\beta}{2p},\qquad
 \Lambda_N=M_N^\nu,\qquad
 M_N\asymp\left\{\frac N{\log N}\right\}^{p/(2\beta+p)},
\]
\[
 L(\widehat f_{N,M_N,\Lambda_N})-L(q)
 =
 O_{\mathbb P}\left[
 \left\{\frac N{\log N}\right\}^{-\beta/(2\beta+p)}\right].
\]
\end{theorem}

\begin{proof}
Theorem~\ref{thm:app-trans-approx} bounds the approximation term.

For estimation, we reuse Steps 1, 2, and 4 of the proof of
Theorem~\ref{Risk_error}, while recomputing Step 3 directly under the logarithmic
entropy bound~\eqref{eq:app-trans-growing-cover}.
After
\(u=\sqrt N\,t\), the entropy integral is bounded by
\[
 C\sqrt{\frac MN}\int_0^{2M_0\sqrt K}
 \sqrt{\log(1+\Xi_{M,\Lambda}/t)}\,dt
 \leq
 C\sqrt{\frac MN}
 \{1+\sqrt{\log(2+\Xi_{M,\Lambda})}\}.
\]
This proves~\eqref{eq:app-trans-total}.  Taking
\(\Lambda=M^\nu\) makes both terms in
\eqref{eq:app-trans-approx} equal to \(M^{-\beta/p}\).
The displayed \(M_N\) balances this with
\(\sqrt{M\log N/N}\), and substitution gives the final rate.
\end{proof}

\begin{remark}

The fitted Trans-CNQ architecture is a fixed member of the
corresponding Transformer sieve.  Under the stated input and norm budgets,
Corollary~\ref{cor:app-kan-oracle} likewise covers the fitted KAN because
its EfficientKAN edge functions satisfy the basis-expansion conditions of
\citet{zhang2024generalization}.\footnote{The spline grid remains
fixed at its initialization in our training pipeline:
\texttt{update\_grid} is never invoked, so the edge dictionary does not
adapt to the training data.}  In contrast, the KAN approximation result,
and hence its total rate, uses the special fixed-grid degree-\(p\) gate in
Lemma~\ref{lem:app-kan-realization}, rather than the fitted default
cubic-spline EfficientKAN architecture.  This distinction applies to all
datasets considered here, whose covariate dimensions are at least seven.
Increasing the grid size refines the number of spline knots but does not
change the cubic spline degree, and therefore does not by itself supply the
exact degree-\(p\) realization used in the proof.  These are upper-bound
comparisons, not evidence that either fitted architecture is statistically
superior.  Neither construction covers TransKAN-CNQ.
\end{remark}

\clearpage

\section{Quantile-Specific Pinball Loss}
\label{app:per_tau_pinball}

This section reports the IPCW-weighted pinball loss decomposed by quantile level $\tau\in\{0.1,0.25,0.5,0.75,0.9\}$ for METABRIC and FLCHAIN, complementing the quantile-averaged values in the main text. Tables~\ref{tab:per-tau-pinball-metabric} and~\ref{tab:per-tau-pinball-flchain} report the mean $\pm$ standard deviation across 25 random splits (log-time scale). On both datasets the loss peaks near the median ($\tau=0.5$) and is smallest in the upper tail. On METABRIC the two Transformer-based variants are at or below KAN-CNQ across all quantile levels; on FLCHAIN the three models are within $0.004$ of one another at every $\tau$, consistent with the aggregate results in the main text.

\begin{table}[t]
  \centering
  \caption{Per-quantile IPCW pinball loss (log-time scale) on METABRIC for the
  three proposed models, reported separately at each of the five quantile levels
  $\tau\in\{0.1,0.25,0.5,0.75,0.9\}$ rather than averaged; entries are the mean
  $\pm$ standard deviation across 25 random splits (lower is better). This
  decomposition complements the quantile-averaged results in the main text and
  shows where in the distribution each model gains: the loss peaks near the
  median ($\tau=0.5$) and is smallest in the upper tail ($\tau=0.9$), and the
  two Transformer-based variants (Trans-CNQ, TransKAN-CNQ) are at or below
  KAN-CNQ at every quantile level. IPCW weights are untruncated, matching the
  headline tables in the main text; sensitivity to truncation is reported in
  Table~\ref{tab:ipcw-trunc-metabric}.}
  \label{tab:per-tau-pinball-metabric}
  \small
  \begin{tabular}{l ccccc}
    \toprule
    Model & $\tau=0.1$ & $\tau=0.25$ & $\tau=0.5$ & $\tau=0.75$ & $\tau=0.9$ \\
    \midrule
    TransKAN-CNQ & 0.168 {\scriptsize $\pm$ 0.014} & 0.278 {\scriptsize $\pm$ 0.024} & 0.310 {\scriptsize $\pm$ 0.031} & 0.215 {\scriptsize $\pm$ 0.033} & 0.103 {\scriptsize $\pm$ 0.018} \\
    Trans-CNQ & 0.171 {\scriptsize $\pm$ 0.013} & 0.283 {\scriptsize $\pm$ 0.022} & 0.321 {\scriptsize $\pm$ 0.034} & 0.231 {\scriptsize $\pm$ 0.037} & 0.115 {\scriptsize $\pm$ 0.026} \\
    KAN-CNQ & 0.177 {\scriptsize $\pm$ 0.013} & 0.290 {\scriptsize $\pm$ 0.020} & 0.325 {\scriptsize $\pm$ 0.033} & 0.229 {\scriptsize $\pm$ 0.043} & 0.114 {\scriptsize $\pm$ 0.032} \\\bottomrule
  \end{tabular}
\end{table}

\begin{table}[t]
  \centering
  \caption{Per-quantile IPCW pinball loss (log-time scale) on FLCHAIN for the
  three proposed models, reported separately at each of the five quantile levels
  $\tau\in\{0.1,0.25,0.5,0.75,0.9\}$; entries are the mean $\pm$ standard
  deviation across 25 random splits (lower is better). The companion to
  Table~\ref{tab:per-tau-pinball-metabric} for the larger, more heavily censored
  ($\approx72\%$) FLCHAIN cohort: the same qualitative pattern holds---the loss
  is largest around the median and smallest in the upper tail. On FLCHAIN the
  three models are within $0.004$ of one another at every $\tau$, so the
  per-quantile ordering is not separated on this cohort. IPCW weights are
  untruncated, matching the headline tables in the main text; sensitivity to
  truncation is reported in Table~\ref{tab:ipcw-trunc-flchain}.}
  \label{tab:per-tau-pinball-flchain}
  \small
  \begin{tabular}{l ccccc}
    \toprule
    Model & $\tau=0.1$ & $\tau=0.25$ & $\tau=0.5$ & $\tau=0.75$ & $\tau=0.9$ \\
    \midrule
    TransKAN-CNQ & 0.278 {\scriptsize $\pm$ 0.022} & 0.396 {\scriptsize $\pm$ 0.024} & 0.378 {\scriptsize $\pm$ 0.020} & 0.237 {\scriptsize $\pm$ 0.012} & 0.107 {\scriptsize $\pm$ 0.006} \\
    Trans-CNQ & 0.280 {\scriptsize $\pm$ 0.021} & 0.399 {\scriptsize $\pm$ 0.023} & 0.382 {\scriptsize $\pm$ 0.021} & 0.239 {\scriptsize $\pm$ 0.014} & 0.108 {\scriptsize $\pm$ 0.006} \\
    KAN-CNQ & 0.277 {\scriptsize $\pm$ 0.021} & 0.396 {\scriptsize $\pm$ 0.023} & 0.380 {\scriptsize $\pm$ 0.021} & 0.238 {\scriptsize $\pm$ 0.013} & 0.107 {\scriptsize $\pm$ 0.006} \\\bottomrule
  \end{tabular}
\end{table}

\section{Simulation Study Details}
\label{app:sim}

\subsection{Design}
To assess the performance and robustness of our proposed methods, we conducted a simulation study with the covariate dimension fixed at $p=10$. We considered three event-time distributions: Gaussian, Gamma, and Weibull. The distributional parameters (the mean and variance for Gaussian, and the shape and scale for Gamma and Weibull) were modeled as functions of the covariates. Censoring times were generated from a uniform distribution on $[0,c]$, with $c$ adjusted to achieve three censoring proportions: light ($\approx 25\%$), medium ($\approx 50\%$), and heavy ($\approx 75\%$). We further varied the sample size over $n\in\{150,750,1500\}$. The combination of three event-time distributions, three censoring levels, and three sample sizes yields $27$ simulation settings.

\subsection{Performance Comparisons}
Table~\ref{tab:sim-pinball} reports the IPCW pinball loss for all 27 settings, and Figure~\ref{fig:sim-grid} visualizes the Weibull results. The picture is strongly \emph{distribution-dependent}. Under the Weibull and Gamma designs, where covariates reshape the distributional shape and tail behavior, the proposed models are clearly best, attaining the lowest pinball loss in $7/9$ and $6/9$ settings respectively, with the largest margins at moderate-to-large samples (e.g., Weibull with $50\%$ censoring and $n=1500$: Trans-CNQ $0.112$ versus the best baseline RSF $0.176$). Under the symmetric Gaussian design, by contrast, the quantile- and hazard-based methods perform very similarly (pinball losses of roughly $0.02$--$0.06$; only AFT under heavy censoring and CQRNN fall outside this range), and Random Survival Forest attains the lowest value in eight of the nine settings, though it leads the proposed models only marginally (e.g., $0.024$ versus $0.025$). This pattern is exactly what one would expect: the advantage of flexible distributional modeling is greatest when the conditional distribution is non-symmetric and its shape varies with covariates, and negligible when a symmetric location-family already captures the outcome.

The clearest and most consistent gains are relative to the other quantile-based neural methods. CQRNN is outperformed by the best proposed model in all $27$ settings, with pinball losses ranging from $0.068$ to $0.548$ and a particularly large gap under the Gaussian design, where it is three to five times less accurate than every other method. DeepQuantreg, which fits one network per quantile level under the same IPCW objective, is the more informative comparison because it isolates the effect of joint non-crossing estimation from the choice of loss: each of the three proposed models attains a lower pinball loss than DeepQuantreg in $25$ of the settings in which both were run, and the best proposed model does so in $26$ of $27$. The exceptions are confined to $n=150$ under medium or heavy censoring in the Gamma and Weibull designs, where all methods are unstable; at $n\ge750$ the gap is uniform and often large (for example, Weibull with $75\%$ censoring at $n=1500$: $0.117$ versus $0.321$). DeepQuantreg is nonetheless the strongest non-proposed method in $2$ of the $27$ settings and is never the weakest, so the ordering reflects the estimation strategy rather than a poorly tuned baseline. Random Survival Forest is the strongest overall competitor (it wins the Gaussian settings and is a close second under Weibull/Gamma), whereas the hazard-based deep learners (DeepSurv, CoxCC, CoxTime, PC-Hazard) and classical models (Cox, AFT, CTQR) are less accurate distributional predictors, with CTQR and AFT the weakest. Within the proposed family, TransKAN-CNQ is the most consistently strong (lowest of the three in $18$ of the $27$ settings), Trans-CNQ is a very close second ($7/27$) and occasionally preferable under small samples or heavy censoring, and KAN-CNQ is typically less competitive but still improves on most baselines. Larger $n$ and lighter censoring improve all methods; at the smallest sample size ($n=150$) the proposed models' advantage narrows and RSF becomes more competitive even under Weibull/Gamma.
\begin{sidewaystable}
  \centering
  \caption{Simulation study: IPCW pinball loss (self-normalized, log-time scale, averaged over $\tau\in\{0.1,0.25,0.5,0.75,0.9\}$) across all 27 settings. Each cell is the mean over 25 replications with the standard deviation in parentheses; lower is better. The best \emph{proposed} model per row is \textbf{bolded}. A dash (--) marks one configuration (Gamma, $50\%$ censoring, $n=150$) for which the TransKAN-CNQ run is unavailable. Gaussian, Gamma, and Weibull are abbreviated as GS, GM, and WB, respectively.}
  \label{tab:sim-pinball}
  \setlength{\tabcolsep}{2pt}\tiny
  \begin{tabular}{lll ccc c cccccc cccc}
    \toprule
    & & & \multicolumn{3}{c}{\emph{Proposed}} & & \multicolumn{6}{c}{\emph{Deep learning}} & \multicolumn{4}{c}{\emph{Traditional}} \\
    \cmidrule(lr){4-6}\cmidrule(lr){8-13}\cmidrule(lr){14-17}
    Dist. & Censoring & $n$ & TransKAN-CNQ & Trans-CNQ & KAN-CNQ & & DeepQuantreg & CQRNN & DeepSurv & CoxCC & CoxTime & PC-Hazard & RSF & Cox & AFT & CTQR \\
    \midrule
    GS & Light (25\%) & 150 & 0.050 {\scriptsize (0.033)} & \textbf{0.047} {\scriptsize (0.026)} & 0.053 {\scriptsize (0.026)} & & 0.112 {\scriptsize (0.026)} & 0.150 {\scriptsize (0.035)} & 0.052 {\scriptsize (0.032)} & 0.052 {\scriptsize (0.032)} & 0.051 {\scriptsize (0.033)} & 0.073 {\scriptsize (0.035)} & 0.048 {\scriptsize (0.036)} & 0.053 {\scriptsize (0.033)} & 0.060 {\scriptsize (0.033)} & 0.054 {\scriptsize (0.039)} \\
     & Light (25\%) & 750 & 0.029 {\scriptsize (0.011)} & \textbf{0.029} {\scriptsize (0.010)} & 0.032 {\scriptsize (0.009)} & & 0.048 {\scriptsize (0.005)} & 0.086 {\scriptsize (0.041)} & 0.029 {\scriptsize (0.007)} & 0.030 {\scriptsize (0.007)} & 0.029 {\scriptsize (0.007)} & 0.030 {\scriptsize (0.007)} & 0.027 {\scriptsize (0.009)} & 0.035 {\scriptsize (0.012)} & 0.045 {\scriptsize (0.021)} & 0.034 {\scriptsize (0.014)} \\
     & Light (25\%) & 1500 & \textbf{0.025} {\scriptsize (0.006)} & 0.025 {\scriptsize (0.006)} & 0.028 {\scriptsize (0.006)} & & 0.035 {\scriptsize (0.006)} & 0.068 {\scriptsize (0.016)} & 0.026 {\scriptsize (0.007)} & 0.027 {\scriptsize (0.007)} & 0.025 {\scriptsize (0.006)} & 0.024 {\scriptsize (0.006)} & 0.024 {\scriptsize (0.006)} & 0.031 {\scriptsize (0.008)} & 0.036 {\scriptsize (0.011)} & 0.029 {\scriptsize (0.007)} \\
     & Medium (50\%) & 150 & \textbf{0.043} {\scriptsize (0.014)} & 0.043 {\scriptsize (0.012)} & 0.051 {\scriptsize (0.013)} & & 0.130 {\scriptsize (0.054)} & 0.147 {\scriptsize (0.038)} & 0.043 {\scriptsize (0.014)} & 0.044 {\scriptsize (0.013)} & 0.044 {\scriptsize (0.013)} & 0.095 {\scriptsize (0.074)} & 0.038 {\scriptsize (0.009)} & 0.043 {\scriptsize (0.012)} & 0.067 {\scriptsize (0.022)} & 0.042 {\scriptsize (0.010)} \\
     & Medium (50\%) & 750 & \textbf{0.031} {\scriptsize (0.014)} & 0.033 {\scriptsize (0.015)} & 0.035 {\scriptsize (0.014)} & & 0.058 {\scriptsize (0.008)} & 0.159 {\scriptsize (0.028)} & 0.031 {\scriptsize (0.009)} & 0.032 {\scriptsize (0.009)} & 0.032 {\scriptsize (0.010)} & 0.032 {\scriptsize (0.008)} & 0.029 {\scriptsize (0.009)} & 0.036 {\scriptsize (0.013)} & 0.058 {\scriptsize (0.037)} & 0.036 {\scriptsize (0.016)} \\
     & Medium (50\%) & 1500 & \textbf{0.027} {\scriptsize (0.008)} & 0.028 {\scriptsize (0.008)} & 0.030 {\scriptsize (0.008)} & & 0.041 {\scriptsize (0.010)} & 0.151 {\scriptsize (0.041)} & 0.027 {\scriptsize (0.009)} & 0.028 {\scriptsize (0.009)} & 0.027 {\scriptsize (0.008)} & 0.026 {\scriptsize (0.008)} & 0.025 {\scriptsize (0.007)} & 0.033 {\scriptsize (0.010)} & 0.059 {\scriptsize (0.057)} & 0.032 {\scriptsize (0.011)} \\
     & Heavy (75\%) & 150 & \textbf{0.042} {\scriptsize (0.018)} & 0.045 {\scriptsize (0.021)} & 0.059 {\scriptsize (0.021)} & & 0.137 {\scriptsize (0.035)} & 0.121 {\scriptsize (0.039)} & 0.040 {\scriptsize (0.011)} & 0.040 {\scriptsize (0.012)} & 0.042 {\scriptsize (0.013)} & 0.096 {\scriptsize (0.035)} & 0.036 {\scriptsize (0.010)} & 0.042 {\scriptsize (0.014)} & 0.170 {\scriptsize (0.079)} & 0.040 {\scriptsize (0.013)} \\
     & Heavy (75\%) & 750 & 0.032 {\scriptsize (0.012)} & \textbf{0.032} {\scriptsize (0.010)} & 0.039 {\scriptsize (0.014)} & & 0.079 {\scriptsize (0.010)} & 0.134 {\scriptsize (0.018)} & 0.032 {\scriptsize (0.010)} & 0.034 {\scriptsize (0.011)} & 0.033 {\scriptsize (0.010)} & 0.034 {\scriptsize (0.012)} & 0.030 {\scriptsize (0.010)} & 0.034 {\scriptsize (0.010)} & 0.178 {\scriptsize (0.100)} & 0.035 {\scriptsize (0.011)} \\
     & Heavy (75\%) & 1500 & \textbf{0.027} {\scriptsize (0.007)} & 0.028 {\scriptsize (0.007)} & 0.032 {\scriptsize (0.008)} & & 0.048 {\scriptsize (0.006)} & 0.115 {\scriptsize (0.017)} & 0.028 {\scriptsize (0.007)} & 0.029 {\scriptsize (0.007)} & 0.028 {\scriptsize (0.007)} & 0.028 {\scriptsize (0.007)} & 0.026 {\scriptsize (0.007)} & 0.031 {\scriptsize (0.008)} & 0.160 {\scriptsize (0.102)} & 0.031 {\scriptsize (0.008)} \\
    \addlinespace
    GM & Light (25\%) & 150 & \textbf{0.371} {\scriptsize (0.128)} & 0.390 {\scriptsize (0.118)} & 0.426 {\scriptsize (0.069)} & & 0.509 {\scriptsize (0.038)} & 0.472 {\scriptsize (0.046)} & 0.492 {\scriptsize (0.056)} & 0.499 {\scriptsize (0.058)} & 0.477 {\scriptsize (0.045)} & 0.599 {\scriptsize (0.064)} & 0.387 {\scriptsize (0.027)} & 0.563 {\scriptsize (0.052)} & 0.563 {\scriptsize (0.051)} & 0.902 {\scriptsize (0.167)} \\
     & Light (25\%) & 750 & \textbf{0.196} {\scriptsize (0.010)} & 0.200 {\scriptsize (0.011)} & 0.240 {\scriptsize (0.014)} & & 0.283 {\scriptsize (0.020)} & 0.330 {\scriptsize (0.026)} & 0.293 {\scriptsize (0.026)} & 0.324 {\scriptsize (0.030)} & 0.318 {\scriptsize (0.031)} & 0.409 {\scriptsize (0.036)} & 0.275 {\scriptsize (0.011)} & 0.533 {\scriptsize (0.021)} & 0.534 {\scriptsize (0.022)} & 0.863 {\scriptsize (0.084)} \\
     & Light (25\%) & 1500 & \textbf{0.192} {\scriptsize (0.007)} & 0.199 {\scriptsize (0.007)} & 0.213 {\scriptsize (0.006)} & & 0.214 {\scriptsize (0.007)} & 0.312 {\scriptsize (0.013)} & 0.239 {\scriptsize (0.016)} & 0.291 {\scriptsize (0.017)} & 0.287 {\scriptsize (0.017)} & 0.373 {\scriptsize (0.019)} & 0.255 {\scriptsize (0.006)} & 0.532 {\scriptsize (0.014)} & 0.532 {\scriptsize (0.014)} & 0.868 {\scriptsize (0.050)} \\
     & Medium (50\%) & 150 & -- & 0.513 {\scriptsize (0.050)} & \textbf{0.502} {\scriptsize (0.047)} & & 0.503 {\scriptsize (0.031)} & 0.519 {\scriptsize (0.058)} & 0.520 {\scriptsize (0.061)} & 0.525 {\scriptsize (0.074)} & 0.531 {\scriptsize (0.070)} & 0.611 {\scriptsize (0.081)} & 0.395 {\scriptsize (0.039)} & 0.559 {\scriptsize (0.076)} & 0.562 {\scriptsize (0.078)} & 0.676 {\scriptsize (0.092)} \\
     & Medium (50\%) & 750 & \textbf{0.203} {\scriptsize (0.015)} & 0.206 {\scriptsize (0.015)} & 0.245 {\scriptsize (0.016)} & & 0.403 {\scriptsize (0.025)} & 0.374 {\scriptsize (0.048)} & 0.317 {\scriptsize (0.028)} & 0.356 {\scriptsize (0.033)} & 0.345 {\scriptsize (0.028)} & 0.420 {\scriptsize (0.036)} & 0.286 {\scriptsize (0.016)} & 0.527 {\scriptsize (0.033)} & 0.530 {\scriptsize (0.034)} & 0.709 {\scriptsize (0.214)} \\
     & Medium (50\%) & 1500 & \textbf{0.185} {\scriptsize (0.008)} & 0.187 {\scriptsize (0.008)} & 0.209 {\scriptsize (0.006)} & & 0.278 {\scriptsize (0.029)} & 0.360 {\scriptsize (0.045)} & 0.257 {\scriptsize (0.013)} & 0.302 {\scriptsize (0.013)} & 0.288 {\scriptsize (0.017)} & 0.372 {\scriptsize (0.023)} & 0.259 {\scriptsize (0.009)} & 0.519 {\scriptsize (0.024)} & 0.521 {\scriptsize (0.024)} & 0.666 {\scriptsize (0.028)} \\
     & Heavy (75\%) & 150 & 0.469 {\scriptsize (0.054)} & \textbf{0.465} {\scriptsize (0.044)} & 0.474 {\scriptsize (0.049)} & & 0.452 {\scriptsize (0.043)} & 0.548 {\scriptsize (0.088)} & 0.579 {\scriptsize (0.093)} & 0.580 {\scriptsize (0.102)} & 0.580 {\scriptsize (0.097)} & 0.573 {\scriptsize (0.090)} & 0.426 {\scriptsize (0.059)} & 0.619 {\scriptsize (0.096)} & 0.629 {\scriptsize (0.098)} & 0.630 {\scriptsize (0.129)} \\
     & Heavy (75\%) & 750 & 0.369 {\scriptsize (0.095)} & 0.409 {\scriptsize (0.059)} & \textbf{0.349} {\scriptsize (0.065)} & & 0.425 {\scriptsize (0.016)} & 0.503 {\scriptsize (0.039)} & 0.375 {\scriptsize (0.037)} & 0.427 {\scriptsize (0.045)} & 0.429 {\scriptsize (0.038)} & 0.450 {\scriptsize (0.027)} & 0.299 {\scriptsize (0.019)} & 0.569 {\scriptsize (0.043)} & 0.577 {\scriptsize (0.044)} & 0.579 {\scriptsize (0.045)} \\
     & Heavy (75\%) & 1500 & \textbf{0.190} {\scriptsize (0.011)} & 0.196 {\scriptsize (0.014)} & 0.237 {\scriptsize (0.012)} & & 0.422 {\scriptsize (0.017)} & 0.409 {\scriptsize (0.046)} & 0.281 {\scriptsize (0.020)} & 0.340 {\scriptsize (0.037)} & 0.329 {\scriptsize (0.043)} & 0.409 {\scriptsize (0.036)} & 0.262 {\scriptsize (0.017)} & 0.554 {\scriptsize (0.049)} & 0.562 {\scriptsize (0.051)} & 0.568 {\scriptsize (0.045)} \\
    \addlinespace
    WB & Light (25\%) & 150 & \textbf{0.230} {\scriptsize (0.054)} & 0.255 {\scriptsize (0.043)} & 0.266 {\scriptsize (0.064)} & & 0.342 {\scriptsize (0.032)} & 0.334 {\scriptsize (0.061)} & 0.355 {\scriptsize (0.048)} & 0.366 {\scriptsize (0.048)} & 0.350 {\scriptsize (0.047)} & 0.435 {\scriptsize (0.071)} & 0.278 {\scriptsize (0.025)} & 0.391 {\scriptsize (0.050)} & 0.391 {\scriptsize (0.048)} & 0.746 {\scriptsize (0.167)} \\
     & Light (25\%) & 750 & 0.121 {\scriptsize (0.014)} & \textbf{0.121} {\scriptsize (0.013)} & 0.142 {\scriptsize (0.029)} & & 0.202 {\scriptsize (0.014)} & 0.293 {\scriptsize (0.077)} & 0.191 {\scriptsize (0.013)} & 0.243 {\scriptsize (0.030)} & 0.234 {\scriptsize (0.028)} & 0.299 {\scriptsize (0.047)} & 0.197 {\scriptsize (0.030)} & 0.384 {\scriptsize (0.086)} & 0.390 {\scriptsize (0.077)} & 0.756 {\scriptsize (0.284)} \\
     & Light (25\%) & 1500 & 0.115 {\scriptsize (0.004)} & \textbf{0.112} {\scriptsize (0.004)} & 0.119 {\scriptsize (0.003)} & & 0.151 {\scriptsize (0.005)} & 0.215 {\scriptsize (0.021)} & 0.159 {\scriptsize (0.007)} & 0.202 {\scriptsize (0.019)} & 0.207 {\scriptsize (0.019)} & 0.258 {\scriptsize (0.019)} & 0.168 {\scriptsize (0.004)} & 0.363 {\scriptsize (0.013)} & 0.356 {\scriptsize (0.075)} & 0.706 {\scriptsize (0.050)} \\
     & Medium (50\%) & 150 & \textbf{0.296} {\scriptsize (0.072)} & 0.345 {\scriptsize (0.070)} & 0.365 {\scriptsize (0.082)} & & 0.396 {\scriptsize (0.053)} & 0.478 {\scriptsize (0.042)} & 0.394 {\scriptsize (0.071)} & 0.399 {\scriptsize (0.068)} & 0.378 {\scriptsize (0.061)} & 0.479 {\scriptsize (0.110)} & 0.282 {\scriptsize (0.041)} & 0.416 {\scriptsize (0.071)} & 0.415 {\scriptsize (0.071)} & 0.656 {\scriptsize (0.202)} \\
     & Medium (50\%) & 750 & \textbf{0.127} {\scriptsize (0.008)} & 0.131 {\scriptsize (0.012)} & 0.151 {\scriptsize (0.012)} & & 0.269 {\scriptsize (0.024)} & 0.366 {\scriptsize (0.023)} & 0.218 {\scriptsize (0.020)} & 0.283 {\scriptsize (0.032)} & 0.265 {\scriptsize (0.026)} & 0.296 {\scriptsize (0.028)} & 0.200 {\scriptsize (0.009)} & 0.390 {\scriptsize (0.024)} & 0.395 {\scriptsize (0.024)} & 0.638 {\scriptsize (0.076)} \\
     & Medium (50\%) & 1500 & \textbf{0.112} {\scriptsize (0.006)} & 0.112 {\scriptsize (0.005)} & 0.124 {\scriptsize (0.006)} & & 0.173 {\scriptsize (0.010)} & 0.285 {\scriptsize (0.016)} & 0.173 {\scriptsize (0.009)} & 0.208 {\scriptsize (0.018)} & 0.209 {\scriptsize (0.015)} & 0.261 {\scriptsize (0.027)} & 0.176 {\scriptsize (0.007)} & 0.388 {\scriptsize (0.018)} & 0.393 {\scriptsize (0.018)} & 0.669 {\scriptsize (0.065)} \\
     & Heavy (75\%) & 150 & 0.421 {\scriptsize (0.079)} & \textbf{0.415} {\scriptsize (0.070)} & 0.438 {\scriptsize (0.073)} & & 0.421 {\scriptsize (0.055)} & 0.467 {\scriptsize (0.110)} & 0.453 {\scriptsize (0.097)} & 0.459 {\scriptsize (0.103)} & 0.437 {\scriptsize (0.089)} & 0.507 {\scriptsize (0.083)} & 0.319 {\scriptsize (0.066)} & 0.465 {\scriptsize (0.098)} & 0.471 {\scriptsize (0.101)} & 0.554 {\scriptsize (0.162)} \\
     & Heavy (75\%) & 750 & \textbf{0.159} {\scriptsize (0.060)} & 0.194 {\scriptsize (0.071)} & 0.181 {\scriptsize (0.057)} & & 0.377 {\scriptsize (0.019)} & 0.395 {\scriptsize (0.029)} & 0.281 {\scriptsize (0.036)} & 0.375 {\scriptsize (0.070)} & 0.366 {\scriptsize (0.040)} & 0.371 {\scriptsize (0.052)} & 0.223 {\scriptsize (0.012)} & 0.444 {\scriptsize (0.033)} & 0.452 {\scriptsize (0.033)} & 0.543 {\scriptsize (0.061)} \\
     & Heavy (75\%) & 1500 & \textbf{0.117} {\scriptsize (0.009)} & 0.129 {\scriptsize (0.012)} & 0.124 {\scriptsize (0.006)} & & 0.321 {\scriptsize (0.021)} & 0.365 {\scriptsize (0.032)} & 0.209 {\scriptsize (0.017)} & 0.247 {\scriptsize (0.036)} & 0.250 {\scriptsize (0.030)} & 0.260 {\scriptsize (0.027)} & 0.190 {\scriptsize (0.013)} & 0.434 {\scriptsize (0.030)} & 0.443 {\scriptsize (0.031)} & 0.551 {\scriptsize (0.034)} \\
    \bottomrule
  \end{tabular}
\end{sidewaystable}

\begin{figure}[t]
  \centering
  \includegraphics[width=\textwidth]{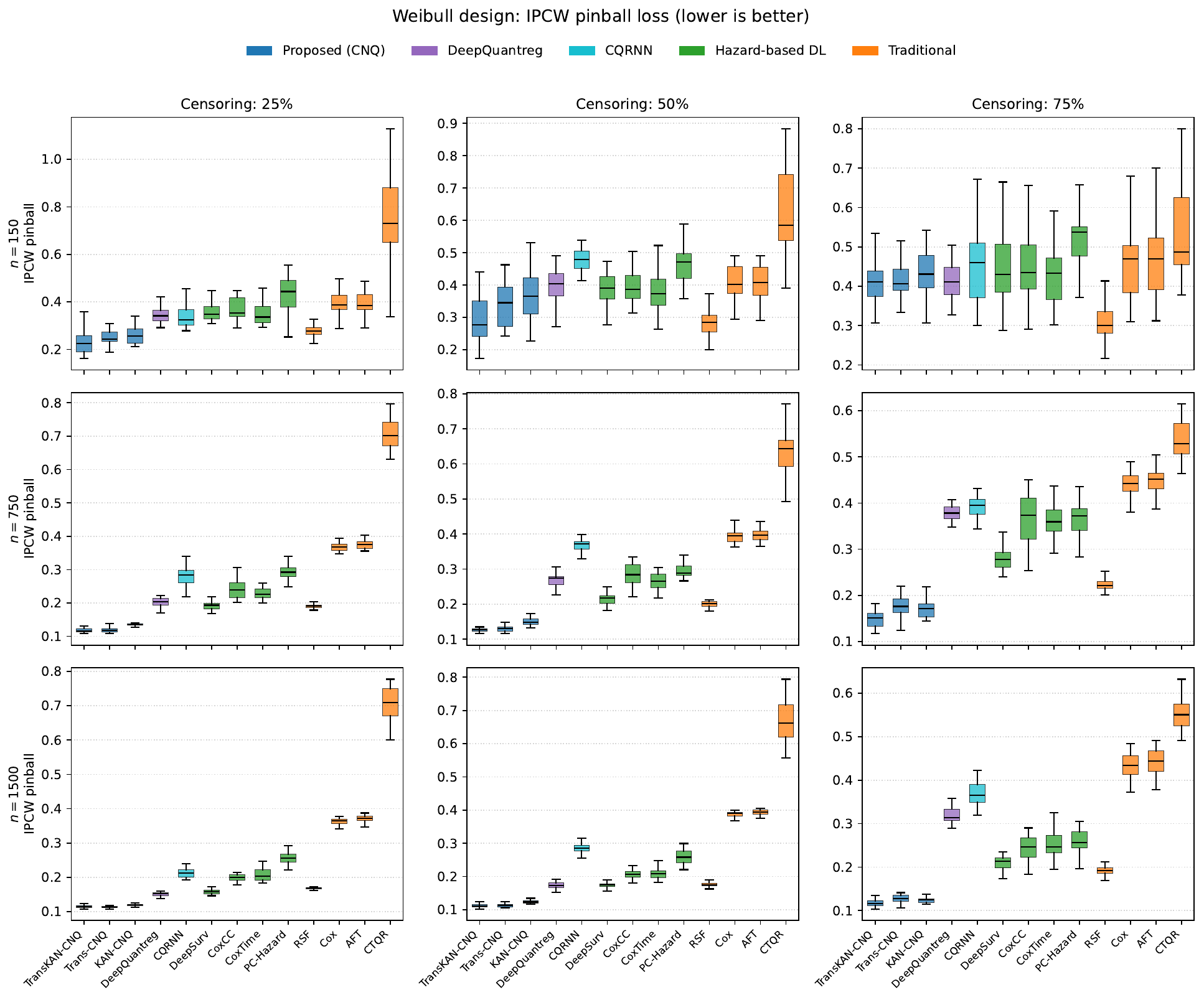}
  \caption{Simulation study, Weibull design: inverse-probability-of-censoring-weighted
  (IPCW) pinball loss (self-normalized, log-time scale, averaged over the five
  quantile levels $\tau\in\{0.1,0.25,0.5,0.75,0.9\}$) by method. The figure is
  a $3\times3$ grid of panels indexed by sample size (rows, $n\in\{150,750,1500\}$)
  and censoring level (columns, light $\approx25\%$, medium $\approx50\%$, heavy
  $\approx75\%$); within each panel the methods are arranged along the horizontal
  axis and the vertical axis reports the pinball loss (lower is better), with
  boxes/markers summarizing the distribution over the 25 replications. Colors
  distinguish the three proposed CNQ models (blue), DeepQuantreg (purple) and
  CQRNN (cyan) from the remaining deep-learning and
  traditional baselines. The proposed models attain the lowest loss in seven of
  the nine panels; the two exceptions are $n=150$ under medium and heavy
  censoring, where Random Survival Forest (RSF) is slightly lower. The margin
  over RSF, the strongest competing baseline, widens markedly as the sample size
  grows (from $-0.096$ to $+0.048$ at $n=150$ to between $+0.056$ and $+0.073$ at
  $n=1500$) and is otherwise stable across censoring levels, visualizing the
  Weibull rows of Table~\ref{tab:sim-pinball}.}
  \label{fig:sim-grid}
\end{figure}

\subsection{Quantile Crossing}
\label{app:crossing}
A practical advantage of the framework is that it enforces valid quantile ordering
by construction. The CNQ output layer returns the lowest quantile together with
non-negative increments, so $\widehat q_{\tau_1}\le\cdots\le\widehat q_{\tau_K}$
holds identically for every subject. We verified this directly on the fitted
models: across all datasets, settings and seeds, the measured crossing rate of all
three CNQ architectures is exactly zero, as the parameterization guarantees.

To quantify what is given up without this constraint, we refit the same five
quantile levels $\{0.1,0.25,0.5,0.75,0.9\}$ with DeepQuantreg, which trains one
network per level under the same IPCW objective, and measured how often the
resulting curves violate ordering. Table~\ref{tab:crossing} reports four rates,
each averaged over the 25 seeds and, for the simulation, over the 27 settings:
the fraction of test subjects with at least one adjacent-pair violation; the
fraction of (subject, adjacent-pair) combinations that are violated; and the
fraction of subjects for whom the two intervals used in this paper inverted,
$\widehat q_{0.1}>\widehat q_{0.9}$ and $\widehat q_{0.25}>\widehat q_{0.75}$.
These are the same DeepQuantreg runs whose losses appear in
Table~\ref{tab:sim-pinball} and in the real-data pinball table of the main text.

\begin{table}[h]
  \centering
  \caption{Quantile-crossing rates for DeepQuantreg when the five levels are fit
  separately (mean over 25 seeds; the simulation figure additionally averages the
  27 settings). The corresponding rate for all three CNQ models is $0$, both by
  construction and as measured on the fitted models. ``Any adjacent'' is the fraction of test subjects with at least
  one violated adjacent pair; ``adjacent pair'' is the fraction of
  (subject, pair) combinations violated.}
  \label{tab:crossing}
  \small
  \begin{tabular}{lcccc}
    \toprule
    & Any adjacent & Adjacent pair & $\widehat q_{0.1}>\widehat q_{0.9}$ & $\widehat q_{0.25}>\widehat q_{0.75}$ \\
    \midrule
    Real data (6 cohorts)   & $18.24\%$ & $5.07\%$  & $0.18\%$ & $0.45\%$ \\
    Simulation (27 settings) & $34.17\%$ & $10.49\%$ & $2.11\%$ & $6.12\%$ \\
    \bottomrule
  \end{tabular}
\end{table}

Crossing is thus not a rare edge case: at least one adjacent pair is inverted for
roughly one in five test subjects on the real cohorts and one in three in the
simulation, and $146$ of the $150$ real fits ($97.3\%$) and $648$ of the $675$
simulated fits ($96.0\%$) contain at least one crossing subject. The rate varies
widely with the data, from $2.4\%$ of subjects on SUPPORT to $45.7\%$ on the
144-subject NKI70 cohort, and reaches $65.2\%$ in the hardest simulation setting
(Weibull, $75\%$ censoring, $n=1500$). Inversions of the outer interval are much
rarer than adjacent-pair violations in the real cohorts ($0.18\%$), so the
$80\%$ coverage reported for DeepQuantreg in the main text is
computed on intervals that are almost always well defined; in the simulation,
however, the outer interval inverts for $2.11\%$ of subjects overall and for
$29.6\%$ in the worst Weibull setting, where the reported coverage of a
``prediction interval'' is correspondingly hard to interpret.

Crossing predictions are incompatible with a coherent conditional distribution and
undermine interpretability, especially when quantiles summarize short-,
intermediate-, and long-term prognosis; eliminating this pathology at the model
level rather than by post hoc correction yields predictions that are both more
coherent and more suitable for downstream clinical interpretation.

\clearpage
\section{Supplementary Tables}

This section collects the numerical results underlying the individual-level and
uncertainty analyses summarized in the main text, first for METABRIC and then
for FLCHAIN. Table~\ref{tab:metabric-profiles} lists the predicted quantile
milestones for five representative METABRIC patients: reading each row from
$\hat q_{0.1}$ to $\hat q_{0.9}$ traces the full predicted survival distribution
for that patient, and the accompanying covariates and observed outcome show how
the model turns a covariate vector into interpretable short-, intermediate-, and
long-term milestones. In particular, patients with similar predicted medians can
differ substantially in their lower-tail risk and in the width of their
predicted interval, which is precisely the individualized information a single
risk score cannot convey. Table~\ref{tab:metabric-unc} aggregates this behavior
across the cohort by stratifying test subjects into terciles of predicted
$80\%$-interval width $w_i=\hat q_{0.9,i}-\hat q_{0.1,i}$: for every architecture,
both the pinball loss and the empirical $50\%$/$80\%$ coverage increase
monotonically from the narrow to the wide tercile, showing that the model's own
interval width is a usable, label-free indicator of predictive difficulty.

\begin{table}[t]
  \centering
  \caption{Predicted quantile milestones for five representative METABRIC
  patients (Trans-CNQ, seed 42), with survival reported in months. For each
  patient (A--E) the table lists the five predicted quantiles
  $\hat q_{0.1},\hat q_{0.25},\hat q_{0.5},\hat q_{0.75},\hat q_{0.9}$
  (non-crossing by construction, so they increase from left to right), the
  observed follow-up time $T_{\mathrm{obs}}$, the event indicator, and two key
  clinical covariates (estrogen-receptor status ER and age in years). The rows
  illustrate how the model translates covariates into individualized,
  interpretable survival-time quantiles---for example, the ER$-$ patient D
  receives a wider predicted interval and a longer upper tail than the ER+
  patients---and correspond to the METABRIC profiles plotted in
  Figure~\ref{fig:patient-profiles}(a).}
  \label{tab:metabric-profiles}
  \small
  \begin{tabular}{lrrrrr rr cc}
    \toprule
    & $\hat{q}_{0.1}$ & $\hat{q}_{0.25}$ & $\hat{q}_{0.5}$ & $\hat{q}_{0.75}$ & $\hat{q}_{0.9}$
    & $T_{\mathrm{obs}}$ & Event & ER & Age \\
    \midrule
    Patient A & 28.3 & 48.7 &  91.5 & 138.8 & 184.8 &  41.8 & Yes & ER+ & 86 \\
    Patient B & 34.0 & 56.6 & 100.6 & 153.5 & 204.4 & 119.5 & Yes & ER+ & 69 \\
    Patient C & 36.9 & 64.7 & 102.7 & 149.9 & 187.8 &  85.6 & Yes & ER+ & 78 \\
    Patient D & 29.8 & 56.3 & 119.3 & 211.2 & 297.7 & 104.7 & Yes & ER$-$ & 58 \\
    Patient E & 34.0 & 56.7 &  99.8 & 147.9 & 192.4 & 100.3 & Yes & ER+ & 74 \\
    \bottomrule
  \end{tabular}
\end{table}

\begin{table}[t]
  \centering
  \caption{Uncertainty stratification on METABRIC (averaged over 25 seeds) for
  each of the three proposed models. Test subjects are grouped into terciles
  ---narrow, medium, wide---according to the predicted $80\%$ interval width
  $w_i=\hat q_{0.9,i}-\hat q_{0.1,i}$, an internal, label-free difficulty score.
  For each model and tercile the table reports the pinball loss (survival scale,
  in months) and the empirical coverage of the nominal $80\%$ and $50\%$
  prediction intervals. Across all three models the pinball loss and both
  coverages rise monotonically from the narrow to the wide tercile, confirming
  that predictions the model itself flags as uncertain (wider intervals) are
  genuinely harder and receive higher coverage; this is the tabular counterpart
  of Figure~\ref{fig:uncertainty}(a).}
  \label{tab:metabric-unc}
  \small
  \begin{tabular}{ll ccc}
    \toprule
    Model & Tercile & Pinball Loss & 80\% Cov. & 50\% Cov. \\
    \midrule
    TransKAN-CNQ  & Narrow & 20.10 & 0.759 & 0.476 \\
                  & Medium & 26.30 & 0.747 & 0.451 \\
                  & Wide   & 29.30 & 0.830 & 0.517 \\
    \addlinespace
    Trans-CNQ     & Narrow & 20.59 & 0.713 & 0.424 \\
                  & Medium & 28.02 & 0.705 & 0.396 \\
                  & Wide   & 30.78 & 0.798 & 0.526 \\
    \addlinespace
    KAN-CNQ       & Narrow & 23.27 & 0.705 & 0.412 \\
                  & Medium & 25.15 & 0.798 & 0.424 \\
                  & Wide   & 32.00 & 0.794 & 0.469 \\
    \bottomrule
  \end{tabular}
\end{table}

The next two tables repeat these analyses on the larger, more heavily censored
FLCHAIN cohort to check that the qualitative conclusions extend beyond oncology.
Table~\ref{tab:flchain-profiles} reports predicted quantile milestones for five
representative FLCHAIN patients, expressed in days rather than months, and again
shows coherent, individualized distributions with widely varying interval
widths. Table~\ref{tab:flchain-unc} gives the corresponding uncertainty
stratification (with IPCW weights truncated at the 95th percentile and an extra
column for the average number of subjects per tercile): as in METABRIC, wider
predicted intervals are accompanied by both larger pinball loss and higher
empirical coverage, so the interval-width diagnostic behaves the same way in a
cohort with $\approx72\%$ censoring. Together, Tables~\ref{tab:metabric-profiles}--\ref{tab:flchain-unc}
provide the detailed numbers behind the patient-profile and uncertainty figures
in Section~\ref{sec:supp-figs}.

\begin{table}[t]
  \centering
  \caption{Predicted quantile milestones for five representative FLCHAIN
  patients (Trans-CNQ, seed 42), with survival reported in days. As in
  Table~\ref{tab:metabric-profiles}, each row gives the five non-crossing
  predicted quantiles $\hat q_{0.1},\ldots,\hat q_{0.9}$, the observed follow-up
  time $T_{\mathrm{obs}}$, the event indicator, and two covariates (sex and age
  in years) for one patient (A--E). The examples show individualized milestones
  on a very different clinical scale (all-cause mortality over thousands of
  days) and correspond to the FLCHAIN profiles plotted in
  Figure~\ref{fig:patient-profiles}(b).}
  \label{tab:flchain-profiles}
  \small
  \begin{tabular}{lrrrrr rr cc}
    \toprule
    & $\hat{q}_{0.1}$ & $\hat{q}_{0.25}$ & $\hat{q}_{0.5}$ & $\hat{q}_{0.75}$ & $\hat{q}_{0.9}$
    & $T_{\mathrm{obs}}$ & Event & Sex & Age \\
    \midrule
    Patient A &  439 & 1{,}062 & 1{,}876 & 2{,}815 & 3{,}757 & 2{,}047 & Yes & F & 90 \\
    Patient B &  155 &   614   & 1{,}735 & 2{,}965 & 4{,}293 & 2{,}706 & Yes & M & 59 \\
    Patient C &  847 & 2{,}086 & 3{,}221 & 4{,}207 & 5{,}363 & 2{,}981 & Yes & F & 50 \\
    Patient D &  658 & 1{,}895 & 3{,}533 & 5{,}063 & 6{,}639 & 3{,}347 & Yes & F & 70 \\
    Patient E &  614 & 1{,}731 & 3{,}246 & 4{,}586 & 5{,}940 & 3{,}245 & Yes & M & 69 \\
    \bottomrule
  \end{tabular}
\end{table}

\begin{table}[t]
  \centering
  \caption{Uncertainty stratification on FLCHAIN (averaged over 25 seeds, with
  IPCW weights truncated at the 95th percentile) for each of the three proposed
  models. As in Table~\ref{tab:metabric-unc}, subjects are split into narrow,
  medium, and wide terciles by predicted $80\%$ interval width $w_i$; for each
  model and tercile the table reports the pinball loss, the empirical $80\%$ and
  $50\%$ interval coverages, and the average number of test subjects per tercile
  $\overline{n}_{\text{unc}}$. The monotone increase of loss and coverage with
  interval width---reproduced here in a larger, more heavily censored cohort
  ---confirms that interval width is a reliable internal indicator of predictive
  difficulty; see also Figure~\ref{fig:uncertainty}(b).}
  \label{tab:flchain-unc}
  \small
  \begin{tabular}{ll cccc}
    \toprule
    Model & Tercile & Pinball Loss & 80\% Cov. & 50\% Cov. & Avg.\ $n_{\text{unc}}$ \\
    \midrule
    TransKAN-CNQ  & Narrow & 354.7 & 0.773 & 0.469 & 234.6 \\
                  & Medium & 389.8 & 0.856 & 0.527 & 105.0 \\
                  & Wide   & 400.8 & 0.884 & 0.582 &  94.2 \\
    \addlinespace
    Trans-CNQ     & Narrow & 359.3 & 0.778 & 0.479 & 224.2 \\
                  & Medium & 393.0 & 0.864 & 0.546 & 106.3 \\
                  & Wide   & 403.7 & 0.891 & 0.614 & 103.3 \\
    \addlinespace
    KAN-CNQ       & Narrow & 346.8 & 0.792 & 0.485 & 188.3 \\
                  & Medium & 395.1 & 0.858 & 0.534 & 107.5 \\
                  & Wide   & 419.9 & 0.894 & 0.578 & 138.0 \\
    \bottomrule
  \end{tabular}
\end{table}

\clearpage
\section{Supplementary Figures and Event-Projection Results for the Real-Data
Analyses}
\label{sec:supp-figs}

This section presents the material supporting the real-data analyses, organized
into four groups: individual-level predictions, event projection, feature
importance, and model checking, followed by two robustness tables.  The
event-projection group carries the full setup and error table for the analysis
summarized in the main text.
Figure~\ref{fig:patient-profiles} shows representative patient-level quantile
profiles for METABRIC (a) and FLCHAIN (b): each horizontal line is one patient's
predicted $80\%$ interval $[\hat q_{0.1},\hat q_{0.9}]$, the markers are the five
predicted quantiles, and a cross marks the observed event time. The figure makes
concrete how the model separates patients with similar medians but different
tail risk or different predictive uncertainty, and it visualizes the same five
example patients tabulated in Tables~\ref{tab:metabric-profiles}
and~\ref{tab:flchain-profiles}. Figure~\ref{fig:uncertainty} then aggregates
this uncertainty information across the cohort, plotting pinball loss and
empirical coverage against the tercile of predicted interval width for METABRIC
(a) and FLCHAIN (b); the upward trend in both panels confirms that wider
self-reported intervals correspond to genuinely harder, better-covered
predictions, and mirrors Tables~\ref{tab:metabric-unc} and~\ref{tab:flchain-unc}.

\begin{figure}[ht]
  \centering
  \begin{tabular}{cc}
    \includegraphics[width=0.47\textwidth]{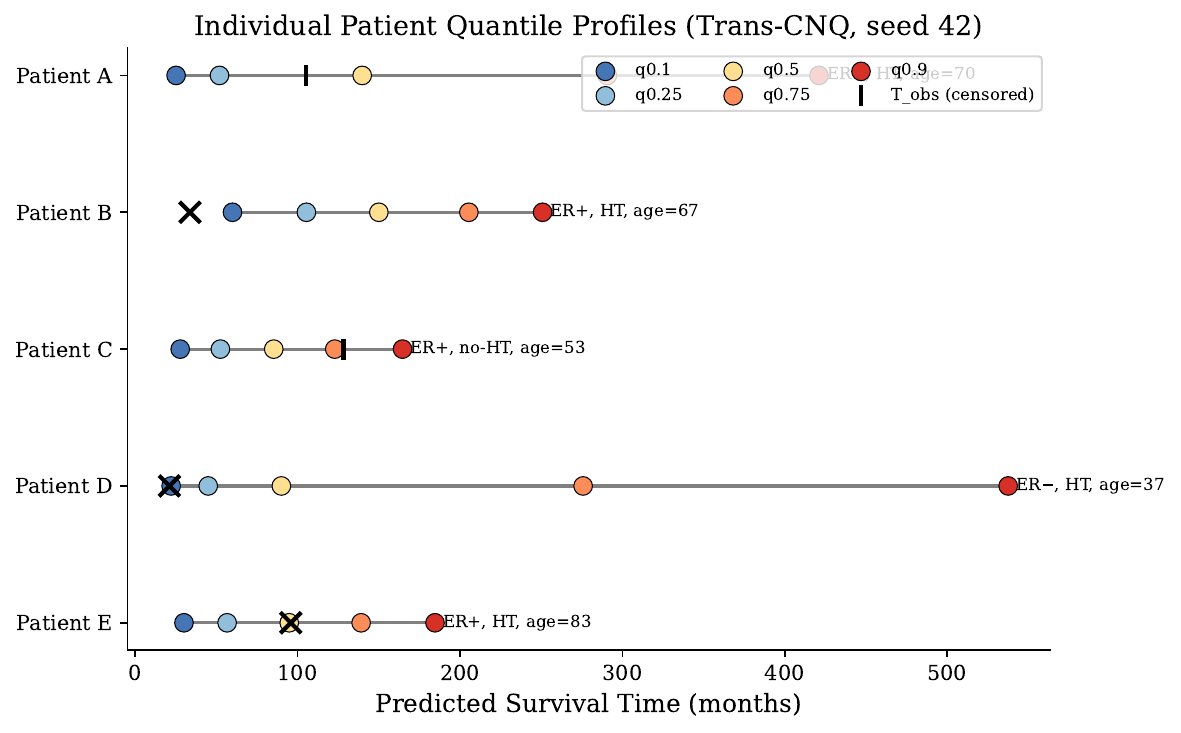} &
    \includegraphics[width=0.47\textwidth]{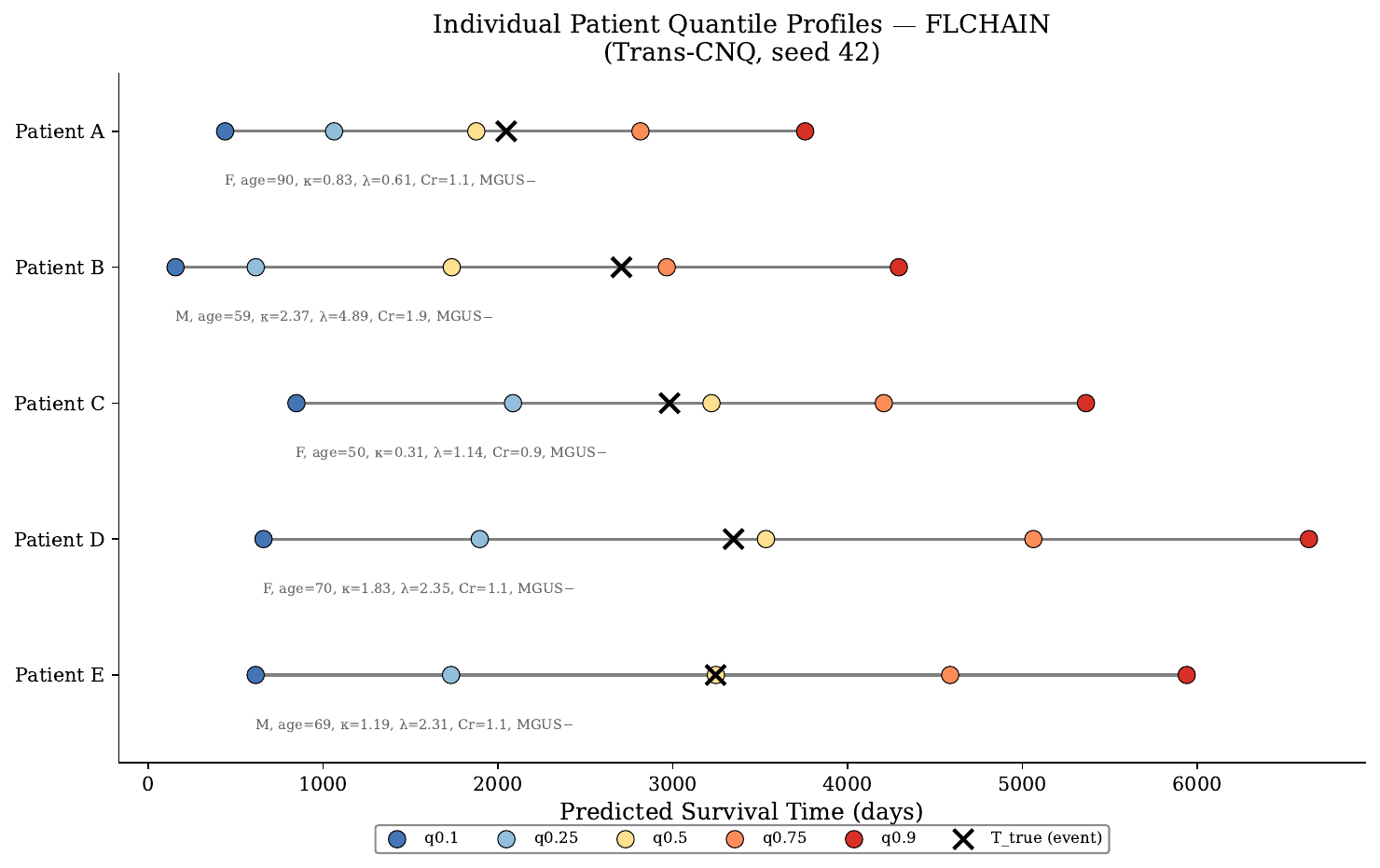} \\[-2pt]
    \small (a) METABRIC patient profiles &
    \small (b) FLCHAIN patient profiles
  \end{tabular}
  \caption{Representative patient-level quantile profiles for (a) METABRIC
  (survival in months) and (b) FLCHAIN (survival in days), produced by
  Trans-CNQ on a single representative split (seed 42). Each row corresponds
  to one patient; the horizontal axis is predicted survival time. For each
  patient the horizontal line spans the predicted $80\%$ prediction interval
  $[\hat q_{0.1},\hat q_{0.9}]$, the colored markers denote the five predicted
  quantiles $\hat q_{0.1},\hat q_{0.25},\hat q_{0.5},\hat q_{0.75},\hat q_{0.9}$
  (increasing left to right, by construction non-crossing), and a cross marks
  the observed event time for uncensored subjects. The panels show that the
  model produces individualized, coherent quantile milestones whose spread
  (interval width) varies substantially across patients, and that observed
  event times for uncensored subjects typically fall within the predicted
  interval. These are the same five example patients tabulated in
  Tables~\ref{tab:metabric-profiles} and~\ref{tab:flchain-profiles}.}
  \label{fig:patient-profiles}
\end{figure}

\begin{figure}[ht]
  \centering
  \begin{tabular}{cc}
    \includegraphics[width=0.47\textwidth]{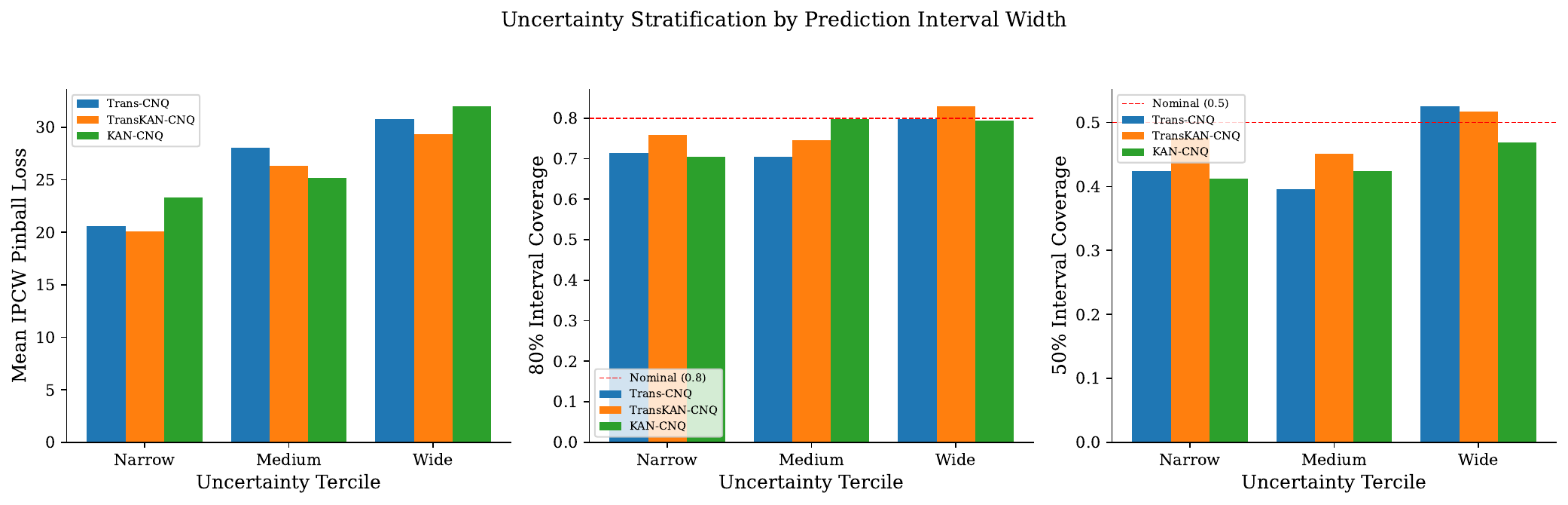} &
    \includegraphics[width=0.47\textwidth]{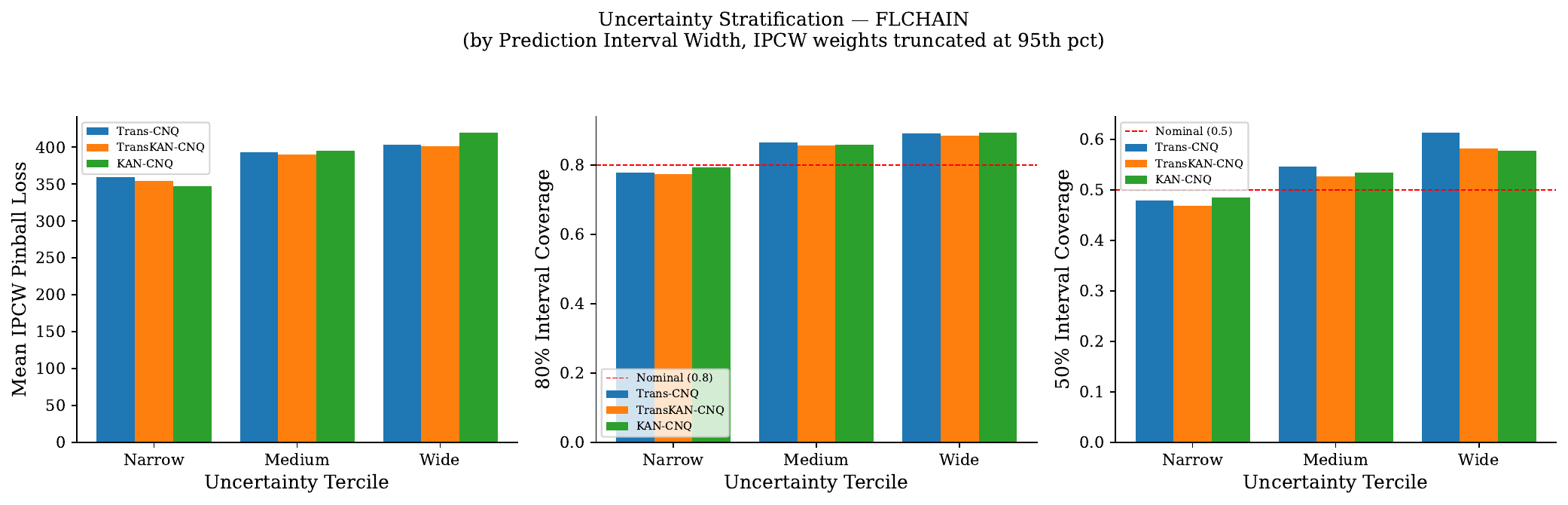} \\[-2pt]
    \small (a) METABRIC uncertainty stratification &
    \small (b) FLCHAIN uncertainty stratification
  \end{tabular}
  \caption{Uncertainty stratification for (a) METABRIC and (b) FLCHAIN
  (Trans-CNQ, averaged over 25 seeds). Test subjects are partitioned into
  terciles---narrow, medium, and wide---by the predicted $80\%$ interval width
  $w_i=\hat q_{0.9,i}-\hat q_{0.1,i}$, a purely internal, label-free measure of
  predictive difficulty. Within each tercile the panels report the IPCW pinball
  loss and the empirical coverage of the nominal $50\%$ and $80\%$ intervals;
  dashed horizontal lines mark the nominal coverage levels. Moving from narrow
  to wide terciles, both the pinball loss and the empirical coverage increase
  monotonically, showing that predictions the model flags as more uncertain
  (wider intervals) are indeed harder and better covered---so interval width
  can serve as a self-reported reliability signal at the individual level.
  The underlying numbers appear in Tables~\ref{tab:metabric-unc}
  and~\ref{tab:flchain-unc}.}
  \label{fig:uncertainty}
\end{figure}

\subsection{Event Projection on METABRIC}
\label{app:metabric-ep}

We evaluate three projection strategies on the METABRIC test sets (seeds 41--65; $n_{\text{sim}}=500$ Monte Carlo draws per subject).

\paragraph{Setup}
The \emph{latent} method projects events by integrating $\hat{F}(t|X_i)$ over all test subjects, treating the estimated CDF as if no administrative cutoff existed.
The \emph{censoring-aware} method (Trans-CNQ $+$ cens.) samples event times $T_{\text{sim}} \sim \hat{F}(\cdot | X_i, T > s)$ via inverse-CDF and censoring times $C_{\text{sim}} \sim \hat{G}(C | C > s)$ via a reverse Kaplan--Meier estimator, then counts individuals for whom $T_{\text{sim}} \leq C_{\text{sim}}$ and $T_{\text{sim}} \leq t$.
The \emph{Weibull} baseline fits a marginal Weibull distribution on the observed $(T, \delta)$ pairs at each cutoff and applies the same censoring model.
Ground truth is the raw cumulative observed event count in the test set up to the final time horizon.

\paragraph{Results}
Table~\ref{tab:metabric-ep} reports the percentage error relative to the raw observed count at the end of follow-up ($t = 200$ months; 25-seed average), and the full projection curves are provided in the Supplementary Material.

\begin{table}[t]
  \centering
  \caption{METABRIC event projection error (\%) at $t = 200$ months (seeds 41--65, mean).
           Positive values indicate overestimation of cumulative events.}
  \label{tab:metabric-ep}
  \small
  \begin{tabular}{lccc}
    \toprule
    Cutoff (months) & Latent & Trans-CNQ $+$ cens. & Weibull $+$ cens. \\
    \midrule
    24 & $+35.2$ & $+19.9$ & $+63.8$ \\
    48 & $+34.4$ & $+20.4$ & $+42.1$ \\
    72 & $+30.4$ & $+19.9$ & $+16.6$ \\
    96 & $+25.4$ & $+18.5$ & $ +6.9$ \\
    \bottomrule
  \end{tabular}
\end{table}

All three methods overestimate at every cutoff.
Notably, the KM-adjusted count (gray dashed) also lies above the raw observed events, indicating that standard KM-based correction itself can become optimistic when the projection horizon extends into sparsely observed follow-up. We therefore use the raw observed event count as the operational reference throughout. The censoring-aware Trans-CNQ procedure (green) consistently tracks this reference more closely than the KM-based curve, with errors of $+18$--$+21\%$ across all cutoffs, indicating that explicit modeling of $\hat{G}(C)$ improves stability. The Weibull benchmark (purple) shows pronounced horizon dependence: it overshoots severely when only short follow-up is available ($+63.8\%$ at the 24-month cutoff), because projection relies heavily on parametric tail extrapolation, but approaches the observed count once most of the event-time distribution has become visible ($+6.9\%$ at the 96-month cutoff).

\paragraph{Root cause}
The persistent overestimation across all approaches appears to share the same mechanism: under IPCW training, censored observations contribute no direct loss beyond their weighting role, so the subject-specific lower-bound information $T_i>C_i$ is not explicitly enforced in the optimization target. Late-censored subjects therefore tend to have their event times predicted too early, which inflates projected cumulative event counts. Likelihood-based approaches such as Weibull or Cox incorporate this lower-bound information through the survival contribution $S(C_i\mid X_i)$ and can therefore be better calibrated in the far tail, although they remain imperfect under the moderate censoring present in METABRIC.

Figure~\ref{fig:metabric-ep} shows the same comparison as curves rather than
endpoint errors. At each of the four cutoffs (24, 48, 72, and 96 months) it
overlays the raw observed cumulative events, a Kaplan--Meier-based latent
reference that adjusts for censoring, the Trans-CNQ projection with the
censoring model, and a parametric Weibull projection. Consistent with
Table~\ref{tab:metabric-ep}, all of the projections lie above the raw observed
count; the Trans-CNQ curve stays closest to the censoring-adjusted reference
across all four horizons, whereas the Weibull projection drifts furthest where
its parametric shape is misspecified and the follow-up is shortest.

\begin{figure}[ht]
  \centering
  \includegraphics[width=\textwidth]{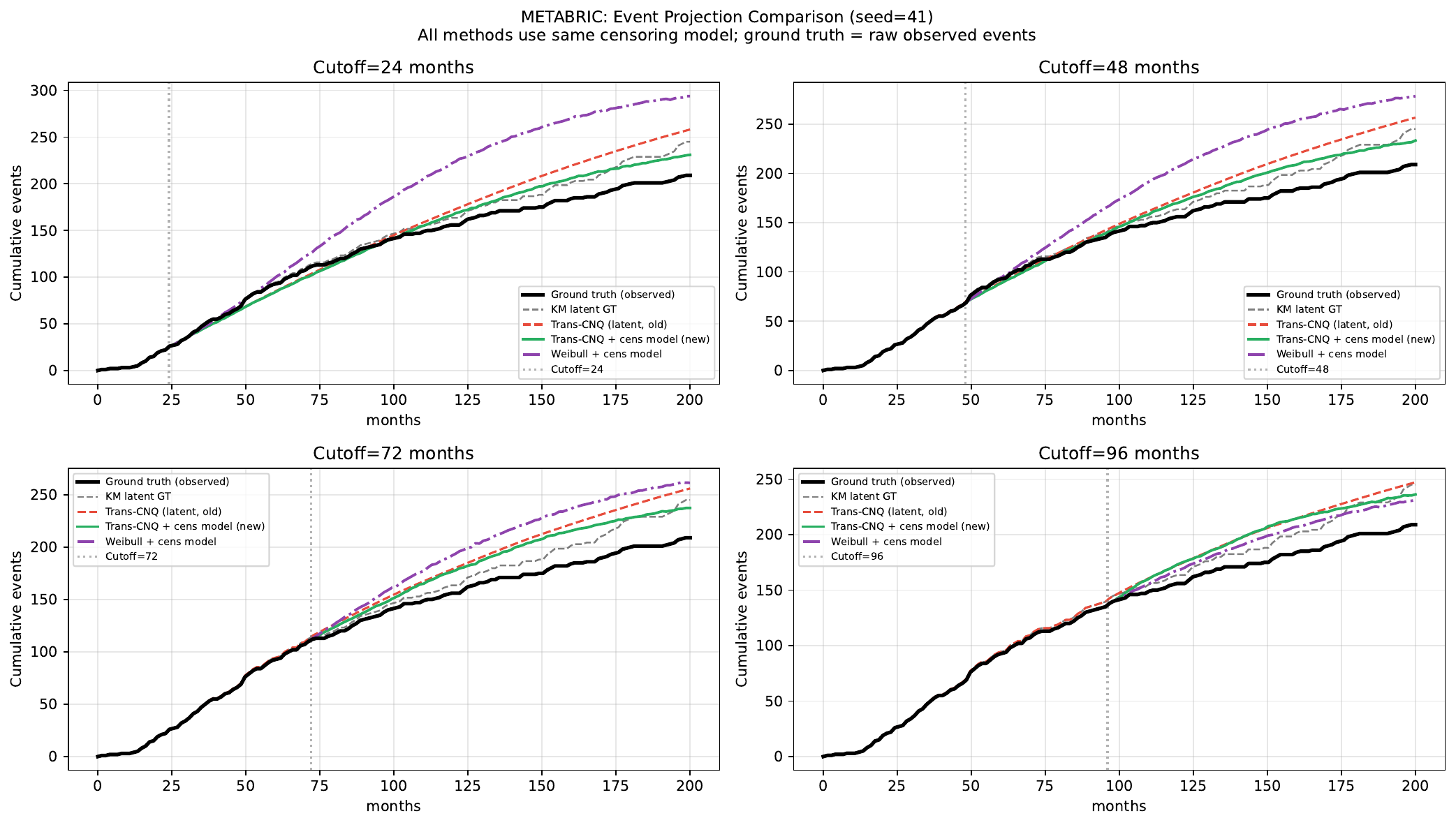}
  \caption{METABRIC event-projection curves (single representative split,
  seed 41) evaluated at four calendar cutoffs---24, 48, 72, and 96 months
  (vertical dotted lines). The horizontal axis is calendar/follow-up time in
  months and the vertical axis is the cumulative number (or proportion) of
  events projected to have occurred by each time. Four curves are overlaid:
  the solid black curve is the raw observed cumulative-event count; the gray
  dashed curve is the Kaplan--Meier-based latent reference that adjusts for
  censoring; the green curve is the Trans-CNQ projection obtained with the
  censoring (IPCW) model; and the purple dash-dot curve is the corresponding
  parametric Weibull projection with the same censoring model. The figure
  assesses how closely each method reproduces the true event accrual over
  time: the Trans-CNQ curve tracks the KM latent reference closely across all
  four cutoffs, whereas the Weibull projection deviates more where the
  parametric shape is misspecified, illustrating the benefit of the flexible
  distributional model for event-count forecasting.}
  \label{fig:metabric-ep}
\end{figure}

The next three figures examine which covariates drive the predictions and
whether the recovered structure is stable across architectures.
Figure~\ref{fig:feature-importance} is a permutation feature-importance heatmap
for METABRIC, averaged over the three proposed models and 25 splits, in which
each cell reports the increase in IPCW pinball loss $\Delta$ from permuting a
covariate at a given quantile level; reading a row across $\tau$ shows whether a
covariate acts on the lower, central, or upper part of the predicted
distribution, and age and chemotherapy emerge as dominant while hormone therapy
shows a pronounced upper-tail effect. Figures~\ref{fig:fi-metabric-compact}
and~\ref{fig:fi-flchain-compact} break this down by architecture for METABRIC and
FLCHAIN, respectively. For METABRIC the three architectures agree closely,
indicating that the importance pattern reflects genuine data structure; for
FLCHAIN the two Transformer-based models agree (age and the serum free light
chains dominate) whereas KAN-CNQ ranks creatinine first, illustrating the
architecture dependence of attribution in the pure-KAN backbone noted
in the main text.

\begin{figure}[ht]
  \centering
  \includegraphics[width=0.72\textwidth]{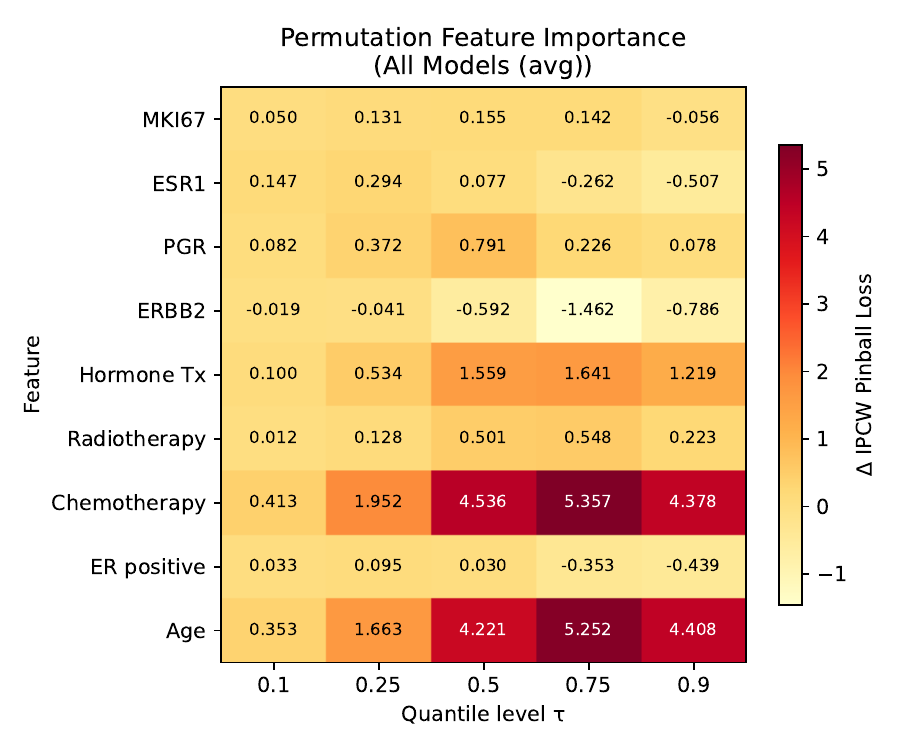}
  \caption{Permutation feature-importance heatmap for METABRIC, averaged over
  25 random splits and pooled across the three proposed architectures
  (Trans-CNQ, TransKAN-CNQ, KAN-CNQ). Rows index the covariates and columns
  index the five quantile levels $\tau\in\{0.1,0.25,0.5,0.75,0.9\}$; the color
  of each cell encodes the increase in IPCW pinball loss, $\Delta$, incurred
  when the corresponding covariate is randomly permuted (breaking its
  association with the outcome) while all others are held fixed, so that larger
  (warmer) values indicate greater prognostic importance at that quantile.
  Reading a row across $\tau$ reveals whether a covariate acts mainly on the
  lower (early-event), central, or upper (long-survivor) part of the predicted
  distribution rather than on a single summary. Age and chemotherapy dominate
  overall, while hormone therapy shows a pronounced effect concentrated in the
  upper quantiles---a tail-specific signal that a single hazard ratio would
  average away.}
  \label{fig:feature-importance}
\end{figure}

\begin{figure}[ht]
  \centering
  \begin{tabular}{ccc}
    \includegraphics[width=0.31\textwidth]{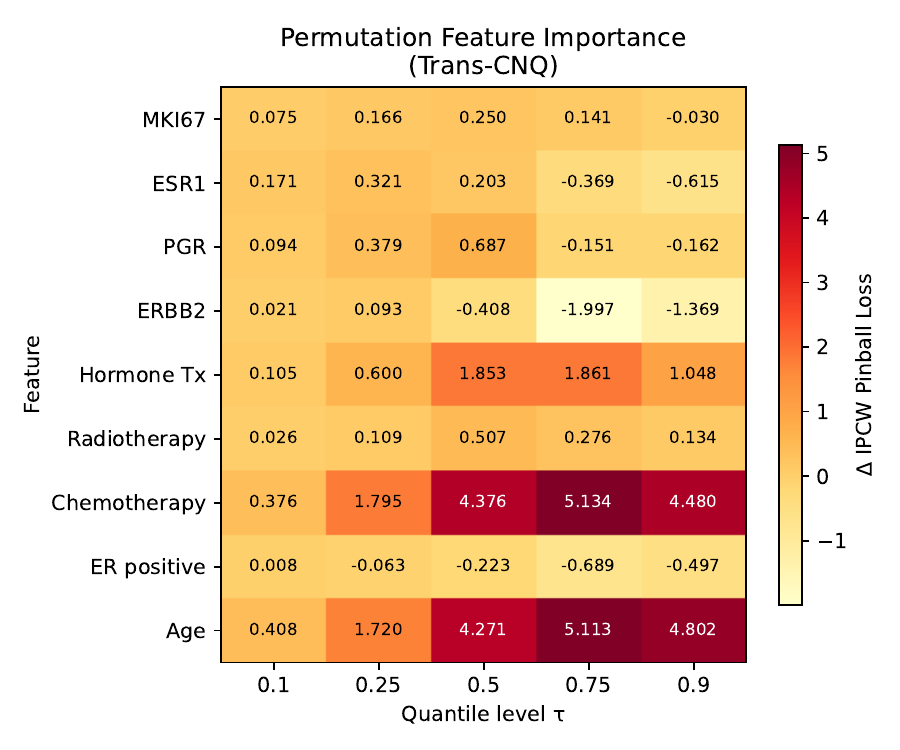} &
    \includegraphics[width=0.31\textwidth]{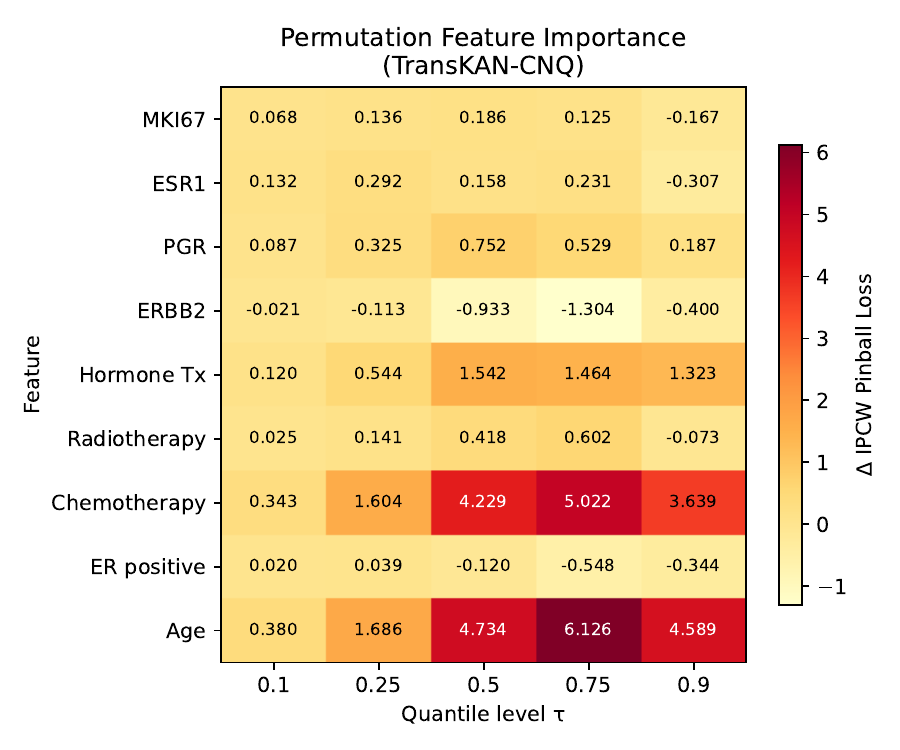} &
    \includegraphics[width=0.31\textwidth]{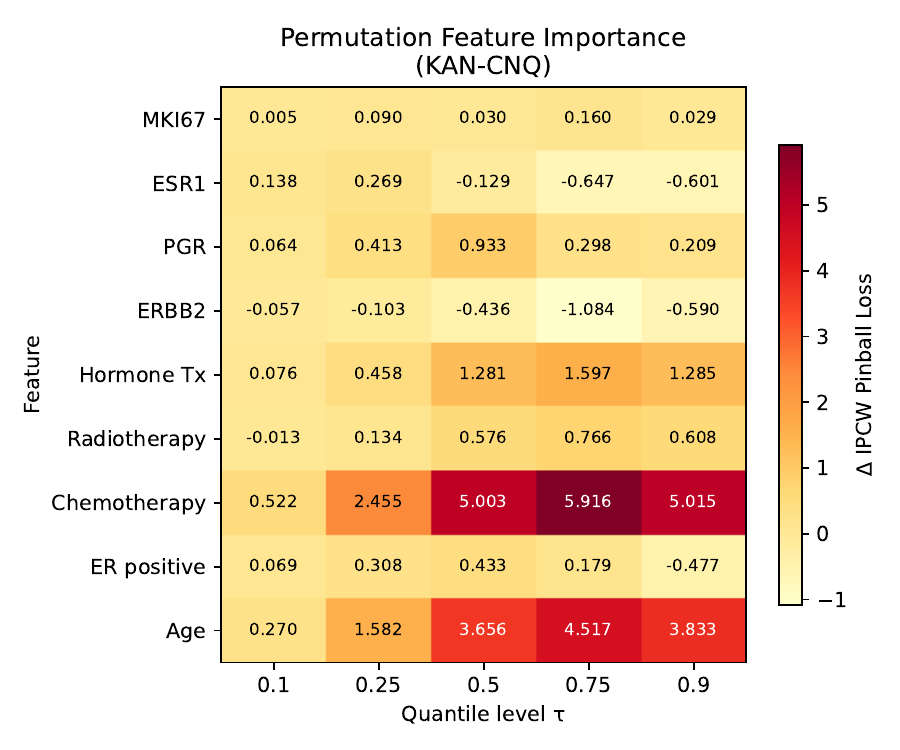} \\[-2pt]
    \small (a) Trans-CNQ &
    \small (b) TransKAN-CNQ &
    \small (c) KAN-CNQ
  \end{tabular}
  \caption{Architecture-specific permutation feature-importance heatmaps for
  METABRIC, shown separately for (a) Trans-CNQ, (b) TransKAN-CNQ, and
  (c) KAN-CNQ (each averaged over 25 random splits). As in
  Figure~\ref{fig:feature-importance}, rows are covariates, columns are the
  five quantile levels $\tau$, and cell color encodes the increase in IPCW
  pinball loss $\Delta$ from permuting that covariate; the three panels share a
  common color scale to make them directly comparable. Decomposing the pooled
  heatmap by architecture shows that the three models yield broadly similar
  qualitative rankings---age and chemotherapy dominate prognosis and hormone
  therapy again exhibits a pronounced upper-quantile ($\tau=0.75,0.9$) effect
  ---indicating that the recovered prognostic structure is a stable property of
  the data rather than an artifact of any single architecture.}
  \label{fig:fi-metabric-compact}
\end{figure}

\begin{figure}[ht]
  \centering
  \begin{tabular}{ccc}
    \includegraphics[width=0.31\textwidth]{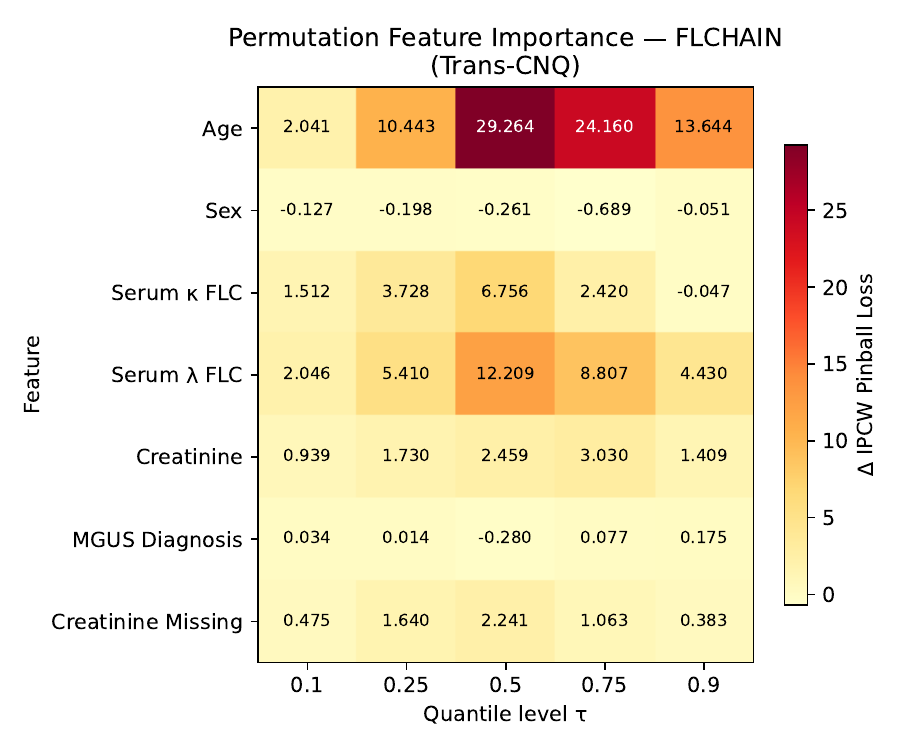} &
    \includegraphics[width=0.31\textwidth]{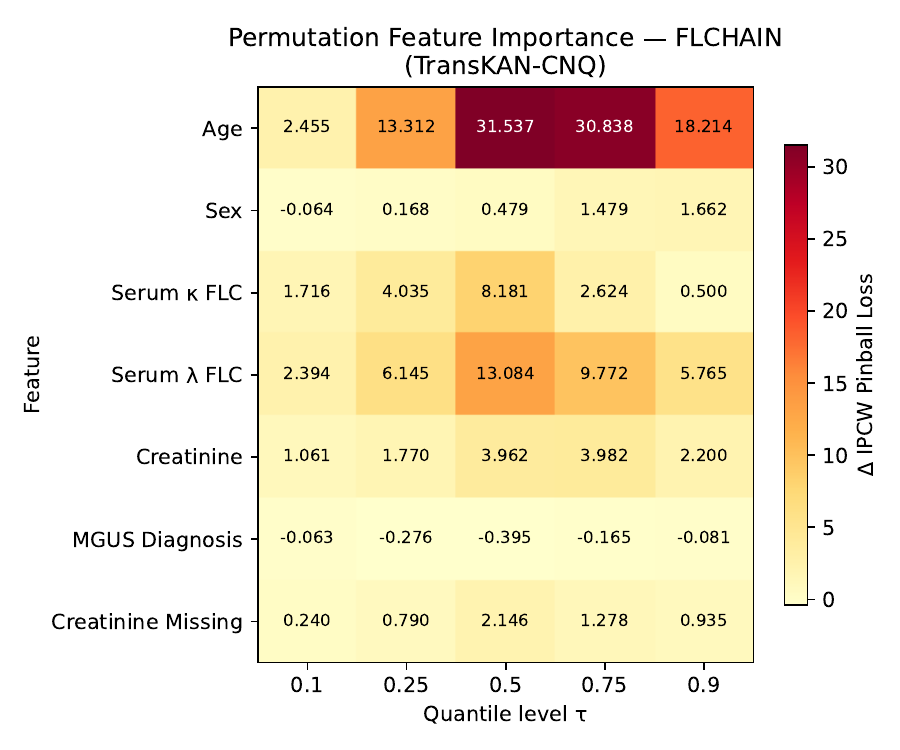} &
    \includegraphics[width=0.31\textwidth]{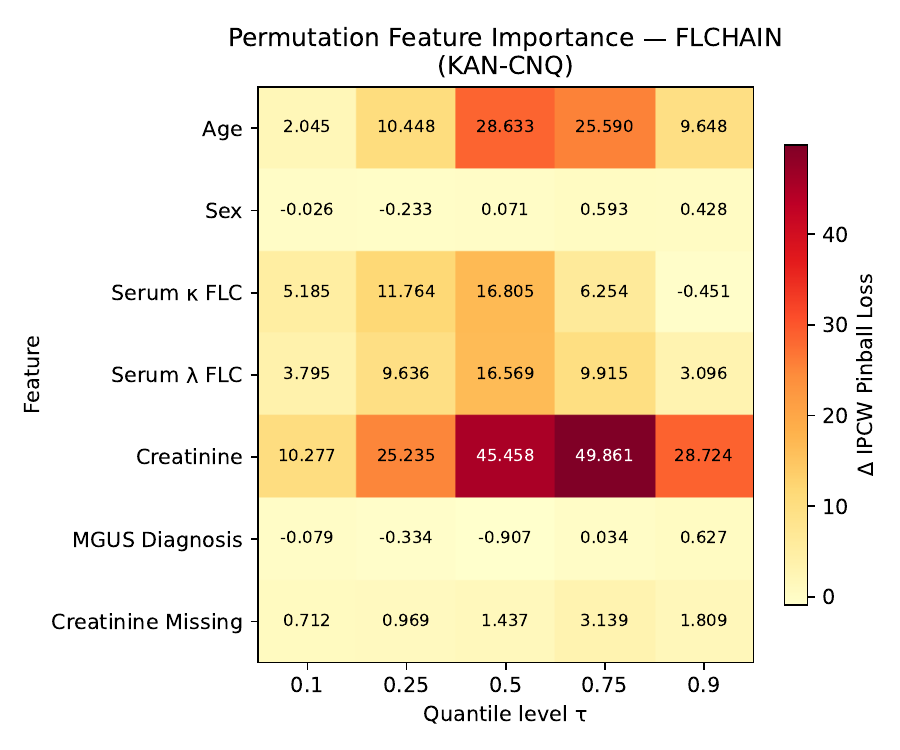} \\[-2pt]
    \small (a) Trans-CNQ &
    \small (b) TransKAN-CNQ &
    \small (c) KAN-CNQ
  \end{tabular}
  \caption{Architecture-specific permutation feature-importance heatmaps for
  FLCHAIN, shown separately for (a) Trans-CNQ, (b) TransKAN-CNQ, and
  (c) KAN-CNQ (each averaged over 25 random splits). Rows are covariates,
  columns are the five quantile levels $\tau$, and cell color encodes the
  increase in IPCW pinball loss $\Delta$ from permuting that covariate. The two
  Transformer-based models produce broadly similar patterns, with age and the
  serum free light chains (kappa and lambda) dominating prognosis, whereas
  KAN-CNQ ranks creatinine first.
  This discrepancy---which is not present for METABRIC in
  Figure~\ref{fig:fi-metabric-compact}---illustrates the architecture dependence
  of attribution in the pure-KAN backbone discussed in the main text, and
  motivates preferring the Transformer-based variants when feature attributions
  are to be reported.}
  \label{fig:fi-flchain-compact}
\end{figure}

The final pair of figures provides direct model-checking diagnostics.
Figure~\ref{fig:residuals} shows the distribution of log-time residuals at the
predicted median for uncensored METABRIC subjects; the histograms for the
Transformer-based models are centered near zero, indicating no systematic over-
or under-prediction of the median. Figure~\ref{fig:calibration-flchain} repeats
the calibration checks of the main text on FLCHAIN, plotting per-quantile
realized coverage against the nominal level (a) and interval coverage against
the nominal interval level (b); both panels lie close to the $45^\circ$ line,
so the calibration observed on METABRIC replicates in a larger cohort with much
heavier ($\approx72\%$) censoring.

\begin{figure}[ht]
  \centering
  \includegraphics[width=0.6\textwidth]{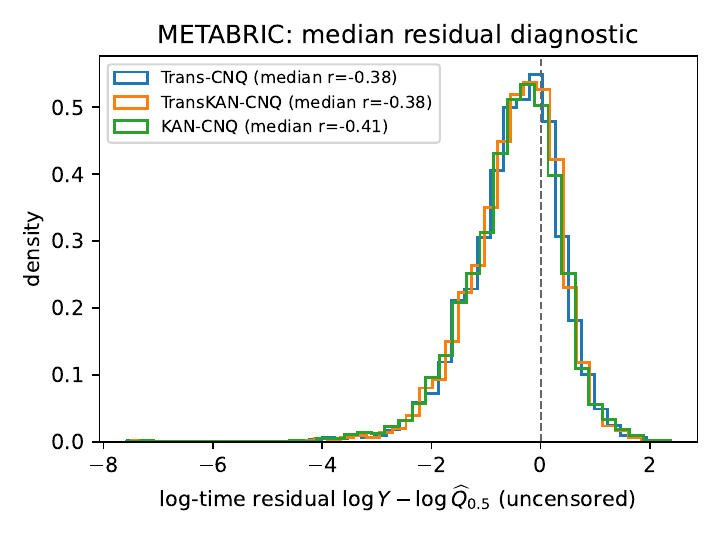}
  \caption{METABRIC log-time residual diagnostic at the predicted median.
  Restricting to uncensored subjects (for whom the event time is observed), the
  histogram shows the log-scale residual $\log Y-\log\widehat Q_{0.5}(X)$, i.e.,
  the signed distance between the observed event time and the predicted median
  survival time on the log scale, pooled over the 25 random splits; the vertical
  dashed line marks zero. If the median is well calibrated the residuals should
  be centered at zero with roughly half above and half below. The Transformer-based
  models (Trans-CNQ, TransKAN-CNQ) produce residual histograms centered near
  zero with no systematic over- or under-prediction, corroborating the
  calibration and coverage checks reported in the main text.}
  \label{fig:residuals}
\end{figure}

\begin{figure}[ht]
  \centering
  \begin{tabular}{cc}
    \includegraphics[width=0.46\textwidth]{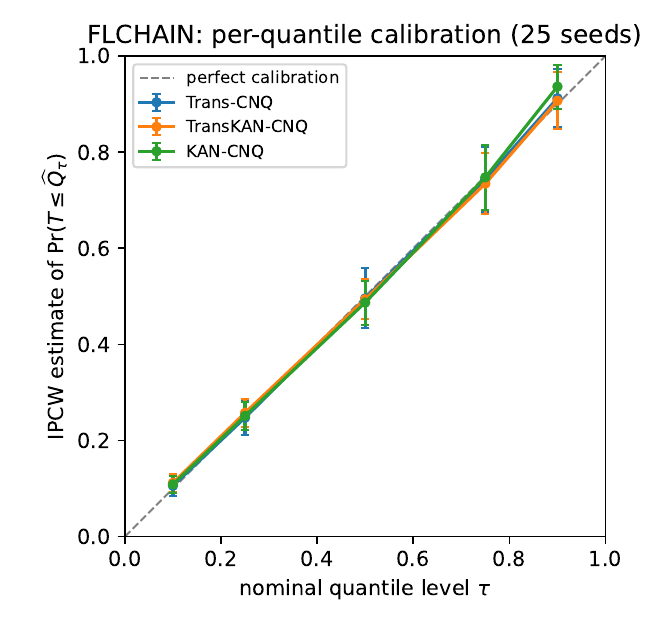} &
    \includegraphics[width=0.46\textwidth]{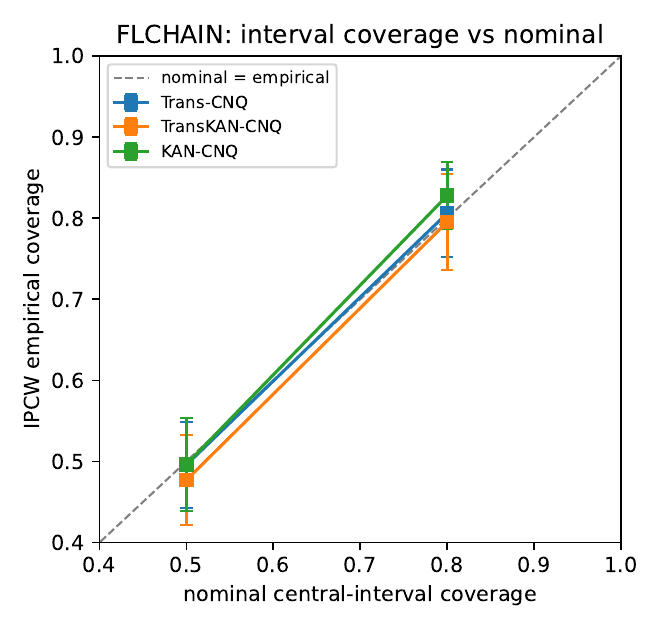} \\[-2pt]
    \small (a) per-quantile calibration &
    \small (b) interval coverage vs.\ nominal
  \end{tabular}
  \caption{FLCHAIN model-checking (Trans-CNQ, averaged over 25 seeds; IPCW
  weights truncated at the 95th percentile to limit the influence of large
  inverse-censoring weights). Panel (a), per-quantile calibration, plots the
  empirical (realized) coverage against each nominal quantile level
  $\tau\in\{0.1,0.25,0.5,0.75,0.9\}$; points on the $45^\circ$ diagonal indicate
  perfect calibration. Panel (b), interval coverage versus nominal, plots the
  empirical coverage of central prediction intervals against their nominal
  level, with the $50\%$ and $80\%$ intervals highlighted and the diagonal
  marking exact coverage. Calibration is near-nominal across all quantile levels
  and the $50\%/80\%$ interval coverage is close to target, replicating the
  METABRIC findings in a larger cohort with substantially heavier ($\approx72\%$)
  censoring and thereby demonstrating that the calibration is not specific to a
  single dataset or censoring regime.}
  \label{fig:calibration-flchain}
\end{figure}

Finally, two tables document that the real-data conclusions are not artifacts of
specific analysis choices. Table~\ref{tab:ipcw-trunc-flchain} varies the IPCW
weight-truncation threshold on the heavily censored FLCHAIN cohort (no truncation
and the 90th, 95th, and 99th percentiles): the pinball loss, the $80\%$ interval
coverage, and the Transformer-over-KAN ordering are all essentially unchanged,
so the reported results do not depend on the truncation level.
Table~\ref{tab:split-sensitivity} retrains the models on METABRIC under three
train/validation/test ratios ($65/15/20$, $50/25/25$, $80/10/10$): performance
and the model ordering are again stable, with only the smallest test set
($80/10/10$) being mildly noisier, as expected. Taken together, these
sensitivity analyses show that the advantages of the proposed models are robust
to the censoring-weight and data-splitting choices used throughout the paper.

\begin{table}[t]
  \centering
  \caption{IPCW weight-truncation sensitivity on FLCHAIN (heavy, $\sim$72\%
  censoring). Because inverse-probability-of-censoring weights can become large
  when the estimated censoring survival $\widehat G$ is small, we truncate them
  at a chosen percentile; this table probes how much that choice matters. The
  two blocks report the mean pinball loss $\overline{L}_{\mathrm{pin}}$ (log
  scale, averaged over the five common quantile levels
  $\tau\in\{0.1,0.25,0.5,0.75,0.9\}$) and the empirical $80\%$ interval coverage
  for each proposed model, as the weights are truncated at the 90th, 95th, and
  99th percentiles or left untruncated (``none''); all values are means over 25
  random splits. Both the loss and the coverage---and the Transformer-over-KAN
  model ordering---are essentially flat across truncation levels, and coverage
  stays close to the nominal $0.80$, showing that the reported results are not
  driven by a particular truncation threshold.}
  \label{tab:ipcw-trunc-flchain}
  \small
  \begin{tabular}{l cccc c cccc}
    \toprule
    & \multicolumn{4}{c}{$\overline{L}_{\mathrm{pin}}$} & & \multicolumn{4}{c}{$80\%$ coverage} \\
    \cmidrule(lr){2-5}\cmidrule(lr){7-10}
    Model & none & 90th & 95th & 99th & & none & 90th & 95th & 99th \\
    \midrule
    TransKAN-CNQ & 0.279 & 0.284 & 0.282 & 0.280 & & 0.795 & 0.817 & 0.814 & 0.806 \\
    Trans-CNQ & 0.282 & 0.286 & 0.285 & 0.283 & & 0.806 & 0.825 & 0.822 & 0.814 \\
    KAN-CNQ & 0.280 & 0.284 & 0.283 & 0.281 & & 0.828 & 0.841 & 0.840 & 0.835 \\
    \bottomrule
  \end{tabular}
\end{table}

\begin{table}[t]
  \centering
  \caption{IPCW weight-truncation sensitivity on METABRIC (moderate, $\sim$40\% censoring), reported in the same format as Table~\ref{tab:ipcw-trunc-flchain}. Weight concentration does not track the censoring rate: across the six cohorts the effective sample size $(\sum_i w_i)^2/\sum_i w_i^2$, expressed as a fraction of the observed events and computed from untruncated weights, is $98\%$ on GBSG, $98\%$ on GBSG-500, $84\%$ on SUPPORT, $84\%$ on FLCHAIN, $75\%$ on NKI70 and only $35\%$ on METABRIC, so METABRIC is where an IPCW-weighted summary rests on the fewest effective observations. Pinball loss is essentially unaffected by truncation, but $80\%$ coverage rises from $0.72$--$0.77$ to $0.80$--$0.82$; for reference, the unweighted coverage among observed events is $0.78$. This suggests the mild undercoverage reported for METABRIC in the main text is driven largely by a few extreme weights rather than by the interval widths themselves. All headline tables use untruncated weights.}
  \label{tab:ipcw-trunc-metabric}
  \small
  \begin{tabular}{l cccc c cccc}
    \toprule
    & \multicolumn{4}{c}{$\overline{L}_{\mathrm{pin}}$} & & \multicolumn{4}{c}{$80\%$ coverage} \\
    \cmidrule(lr){2-5}\cmidrule(lr){7-10}
    Model & none & 90th & 95th & 99th & & none & 90th & 95th & 99th \\
    \midrule
    TransKAN-CNQ & 0.215 & 0.222 & 0.221 & 0.218 & & 0.784 & 0.822 & 0.820 & 0.803 \\
    Trans-CNQ & 0.224 & 0.225 & 0.223 & 0.225 & & 0.720 & 0.797 & 0.795 & 0.763 \\
    KAN-CNQ & 0.227 & 0.231 & 0.230 & 0.230 & & 0.765 & 0.810 & 0.808 & 0.792 \\
    \bottomrule
  \end{tabular}
\end{table}

\begin{table}[t]
  \centering
  \caption{Split-ratio sensitivity on METABRIC, assessing whether the results
  depend on how the data are partitioned. Each proposed model is retrained from
  scratch under three train/validation/test ratios---$65/15/20$, $50/25/25$, and
  $80/10/10$ (with all other training choices held fixed)---on partitions that are
  regenerated for each ratio, and evaluated on its own held-out test set at the five common quantile levels
  $\tau\in\{0.1,0.25,0.5,0.75,0.9\}$; all entries are means over 25 random
  splits.  Because the partitions are drawn afresh, the $65/15/20$ column is not
  the main-text configuration re-scored: it uses different training and test
  sets, so it is expected to differ from the headline table in the
  main text rather than reproduce it. The two blocks report the IPCW pinball loss $\overline{L}_{\mathrm{pin}}$
  (log scale) and the empirical coverage of the nominal $80\%$ interval. Both the
  performance level and the Transformer-over-KAN ordering are stable across the
  three ratios; only the configuration with the smallest test set ($80/10/10$)
  is mildly noisier, as expected from the reduced evaluation sample.}
  \label{tab:split-sensitivity}
  \small
  \begin{tabular}{l ccc c ccc}
    \toprule
    & \multicolumn{3}{c}{$\overline{L}_{\mathrm{pin}}$} & & \multicolumn{3}{c}{$80\%$ coverage} \\
    \cmidrule(lr){2-4}\cmidrule(lr){6-8}
    Model & 65/15/20 & 50/25/25 & 80/10/10 & & 65/15/20 & 50/25/25 & 80/10/10 \\
    \midrule
    TransKAN-CNQ & 0.228 & 0.226 & 0.242 & & 0.736 & 0.738 & 0.670 \\
    Trans-CNQ & 0.229 & 0.224 & 0.237 & & 0.720 & 0.739 & 0.691 \\
    KAN-CNQ & 0.233 & 0.230 & 0.247 & & 0.729 & 0.731 & 0.690 \\\bottomrule
  \end{tabular}
\end{table}

\clearpage
\bibliographystyle{imsart-nameyear}
\bibliography{biomsample_bib}

\end{document}